\documentclass[11pt]{article}

\usepackage{arxiv-style}
\usepackage{amsmath}
\usepackage{amsmath}
\usepackage{amsfonts}
\usepackage{amssymb}
\usepackage{amsthm}
\usepackage{algorithm}
\usepackage{algpseudocode}
\usepackage{bm, bbm}
\usepackage{dsfont}
\usepackage{booktabs}           %
\usepackage{pifont}     %
\usepackage{multirow, multicol}
\usepackage{svg}
\usepackage{xcolor} 
\usepackage{colortbl}
\usepackage{subcaption}
\usepackage{centernot}
\usepackage{enumitem}
\usepackage{titletoc}

\definecolor{darkblue}{RGB}{0,0,139}
\definecolor{red}{RGB}{210,35,35}
\definecolor{alert}{RGB}{0,0,0}
\usepackage[
  colorlinks=true,
  linkcolor=darkblue,
  citecolor=darkblue,
  urlcolor=darkblue,
  filecolor=darkblue
]{hyperref}

\usepackage{aliascnt}

\newtheorem{theorem}{Theorem}

\newaliascnt{lemma}{theorem}
\newtheorem{lemma}[lemma]{Lemma}
\aliascntresetthe{lemma}

\newaliascnt{proposition}{theorem}

\aliascntresetthe{proposition}

\newaliascnt{corollary}{theorem}

\aliascntresetthe{corollary}

\newaliascnt{remark}{theorem}
\newtheorem{remark}[remark]{Remark}
\aliascntresetthe{remark}

\theoremstyle{definition}

\newaliascnt{definition}{theorem}

\aliascntresetthe{definition}

\newaliascnt{example}{theorem}

\aliascntresetthe{example}

\usepackage{cleveref}

\crefname{theorem}{theorem}{theorems}
\Crefname{theorem}{Theorem}{Theorems}

\crefname{lemma}{lemma}{lemmas}
\Crefname{lemma}{Lemma}{Lemmas}

\crefname{proposition}{proposition}{propositions}
\Crefname{proposition}{Proposition}{Propositions}

\crefname{corollary}{corollary}{corollaries}
\Crefname{corollary}{Corollary}{Corollaries}

\crefname{definition}{definition}{definitions}
\Crefname{definition}{Definition}{Definitions}

\crefname{assumption}{assumption}{assumptions}
\Crefname{assumption}{Assumption}{Assumptions}

\crefname{appendix}{appendix}{appendices}
\Crefname{appendix}{Appendix}{Appendices}

\newcommand{\yes}{\ding{51}}
\newcommand{\no}{\ding{55}}
\newcommand{\cbias}{c_{\text{bias}}}
\newcommand{\khit}{k_{\text{hit}}}

\newcommand{\Id}{\mathbf{I}}
\newcommand{\Ind}{\mathds{1}}

\DeclareMathOperator{\Regret}{Regret}
\DeclareMathOperator{\Violation}{Violation}
\DeclareMathOperator{\VioHard}{Violation_{Hard}}

\DeclareMathOperator*{\argmin}{arg\,min}
\newcommand{\KL}{D_{\mathrm{KL}}}

\newcommand{\piunif}{\pi_{\mathrm{unif}}}
\newcommand{\bLambda}{\bm{\Lambda}}
\newcommand{\bphi}{\bm{\phi}}
\newcommand{\bmu}{\bm{\mu}}
\newcommand{\btheta}{\bm{\theta}}

\newcommand{\st}{\;\mathrm{s.t.}\;}

\newcommand{\bigO}{\mathcal{O}}
\newcommand{\tbigO}{\widetilde{\mathcal{O}}}

\newcommand{\calA}{{\mathcal{A}}}

\newcommand{\calE}{{\mathcal{E}}}
\newcommand{\calF}{{\mathcal{F}}}

\newcommand{\calM}{{\mathcal{M}}}
\newcommand{\calN}{{\mathcal{N}}}

\newcommand{\calQ}{{\mathcal{Q}}}

\newcommand{\calS}{{\mathcal{S}}}

\newcommand{\calV}{{\mathcal{V}}}

\newcommand{\bbE}{\mathbb{E}}

\newcommand{\bbP}{\mathbb{P}}

\newcommand{\bbR}{\mathbb{R}}

\newcommand{\bfv}{\mathbf{v}}
\newcommand{\bfw}{\mathbf{w}}
\newcommand{\bfx}{\mathbf{x}}
\newcommand{\bfy}{\mathbf{y}}

\newcommand{\bfA}{\mathbf{A}}
\newcommand{\bfB}{\mathbf{B}}
\newcommand{\bfC}{\mathbf{C}}

\title{
Rate-Optimal Algorithm for Adversarial Linear CMDPs
}
\author{
Kihyun Yu$^{1}$ \quad Honghao Wei$^{2}$ \quad Dabeen Lee$^{3}$\\[0.3em]
$^{1}$KAIST \quad 
$^{2}$Washington State University \quad
$^{3}$Seoul National University\\
\texttt{khyu99@kaist.ac.kr}\quad
\texttt{honghao.wei@wsu.edu}\quad
\texttt{dabeenl@snu.ac.kr}
}

\date{}

\begin{document}
\maketitle

\begin{abstract}
    We study episodic adversarial linear constrained Markov decision processes (CMDPs) with unknown transitions, where both the loss and constraint functions may vary adversarially across episodes. The best previous algorithm achieves $\widetilde{\mathcal{O}}(K^{3/4})$ regret and cumulative constraint violation, leaving a gap to the optimal $\widetilde{\mathcal{O}}(\sqrt{K})$ dependence on the number of episodes $K$. We close this gap by proposing a new primal dual algorithm that achieves $\widetilde{\mathcal{O}}(\sqrt{K})$ regret and cumulative constraint violation without assuming Slater's condition. 
{\color{alert}We further extend the algorithm to achieve the same $\widetilde{\mathcal{O}}(\sqrt{K})$ guarantees for regret and hard constraint violation, which does not allow constraint violations to cancel across episodes. }
The main challenge is that learning linear CMDPs requires uniform concentration over a value function class with a controlled covering number, whereas standard techniques in constrained online learning, such as policy mixing, can make this class more complex. Our algorithm combines adaptive Follow the Regularized Leader (FTRL), contracted value estimation, and an exponential Lyapunov function. An adaptive dual regularizer offsets the dependence on the dual weights in the primal regret bound, removing the need for policy mixing. We further show that the normalization in the FTRL update bounds the policy parameters independently of the magnitudes of the dual weights, which explains why the resulting policy class remains compatible with uniform concentration. Under feature access, the computational complexity is independent of the size of the state space.

\end{abstract}


\section{Introduction}

Safe reinforcement learning (RL) studies sequential decision-making under safety constraints through interaction with an unknown environment. Many safety-critical real-world applications have been studied under the safe RL framework, including autonomous driving~\citep{driving-isele2018safe}, robotics~\citep{robotics-achiam2017constrained}, health care~\citep{healthcare-coronato2020reinforcement}, and large language models~\citep{dai2024safe}. A standard framework for modeling safe RL is the constrained Markov decision process (CMDP), where an agent aims to optimize an expected cumulative loss while satisfying expected constraints that capture safety requirements~\citep{altman2021constrained}. 

Recently, algorithms for learning CMDPs have been extensively studied along several important directions. One direction is to incorporate function approximation in order to handle environments with large state spaces. In particular, linear CMDPs assume that the transition kernel, loss functions, and constraint functions admit linear representations with respect to a known feature mapping. Under stochastic loss and constraint functions, where the functions are sampled from fixed distributions, several algorithms have been developed for linear CMDPs with sublinear regret and constraint violation guarantees~\citep{ghosh2022provably, ghosh2024towards, kitamura2025provably, peng2026towards}. However, these works focus only on fixed environments, in the sense that the underlying distributions of the functions remain fixed, and do not handle changing environments.

To capture changing environments, another important direction is to consider adversarial environments. In many real-world applications, loss and constraint functions may change arbitrarily over time rather than being sampled from fixed distributions. For example, safety requirements or environmental conditions may vary across episodes due to changes in operating conditions, user preferences, or external factors. To capture such scenarios, online adversarial CMDPs allow loss or constraint functions to be chosen adversarially across episodes. This setting has been extensively investigated for tabular CMDPs~\citep{zhu2025optimistic, stradi2024online, roy2026optimal}. While these works successfully handle adversarial environments, they are limited to the tabular setting, which cannot handle large state spaces.

A natural question is whether these two directions---linear function approximation and adversarially changing environments---can be handled simultaneously. Recently, \citet{yu2026near,yu2026primaldual} initiated the study of CMDPs with adversarial losses and stochastic constraints under linear function approximation. Specifically, while \citet{yu2026near} consider adversarial linear mixture CMDPs with rate-optimal guarantees of $\tbigO(\sqrt{K})$, where $K$ denotes the number of episodes, their algorithm becomes intractable when the state space is arbitrarily large. On the other hand, \citet{yu2026primaldual} propose a model-free algorithm for adversarial linear CMDPs that remains computationally efficient even when the state space is large, but their guarantees remain suboptimal in terms of $K$, i.e., $\tbigO(K^{3/4})$. In summary, existing algorithms for adversarial CMDPs under linear function approximation are either computationally inefficient or statistically suboptimal in terms of $K$.

To close this gap, we propose a model-free and rate-optimal algorithm for learning adversarial linear CMDPs, which is thus computationally efficient even when the state space is large and statistically optimal in terms of $K$. However, achieving this goal is nontrivial because learning linear CMDPs requires controlling the covering number of the value function class~\citep{jin2020provably}. In particular, while the covering number argument is an essential analytical tool for establishing concentration in linear MDPs, it also limits the direct use of other useful techniques from constrained online learning. For instance, \citet{yu2026primaldual} employ policy mixing, which is a common technique in the broad area of constrained online learning~\citep{wei2020online, qiu2020upper, yu2026near}, but is problematic for linear CMDPs due to the covering number issue. Indeed, this induces a trade-off, which is a main reason why their guarantees remain suboptimal.

{\bf Main Contributions.} As previously mentioned, learning adversarial linear CMDPs is particularly challenging due to the covering number issue, which does not appear in other constrained online learning problems. Nevertheless, we design a rate-optimal algorithm for learning adversarial linear CMDPs. We summarize our main contributions below and in \Cref{tab:comparison}.
\begin{itemize}[leftmargin=*]
    \item 
    {\color{alert}
    We propose algorithms for learning adversarial linear CMDPs, where both the loss and constraint functions are adversarially chosen in each episode under full-information feedback. In particular, \Cref{alg:main} achieves the rate-optimal guarantees of $\tbigO(\sqrt{K})$ for both regret and constraint violation. \Cref{alg:hard} then extends these guarantees to hard constraint violation, which provides a stronger notion of constraint satisfaction by preventing constraint violations from canceling across episodes.
    }
    
    \item Moreover, our algorithms do not require Slater's condition, which assumes the existence of a strictly feasible solution with some positive margin, called the Slater constant $\xi$. Specifically, our algorithms do not take $\xi$ as an input, and our guarantees do not depend on $\xi$; thus, they do not deteriorate even when $\xi$ becomes arbitrarily close to $0$. Recently, in the linear CMDP literature, several algorithms have been proposed that do not require this condition, but they are either computationally inefficient~\citep{ghosh2024towards} or limited to suboptimal guarantees~\citep{peng2026towards}. In contrast, our algorithms are computationally efficient and achieve rate-optimal guarantees without Slater's condition in the linear CMDP setting.

    \item The intuition behind our algorithm design is to develop a new drift analysis for adversarial linear CMDPs that avoids policy mixing, thereby avoiding the conflict with the covering number argument. To this end, we incorporate an adaptive regularizer into the dual update together with an exponential Lyapunov function, inspired by \citet{yu2026constrained}, where such a technique was originally developed for convex optimization. This adaptive regularization leads to a drift upper bound that can be controlled without policy mixing. In contrast, the analysis of \citet{yu2026primaldual} relies on policy mixing to control the KL divergence terms appearing in the drift analysis.
    
    \item While adopting the adaptive regularizer in the dual update allows us to avoid policy mixing, there are additional technical challenges that we must address. The first challenge is the unbounded diameter of the solution space. In particular, a direct extension of \citet{yu2026constrained} to the policy space would require a bounded KL-diameter, which does not hold for the probability simplex. To resolve this issue, we adopt adaptive Follow-the-Regularized-Leader (FTRL) for the policy update, whose regret analysis avoids such a diameter dependence. The second challenge is that a na\"ive bound on the dual variable can grow exponentially in $K$ due to the use of an exponential Lyapunov function. Since such a bound enters the covering number analysis, we clip the dual variable at a polynomial threshold and then show that this clipping does not affect the dual dynamics.
\end{itemize}

        
        
        

\begin{table}[t]
\caption{
Comparison of recent algorithms for CMDPs.
Here, Sto. and Adv. denote stochastic and adversarial, respectively.
The \emph{$|\calS|$-free} column indicates whether the computational complexity is independent of $|\calS|$.
\emph{Slater-free} indicates that neither the algorithm nor its guarantees require Slater's condition. 
{\color{alert}\emph{Hard} indicates whether a hard constraint violation guarantee is provided.}
}
    \label{tab:comparison}
    \centering
    \begin{tabular}{|c|c|c|c|c|c|c|c|}
        \hline
         Algorithm & MDP & Loss & Constr. & Rate & $|\calS|$-free & Slater-free & Hard\\
         \hline
        
        \citet{roy2026optimal} 
        & {\color{red}Tabular} 
        & Adv. 
        & Adv. 
        & $\sqrt{K}$ 
        & {\color{red}\no} 
        & \yes 
        & \yes 
        \\
        
        \citet{peng2026towards} 
        & Linear
        & {\color{red}Sto.} 
        & {\color{red}Sto.} 
        & $\sqrt{K}$ 
        & \yes 
        & {\color{red}\no} 
        & \yes 
        \\

        \citet{peng2026towards} 
        & Linear
        & {\color{red}Sto.} 
        & {\color{red}Sto.} 
        & {\color{red}$K^{3/4}$} 
        & \yes 
        & \yes 
        & \yes 
        \\
        
        \citet{yu2026near} 
        & Lin. Mixture
        & Adv. 
        & {\color{red}Sto.} 
        & $\sqrt{K}$ 
        & {\color{red}\no} 
        & {\color{red}\no} 
        & {\color{red}\no} 
        \\
        
        \citet{yu2026primaldual} 
        & Linear 
        & Adv. 
        & {\color{red}Sto.} 
        & {\color{red}$K^{3/4}$} 
        & \yes 
        & {\color{red}\no} 
        & {\color{red}\no} 
        \\

        \hline
        
        \rowcolor{gray!15}
        \textbf{\Cref{alg:main}} 
        & Linear 
        & Adv. 
        & Adv. 
        & $\sqrt{K}$ 
        & \yes 
        & \yes 
        & {\color{red}\no}
        \\

        \rowcolor{gray!15}
        \textbf{\Cref{alg:hard}} 
        & Linear 
        & Adv. 
        & Adv. 
        & $\sqrt{K}$ 
        & \yes 
        & \yes 
        & \yes 
        \\
        
        \hline
    \end{tabular}
    \vspace{-3ex}
\end{table}

\section{Problem Setting}\label{sec:problem setting}
{\bf Notation.} Let $[n] = \{1,\ldots, n\}$ for some positive integer $n$, let $[\cdot]_+ = \max\{0,\cdot\}$, and let $\Delta(\calA)$ denote the probability simplex over the set $\calA$. Given a positive integer $d$ and a vector $\bfx \in \bbR^d$, let $||\bfx||$ denote the $\ell_2$-norm, and let $||\bfx||_1, ||\bfx||_\infty$ denote the $\ell_1$- and $\ell_\infty$-norms, respectively. For a positive definite matrix $\bfA \in \bbR^{d\times d}$, let $||\bfA||$ denote the operator norm, let $||\bfx||_{\bfA} = \sqrt{\bfx^\top\bfA\bfx}$, and let $\lesssim$ denote the inequality that omits constants and polylogarithmic factors. The additional notation used throughout the paper is summarized in Appendix~\ref{appendix:notation}.

{\bf Adversarial CMDPs.} Let $\calM = (\calS, \calA, H, K, s_1, b, \{f^k,g^k\}_{k=1}^K, \{\bbP_h\}_{h=1}^H)$ denote a CMDP, where $\calS$ is a finite state space, $\calA$ is a finite action space, $H$ is a horizon length, $K$ is the number of episodes, $s_1 \in\calS$ is an initial state in each episode $k$, and $b\in[0,H]$ is a constraint threshold. While we assume a finite state space for simplicity, it can be extended to the arbitrarily large state case, as the statistical guarantees and computational complexity do not depend on $|\calS|$. Additionally, $\{f^k,g^k\}_{k=1}^K$ denotes the sequence of loss and constraint functions, where $f_h^k,g_h^k:\calS\times\calA \to [0,1]$ for each $(h,k)\in [H] \times [K]$, and $\bbP$ denotes a transition kernel, where $\bbP_h(s'|s,a)$ is the probability of transitioning from state $s$ to $s'$ by taking action $a$ at step $h\in [H]$.

The interaction between the agent and the environment proceeds as follows. For each episode $k$ and step $h$, after observing $s_h^k$, the agent takes an action $a_h^k \sim \pi_h^k(\cdot|s_h^k)$, where $\pi_h^k$ denotes the agent's policy. Then, the next state $s_{h+1}^k$ is sampled from $\bbP_h(\cdot|s_h^k,a_h^k)$. Once episode $k$ ends, the full-information feedback of $f^k,g^k$ is revealed, while the transition kernel $\{\bbP_h\}_{h=1}^H$ remains unknown.\footnote{While we consider adversarial constraints under full-information feedback, our algorithm can be applied to stochastic constraints under bandit feedback with a few modifications to the algorithm, maintaining the main results. This is possible because, if the constraint function satisfies the linear structure, the cost of estimating the constraint function does not change the dominant terms in the regret and violation bounds.}

Given a CMDP $\calM$, we define the value functions as follows. We define the value function as $V_{f^k,h}^\pi(s) = \bbE_{\bbP,\pi}[\sum_{h'=h}^H f_{h'}^k(s_{h'},a_{h'})|s_h=s]$ and the $Q$-function as $Q_{f^k,h}^\pi(s,a) = \bbE_{\bbP, \pi}[\sum_{h'=h}^H f_{h'}^k(s_{h'},a_{h'}) |s_h=s,a_h=a]$, where the expectations are taken over the trajectory $\{s_h,a_h\}_{h=1}^H$ given $\bbP,\pi$. We consider regret and constraint violation as performance metrics: 
\begin{align*}
    \Regret(K) = & \sum\nolimits_{k=1}^K (V_{f^k,1}^{\pi^k}(s_1) - V_{f^k,1}^{\pi^\star}(s_1)) \\
    \Violation(K) =& \sum\nolimits_{k=1}^K  (V_{g^k,1}^{\pi^k}(s_1) - b)
\end{align*}
where $\pi^\star \in \argmin_{\pi\in\Pi} \sum_{k=1}^K V_{f^k,1}^\pi(s_1) \ \st\ V_{g^k,1}^\pi(s_1) \leq b \ \forall k\in [K]$, under the assumption that the corresponding feasible policy region is non-empty. Here, $\Pi$ denotes the set of Markovian policies, i.e., $\Pi = \{\{\pi_h\}_{h=1}^H: \pi_h: \calS \to \Delta(\calA), \ \forall h \in [H]\}$.

\begin{remark}
{\rm
For adversarial losses and constraints, we may consider a different definition of $\pi^\star$, for example, $\argmin_{\pi\in\Pi} \sum_{k=1}^K V_{f^k,1}^\pi(s_1) \ \st\ \sum_{k=1}^K V_{g^k,1}^\pi(s_1) \leq Kb$. However, under this definition of $\pi^\star$, there is no algorithm that makes both $\Regret(K)$ and $\Violation(K)$ sublinear~\citep{mannor2009online}. Given this fact, our definition of $\pi^\star$ is standard in the constrained online learning literature~\citep{sinha2024optimal, zhu2025optimistic}.
}
\end{remark}

{
\color{alert}
Moreover, we consider the hard constraint violation, denoted by $\VioHard(K)$, which does not allow violations to cancel across episodes, unlike $\Violation(K)$. It is defined as
\begin{align*}
    \VioHard(K) = \sum\nolimits_{k=1}^K [V_{g^k,1}^{\pi^k}(s_1) - b]_+.
\end{align*}
This is a stronger notion of constraint violation than $\Violation(K)$: any upper bound on $\VioHard(K)$ immediately yields the same upper bound on $\Violation(K)$, whereas the converse does not generally hold.
}

{\bf Linear CMDPs.} Throughout the paper, we consider linear CMDPs. That is, given a dimension $d$ and a known feature mapping $\bphi:\calS\times\calA \to\bbR^d$, the loss and constraint functions admit the following linear representation: $f_h^k(s,a) = \bphi(s,a)^\top\btheta_{f,h}^k$ and $g_h^k(s,a) = \bphi(s,a)^\top\btheta_{g,h}^k$ for each $(s,a,h,k) \in\calS\times\calA\times [H]\times[K]$. Similarly, the transition kernel can be represented as $\bbP_h(s'|s,a) = \bphi(s,a)^\top\bmu_h(s')$ for some $\bmu_h:\calS\to\bbR^d$. We note that $\{\bmu_h\}_{h=1}^H$ is unknown during the learning process, while $\{\btheta_{f,h}^k, \btheta_{g,h}^k\}_{h=1}^H$ are revealed at the end of episode $k$. For each $(s,a,h,k)\in\calS\times\calA\times[H]\times [K]$, we assume that $||\btheta_{f,h}^k||,||\btheta_{g,h}^k||\leq \sqrt{d}$, $||\bphi(s,a)||\leq 1$, and $||\sum_{s'\in\calS}|\bmu_h(s')|||\leq \sqrt{d}$, where $|\bmu_h(s')|$ denotes the component-wise absolute value of a vector. Note that we follow the standard regularity assumptions for linear MDPs~\citep{cassel2024warm}.

\begin{remark}
{\rm
Unlike the previous work~\citep{yu2026primaldual}, our problem setting does not impose Slater's condition. Slater's condition requires the existence of a policy that satisfies the constraint with a strictly positive margin $\xi$, referred to as the Slater constant. The guarantees of \citet{yu2026primaldual} depend inversely on $\xi$ and therefore deteriorate as $\xi$ becomes small. In contrast, our guarantees are independent of $\xi$ and hold without assuming Slater's condition.
}
\end{remark}

\section{Difficulty and Our Approach}
The fundamental difficulty of learning adversarial linear CMDPs arises from the need to control the covering number. Here, the covering number argument allows us to establish concentration properties for linear CMDPs, constituting an essential analytical tool~\citep{jin2020provably}. However, this requirement often creates an obstacle to directly applying useful techniques from the constrained online learning literature. This is why the state-of-the-art algorithm for adversarial linear CMDPs remains suboptimal~\citep{yu2026primaldual}, while several optimal algorithms have been proposed in other related settings, such as constrained online convex optimization~\citep{sinha2024optimal}, tabular CMDPs~\citep{qiu2020upper}, and linear mixture CMDPs~\citep{yu2026near}, which do not require a covering number argument.

To elaborate on this difficulty, we investigate the bottleneck in the previous work toward achieving optimal guarantees. \citet{yu2026primaldual} proposed a primal-dual algorithm for linear CMDPs with adversarial losses, but their guarantees are limited to $\tbigO(K^{3/4})$, implying that there remains a gap from the optimal rate $\tbigO(\sqrt{K})$. Importantly, this gap stems from the policy mixing step in their algorithm, which slightly perturbs the policy before optimization. Indeed, this is a well-known technique in the broad area of constrained online learning, as it is helpful for the stability of the dual process by mitigating KL divergence terms that appear in the Lyapunov drift analysis~\citep{wei2020online, qiu2020upper, yu2026near}; see also \eqref{eq:drift kl simple}. However, in the linear CMDP setting, this causes additional problems, as it leads to a large covering number, resulting in looser concentration bounds. This is because the policy structure resulting from mixing is complex~\citep{yu2026primaldual}. 
In summary, policy mixing is required for their Lyapunov analysis, whereas it causes additional problems particularly in the linear CMDP setting, leading to the suboptimal guarantees.

Motivated by this investigation, the key intuition behind our approach is to develop a new Lyapunov drift analysis for adversarial linear CMDPs that does not rely on mixing. Specifically, its particular advantage in this setting is that the corresponding policy structure remains simple, leading to a small covering number. To achieve the desired drift analysis, we adopt an adaptive regularizer in the dual update together with an exponential Lyapunov function, inspired by \citet{yu2026constrained}. For comparison, in the analysis of \citet{yu2026primaldual}, the Lyapunov drift term is bounded by KL divergence terms, which can be expressed in the following simplified form, while the detailed version is provided in Appendix~\ref{appendix:discussion}: given the dual variable $Y_k$, a Slater constant $\xi$, and the corresponding policy $\bar\pi$,
\begin{equation}\label{eq:drift kl simple}
    Y_{k+1}^2 - Y_k^2 \lesssim -\xi Y_k + \KL(\bar\pi || \pi_h^k).
\end{equation}
From \eqref{eq:drift kl simple}, to ensure that the drift term is bounded, perturbing $\pi_h^k$ is essential for bounding the KL divergence term on the right-hand side. Without perturbation, the KL divergence could be unbounded over the policy space. On the other hand, our method yields the following drift bound. Let $R_k$ denote the adaptive regularizer, $\Phi$ denote the exponential Lyapunov function, and $\hat V_{g,1}^k(s_1)$ denote the constraint value function estimate, whose precise formulations are given in \Cref{sec:algorithm}. Then we have
\begin{equation}\label{eq:drift Rk}
    \Phi(Y_{k+1}) - \Phi(Y_k) \lesssim - R_k \Phi'(Y_k) + \Phi'(Y_k)(\hat V_{g,1}^k(s_1) - b).
\end{equation}
From \eqref{eq:drift Rk}, we observe that the right-hand side does not depend on such a KL divergence term. This suggests that our Lyapunov analysis does not require the mixing step to guarantee the stability of the dual process and is therefore more appropriate for the linear CMDP setting.
\begin{remark}
{\rm
Another important distinction between \eqref{eq:drift kl simple} and \eqref{eq:drift Rk} lies in the role of the Slater constant $\xi$. In particular, the negative drift term $-\xi Y_k$ in \eqref{eq:drift kl simple} relies on $\xi$, which plays a conceptual role in ensuring the stability of the dual process. On the other hand, the Lyapunov drift in \eqref{eq:drift Rk} is not controlled by $\xi$, implying that stability can be guaranteed even without the presence of $\xi$. That said, our algorithm and analysis still work without Slater's condition, while those of \citet{yu2026primaldual} do not.
}
\end{remark}
However, directly applying the approach of \citet{yu2026constrained}, which was originally proposed for convex optimization, does not lead to rate-optimal guarantees in the linear CMDP setting. In particular, we must address the following two technical challenges. The first challenge is the unbounded diameter. Specifically, they assume that the diameter of the primal space in terms of the $\ell_2$-norm is bounded, whereas in our case, the diameter of the policy space in terms of KL divergence is unbounded. The second challenge is that, since we adopt an exponential Lyapunov function $\Phi$, na\"ive bounds on $Y_k$ can grow exponentially in $K$. Under our analysis, such a na\"ive bound enters the covering number as in \eqref{eq:naive Yk} (see also \Cref{lem:covering number}), it could make the covering number large, leading to poor guarantees. That is, letting $C_Y$ be the uniform na\"ive upper bound on $Y_k$,
\begin{equation}\label{eq:naive Yk}
    \log \calN_\epsilon \lesssim d^2 \log(C_Y) \lesssim d^2 K.
\end{equation}
We address these technical challenges as follows. To resolve the first challenge, we apply adaptive FTRL~\citep{orabona2019modern} to the MDP setting.\footnote{While in the convex optimization literature, \citet{yu2026constrained} considers the Euclidean setting together with gradient descent algorithms, based on our observation, their work may be extended to the probability simplex case by adopting adaptive FTRL.} Specifically, Adaptive Online Mirror Descent (AdaOMD), which is the most natural extension of the Adaptive Gradient Descent used in \citet{yu2026constrained} to the MDP setting, typically yields the following regret analysis: given the diameter $D$ and the gradients $\nabla_k$,\footnote{For simplicity, the provided regret bounds are written for online convex optimization rather than MDPs.}
\begin{equation*}
    \Regret_{\mathrm{AdaOMD}}(K) \lesssim \underbrace{D}_{\substack{\text{unbounded} \\ \text{for MDPs}}}\sqrt{\sum_{k=1}^K ||\nabla_k||^2}.
\end{equation*}
Here, $\Regret_{\textrm{AdaOMD}}(K)$ requires the diameter parameter $D$, which is unbounded in the MDP setting. On the other hand, we observe that the regret bound of adaptive FTRL can be written as follows:
\begin{equation*}
\Regret_{\mathrm{AdaFTRL}}(K) \lesssim \log |\calA|\sqrt{\sum_{k=1}^K ||\nabla_k||^2}.
\end{equation*}
For comparison, $\Regret_{\mathrm{AdaFTRL}}$ depends on $\log|\calA|$ instead of $D$. We then show that adaptive FTRL can be successfully applied to the linear CMDP setting. A detailed analysis is provided in Appendix~\ref{appendix:adpative FTRL}.

To resolve the second challenge, we clip the dual variable at a polynomial in $K$ so that the na\"ive bound on $Y_k$ becomes polynomial in $K$. Given this, the log covering number in \eqref{eq:naive Yk} becomes $d^2\log K$ rather than $d^2K$. However, this modification leads to an additional issue: the clipping step may affect the analysis of constraint violation, since it may change the dual variable dynamics. We resolve this issue by showing that such clipping does not occur once we choose a sufficiently large clipping parameter (\Cref{lem:no clip}). This allows us to complete the violation analysis successfully.

\section{Algorithm}\label{sec:algorithm}
In this section, we present our rate-optimal algorithm for adversarial linear CMDPs (\Cref{alg:main}). We first describe the main components of the algorithm and explain their roles, and then state the resulting theoretical guarantees. Finally, we compare our approach with prior work in terms of both guarantees and algorithmic design.

\subsection{Description of Algorithm~\ref{alg:main}}
\begin{algorithm}[!t]
\caption{Rate-Optimal Algorithm for Adversarial Linear CMDPs}\label{alg:main}
\textbf{Input:} $\delta \in (0,1)$; $\beta_p=187dH\sqrt{\log(2HK/\delta)}$; $\beta_\phi=4\beta_p\log K$; $\beta_Q=2H$; $\beta_R = 2H^2(\log|\calA|+20)\sqrt{dH\log(2K)}$; $\gamma= \min\{1/(36H\sqrt{K}), 1/(6\beta_R)\}$\\
\textbf{Initialization:} $m\leftarrow 0$; $\Phi(x) = e^{\gamma x}-1$; $Z_1\leftarrow 0$; $Y_1 \leftarrow \Phi'(Z_1)$; $\bLambda_h^1 \leftarrow \Id \ \forall h\in [H]$; $\hat Q_{f,H+1}^k(\cdot,\cdot),\hat Q_{g,H+1}^k(\cdot,\cdot),\hat V_{f,H+1}^k(\cdot), \hat V_{g,H+1}^k(\cdot)=0$, $\forall k \in [K]$
\begin{algorithmic}[1]
    \For{$k=1, \ldots, K$}
        \If{$k=1$ or $\exists h' \in [H]$ such that $\det(\bLambda_{h'}^k) \geq 2\det(\bLambda_{h'}^{k(m)})$} \label{line:epoch start}
            \State $m \leftarrow m+1$ and $k(m) \leftarrow k$ 
            \State $\pi_h^{k(m)}(\cdot|s) \leftarrow \pi_{\text{unif}}(\cdot|s), \ \forall h \in [H]$ \label{line:policy reset}
            \State $\bar\bphi_h^m(\cdot,\cdot) = \bphi(\cdot,\cdot)\sigma(-\beta_\phi ||\bphi(\cdot,\cdot)||_{(\bLambda_h^{k(m)})^{-1}} + \log K), \ \forall h\in [H]$ \label{line:feature contraction}
        \EndIf
        \State Execute $\pi^k$ and observe $\{s_h^k,a_h^k,\btheta_{f,h}^k, \btheta_{g,h}^k\}_{h=1}^H$ \label{line:policy execute}
        \For{$h=H,\ldots,1$}
            \State $\bLambda_h^{k+1} \leftarrow \bLambda_h^k+ \bphi(s_h^k,a_h^k)\bphi(s_h^k,a_h^k)^\top$ \label{line:estimate start}
            \For{$\ell = f,g$}
                \State $\hat\bmu_h^k \hat V_{\ell,h+1}^k \leftarrow (\bLambda_h^k)^{-1}\sum_{\tau \in [k-1]}\bphi(s_h^\tau, a_h^\tau) \hat V_{\ell,h+1}^k(s_{h+1}^\tau)$ \label{eq:muV estimate}
                \State $\hat Q_{\ell,h}^k(\cdot,\cdot) \leftarrow \bar\bphi_h^m(\cdot,\cdot)^\top \left[ \btheta_{\ell,h}^k + \hat\bmu_h^k \hat V_{\ell,h+1}^k\right] - \beta_p ||\bar\bphi_h^m(\cdot,\cdot)||_{(\bLambda_h^{k(m)})^{-1}}$
                \State $\hat V_{\ell,h}^k (\cdot)\leftarrow \sum_{a\in \calA}\pi_h^k(a|\cdot) \hat Q_{\ell,h}^k(\cdot,a)$ \label{line:estimate end}
            \EndFor
        \EndFor
            \State Update the policy $\pi_h^{k+1}$ for all $h \in [H]$: \label{line:FTRL start}
            \[
                \pi_h^{k+1}(\cdot|s) \in \argmin_{\pi(\cdot|s)\in \Delta(\calA)}\ \psi_k(\pi(\cdot|s)) + \sum_{\tau=k(m)}^k (\hat Q_{f,h}^\tau(s,\cdot) + Y_\tau\hat Q_{g,h}^\tau(s,\cdot))^\top \pi(\cdot|s)
            \]
            \hspace{5mm}
            where $\psi_k(\pi(\cdot|s)) = \alpha_k\sum_{a\in\calA}\pi(a|s)\log\pi(a|s)$ and $\alpha_k = \frac{\beta_Q}{4}\sqrt{\sum_{\tau=k(m)}^{k}(1+Y_\tau)^2}.$\label{line:FTRL end}
        
        \State $R_k = 108\gamma H^2 + \beta_R \frac{Y_k}{\sqrt{\sum_{\tau=1}^{k}Y_\tau^2}}$ \label{line:Rk}
        \State $Z_{k+1} = Z_k + \hat V_{g,1}^k(s_1) - b - R_k$\label{line:Zk}
        \State $Y_{k+1} = \min\{\Phi'(Z_{k+1}), K^8\}$ \label{line:Yk}
    \EndFor
\end{algorithmic}
\end{algorithm}
From Lines~\ref{line:epoch start}-\ref{line:estimate end}, the algorithm proceeds through epoch-wise policy updates and value estimation. A new epoch begins when $\det(\bLambda_{h'}^k) \geq 2\det(\bLambda_{h'}^{k(m)})$ for some $h'$, where $k(m)$ denotes the initial episode of the epoch $m$. At the beginning of each epoch, the policy is reset to the uniform policy $\piunif$, and the contracted feature $\bar\bphi_h^m$ is updated using the sigmoid function $\sigma(z) = 1/(1+e^{-z})$ for any $z\in \bbR$. This contracted feature is then used to define the $Q$-function and value function estimates, ensuring their boundedness without explicit clipping such as $\max\{\cdot,0\}$. In each episode $k$, the agent executes the policy $\{\pi_h^k\}_{h=1}^H$, observes the trajectory $\{s_h^k,a_h^k\}_{h=1}^H$ and the full-information feedback $\{\btheta_{f,h}^k,\btheta_{g,h}^k\}_{h=1}^H$, and updates $\bLambda_h^{k+1}$ and the estimates $\hat Q_{f,h}^k,\hat Q_{g,h}^k,\hat V_{f,h}^k,\hat V_{g,h}^k$. Importantly, avoiding explicit clipping yields a simpler policy parameterization, which is crucial for the covering number analysis and is a standard technique in adversarial linear MDPs in both the unconstrained~\citep{cassel2024warm} and constrained~\citep{yu2026primaldual} settings. The precise parameterizations are provided in Appendix~\ref{appendix:param and cov}.

\textbf{Adaptive FTRL (Line~\ref{line:FTRL start}):}
To update the policy $\pi_h^{k+1}$, we use adaptive FTRL to accommodate the policy space. Here, adaptive FTRL minimizes the sum of the linearized losses $\sum_{\tau=k(m)}^k (\hat Q_{f,h}^\tau(s,\cdot) + Y_\tau \hat Q_{g,h}^\tau(s,\cdot))^\top \pi(\cdot|s)$ together with the adaptive regularizer $\psi_k$. This differs from the Euclidean setting of \citet{yu2026constrained}, whose analysis relies on a bounded $\ell_2$-diameter of the feasible primal domain. This is a crucial step because the adaptive FTRL update does not require bounding the diameter of the policy space with respect to KL divergence. We then show that it can be combined with the feature contraction technique used in our value function estimation.

\textbf{Adaptive Regularizer for Dual Update (Lines~\ref{line:Rk}-\ref{line:Zk}):}
To update the proxy dual variable $Z_{k+1}$, we incorporate the adaptive regularizer $R_k$ by subtracting it in the update of the proxy dual variable $Z_{k+1}$. Intuitively, this subtraction helps stabilize the dual process. For comparison, \citet{yu2026primaldual} also introduces a regularization term in the dual update, but their resulting Lyapunov drift analysis contains KL divergence terms and therefore requires policy mixing to control them. In contrast, with our choice of $R_k$, the resulting drift upper bound does not involve KL divergence terms. This allows the stability analysis to be carried out without relying on policy mixing, which is crucial in the linear CMDP setting.

\textbf{Clipping for Dual Update (Line~\ref{line:Yk}):}
To update the dual variable $Y_{k+1}$, we evaluate the derivative of the Lyapunov function $\Phi'$ at the proxy dual variable $Z_{k+1}$ and then clip the resulting value at $K^8$. Applying $\Phi'$ in the dual update is standard in the literature~\citep{sinha2024optimal}, while the clipping step is introduced to control the covering number in our linear CMDP setting. In our covering number analysis, a polynomial upper bound on $Y_{k+1}$ in terms of $K$ is necessary, whereas $\Phi'(Z_{k+1})$ can be exponentially large in $K$ because $\Phi$ is an exponential function. Clipping $\Phi'(Z_{k+1})$ at $K^8$ therefore prevents this na\"ive exponential bound from entering the covering number. Importantly, we later show that the clipping is not activated with high probability, which allows the violation analysis to proceed with the original, unclipped dual dynamics.

The computational complexity of \Cref{alg:main} is $\bigO(d^3HK + d^2|\calA|HK^2)$, which is comparable to that of \citet{yu2026primaldual} and importantly is independent of $|\calS|$. Specifically, computing the determinants requires $\bigO(d^3HK)$ operations, while the feature contraction and value function estimation require $\bigO(d^2|\calA|HK^2)$ operations. Moreover, since adaptive FTRL admits a closed-form update, we do not need to solve a convex optimization problem (Appendix~\ref{appendix:adpative FTRL}); its computational cost, as well as that of the dual update, is not dominant.

\subsection{Main Result of Algorithm~\ref{alg:main}}
We present the following theorem, which establishes the rate-optimal guarantees of \Cref{alg:main}.
\begin{theorem}\label{thm:main}
    Suppose that we run \Cref{alg:main}. Let $H\geq 2$, and let $\delta\in (0,1)$. With probability at least $1-\delta$, we have
    \begin{align*}
        \Regret(K),\ \Violation(K) = \tbigO\left(\sqrt{d^3H^4K}+ \sqrt{dH^5K} +d^3H^4\right).
    \end{align*}
    where $\tbigO(\cdot)$ hides polylogarithmic factors in $d,H,|\calA|,K,1/\delta$.
\end{theorem}
Note that the regret and constraint violation guarantees in \Cref{thm:main} have comparable orders in terms of $d,H,K$ to those of \citet{cassel2024warm}, which studied adversarial linear (unconstrained) MDPs. Indeed, these orders are the best-known guarantees without applying variance-related techniques. However, even in the unconstrained setting, no algorithm for adversarial linear MDPs has been proposed that achieves near-optimal guarantees, i.e., $\tbigO(d\sqrt{H^3 K})$~\citep{he2023nearly}. Therefore, we leave achieving optimal dependence on $d$ and $H$ as future work in our setting.

\paragraph{Comparison}
From the perspectives of both results and algorithms, we compare our algorithm with~\citet{yu2026primaldual}, which is the previous state-of-the-art algorithm in a similar setting, i.e., adversarial losses and stochastic constraints. From the perspective of results, the improvements of \Cref{alg:main} can be summarized as follows: (i) improved guarantees of $\tbigO(\sqrt{K})$, (ii) no Slater condition is required, and (iii) adversarial constraints can be handled. On the other hand, their algorithm can only attain $\tbigO(K^{3/4})$, requires Slater's condition, and can only handle stochastic constraints.

From the algorithmic perspective, the differences can be summarized as follows: (i) we adopt adaptive FTRL, while the policy update of \citet{yu2026primaldual} can be viewed as FTRL with a constant step size, which is easier to handle than in our case; (ii) we adopt an adaptive regularizer for the dual update, which plays a key role in deriving a KL divergence-free drift analysis; and (iii) we adopt upper clipping for the dual variable, while they do not. Hence, in our case, the violation analysis is more challenging, since it requires additional techniques to show that the clipping step does not affect the analysis.

\section{Analysis of Algorithm~\ref{alg:main}}\label{sec:analysis}
This section presents the analysis of \Cref{alg:main} for establishing the rate-optimal guarantees in \Cref{thm:main}. The overview of our analysis is as follows. We first introduce the key relation~\eqref{eq:Regret + Phi}, inspired by \citet{yu2026constrained}, which bounds the sum of regret and $\Phi(Z_{K+1})$. To establish this relation, we provide two key ingredients, \eqref{eq:ingredient 1} and \eqref{eq:ingredient 2}, and then explain how they can be obtained in the linear CMDP setting. The details of the proofs are provided in Appendix~\ref{appendix:proofs}.

Before presenting the main analysis, we introduce the good event $\calE_g$. Intuitively, this event captures the accuracy and boundedness of the value function estimates and holds with high probability. Although establishing such a good event is a standard step in the online MDP literature, our dual variable clipping is useful here, as it ensures a polynomial upper bound on the dual variable and thereby keeps the resulting covering number under control. The formal definition and the proofs that $\calE_g$ holds with high probability are given in Appendix~\ref{appendix:good event}. 

In our analysis, the key relation we need to establish is as follows: the sum of $\Regret(K)$ and the Lyapunov evaluation of $Z_{K+1}$ is bounded by $\tbigO(\sqrt{K})$, i.e.,
\begin{align}\label{eq:Regret + Phi}
\Regret(K) + \Phi(Z_{K+1}) \lesssim \sqrt{K}.
\end{align}
This relation is useful because the regret and violation upper bounds naturally follow from it. Specifically, since we take $\Phi(x) = e^{\gamma x}-1\geq -1$ for any $x$, it directly gives us that $\Regret(K) = \tbigO(\sqrt{K})$. To obtain an upper bound on the violation, we observe that
\begin{align*}
    \Violation(K) &= \sum_{k=1}^K (V_{g^k,1}^{\pi^k}(s_1) - \hat V_{g,1}^k(s_1)) + \sum_{k=1}^K (\hat V_{g,1}^k(s_1)-b)\\
    &= \sum_{k=1}^K (V_{g^k,1}^{\pi^k}(s_1) - \hat V_{g,1}^k(s_1)) + Z_{K+1}+ \sum_{k=1}^K R_k.
\end{align*}
The first sum, which consists of the bias terms, can be bounded using a standard argument. On the other hand, an upper bound on $Z_{K+1}$ can be obtained from \eqref{eq:Regret + Phi}, while an upper bound on the sum of $R_k$ can be obtained from the following lemma.
\begin{lemma}\label{lem:R sum}
    Assume that the good event $\calE_g$ holds. Let $H\geq 2$, and let $K \geq c dH(\log|\calA|+20)\log(2H/\delta)$ for some absolute constant $c$. We have $\sum_{k=1}^K R_k = \tbigO(\sqrt{K}).$
\end{lemma}
So far, we have shown why establishing \eqref{eq:Regret + Phi} is crucial in our analysis. Next, to explain how to obtain this relation, we introduce two ingredients. The first is a bound on the sum of regret and the optimistic constraint violation weighted by the dual variable as follows:
\begin{align}\label{eq:ingredient 1}
    \sum_{k=1}^K \left(V_{f^k,1}^{\pi^k}(s_1) - V_{f^k,1}^{\pi^\star}(s_1)\right) + \sum_{k=1}^K Y_k \left(\hat V_{g,1}^{k}(s_1) - b\right) \lesssim \sqrt{K} + \sqrt{\sum_{k=1}^K Y_k^2}.
\end{align}
The second is an upper bound on the Lyapunov drift as follows:
\begin{align}\label{eq:ingredient 2}
    \Phi(Z_{k+1}) - \Phi(Z_k) \lesssim Y_k\left(\hat V_{g,1}^k(s_1)-b-R_k\right).
\end{align}
Indeed, \eqref{eq:ingredient 1} and \eqref{eq:ingredient 2} can be combined to derive \eqref{eq:Regret + Phi} as follows. Summing \eqref{eq:ingredient 2} over $k$ gives $\Phi(Z_{K+1}) \lesssim \sum_{k=1}^K Y_k (\hat V_{g,1}^k(s_1) - b) - \sum_{k=1}^K Y_k R_k$. Combining this with \eqref{eq:ingredient 1} then yields \eqref{eq:Regret + Phi}. This is because the choice of $R_k$ allows us to show that $\sum_{k=1}^K Y_k R_k \approx \sqrt{\sum_{k=1}^K Y_k^2}$, and then the two corresponding terms then cancel each other out.

Now, we explain how to derive \eqref{eq:ingredient 1} and \eqref{eq:ingredient 2} in the linear CMDP setting. To show \eqref{eq:ingredient 1}, we prove the adaptive FTRL lemma for the policy space as below. Importantly, we observe that the regret bound for adaptive FTRL does not require a uniform bound on the feasible set of the primal solution. This observation is essential, since the corresponding term in our case is the diameter of the policy space in terms of KL divergence, which is unbounded. 
\begin{lemma}[Informal\footnote{For simplicity, we focus on the formulation while omitting expectations with respect to the trajectory on the left-hand side. For the precise statement, we refer to \Cref{lem:FTRL lemma policy appendix}.}]\label{lem:FTRL lemma policy}
    Assume that the good event $\calE_g$ holds. Then we have for all $k \in [K]$,
    \begin{align*}
        \sum_{k=1}^K\sum_{h=1}^H \left(\hat Q_{f,h}^k(s_h^k,\cdot) + Y_k\hat Q_{g,h}^k(s_h^k,\cdot)\right)^\top(\pi_h^k(\cdot|s_h^k) - \pi_h^\star(\cdot|s_h^k)) \lesssim \sqrt{K} + \sqrt{\sum_{k=1}^K Y_k^2}.
    \end{align*}
\end{lemma}
For comparison, a na\"ive extension of the previous work requires additional techniques to bound the KL divergence, e.g., policy mixing, which is not a promising method in the linear CMDP setting, due to the covering number issue. Roughly, combining this lemma with the standard value difference lemma (\Cref{lem:extended value diff}) then yields \eqref{eq:ingredient 1}.

To show \eqref{eq:ingredient 2}, we may start with the following drift bound, which follows from basic algebraic properties of the exponential function together with the choice of $R_k$.
\begin{lemma}
\label{lem:drift}
Assume that the good event $\calE_g$ holds. Then, 
\begin{align*}
    \Phi(Z_{k+1}) - \Phi(Z_k)
    \leq
    \Phi'(Z_k)\left(\hat V_{g,1}^k(s_1)-b-R_k/2\right).
\end{align*}
\end{lemma}
Although this lemma shares a similar formulation with \eqref{eq:ingredient 2}, there is an important distinction due to the discrepancy between $Y_k$ and $\Phi'(Z_k)$. In particular, since $Y_k$ is obtained by clipping $\Phi'(Z_k)$ at $K^8$, \Cref{lem:drift} does not directly imply \eqref{eq:ingredient 2}. Recall that clipping is introduced to bound the covering number in the linear CMDP setting, whereas it does not appear in other constrained online learning settings. To address this issue, we introduce the following lemma, which shows that such clipping does not occur. Consequently, by combining Lemmas~\ref{lem:drift} and~\ref{lem:no clip}, we deduce \eqref{eq:ingredient 2}.
\begin{lemma}\label{lem:no clip}
Assume that the good event $\calE_g$ holds. Let $H\geq 2$, and let $K \geq c dH(\log|\calA|+20)\log(2H/\delta)$ for some absolute constant $c$. Then, for all $k\in [K+1]$, we have
\begin{equation*}
    \Phi'(Z_k) = Y_k.
\end{equation*}
\end{lemma}

{
\color{alert}
\section{Extension to Hard Constraint Violation}
\label{sec:extension}
In this section, we extend \Cref{alg:main} to obtain a hard constraint violation guarantee. Recall that \Cref{alg:main}, together with the analysis in \Cref{sec:analysis}, provides guarantees for $\Regret(K)$ and $\Violation(K)$, where $\Violation(K)$ allows constraint violations to cancel across episodes. Here, we introduce a few modifications to \Cref{alg:main} and establish a guarantee for hard constraint violation, which does not permit such cancellation. The full algorithm, \Cref{alg:hard}, and its analysis are provided in Appendix~\ref{appendix:extension}. The following theorem states the resulting guarantees.
\begin{theorem}\label{thm:hard}
    Suppose that we run \Cref{alg:hard}. Let $H\geq 2$, and let $\delta\in (0,1)$. With probability at least $1-\delta$, we have
    \begin{align*}
        \Regret(K),\ \VioHard(K) = \tbigO\left(\sqrt{d^3H^4K}+ \sqrt{dH^5K} +d^3H^4\right).
    \end{align*}
    where $\tbigO(\cdot)$ hides polylogarithmic factors in $d,H,|\calA|,K,1/\delta$.
\end{theorem}

To describe the algorithmic modifications, we first identify the key analytical ingredients required to establish a hard constraint violation guarantee. In particular, we observe that $\VioHard(K) \leq \sum_{k=1}^K [V_{g^k,1}^{\pi^k}(s_1) - \hat V_{g,1}^{k}(s_1)]_+ + \sum_{k=1}^K [\hat V_{g,1}^{k}(s_1) -b]_+$, where the first summation on the right-hand side can be bounded using standard arguments. Thus, it remains to control the second summation. To this end, we establish the following two inequalities, which are analogues of \eqref{eq:ingredient 1} and \eqref{eq:ingredient 2}:
\begin{align}
    &\sum_{k=1}^K \left(V_{f^k,1}^{\pi^k}(s_1) - V_{f^k,1}^{\pi^\star}(s_1)\right) + \sum_{k=1}^K Y_k {\color{red}\left[\hat V_{g,1}^{k}(s_1) - b\right]_+} \lesssim \sqrt{K} + \sqrt{\sum_{k=1}^K Y_k^2}, \label{eq:ingredient 1 hard}\\
    &\Phi(Z_{k+1}) - \Phi(Z_k) \lesssim Y_k({\color{red}[\hat V_{g,1}^k(s_1)-b]_+}-R_k). \label{eq:ingredient 2 hard}
\end{align}

To establish \eqref{eq:ingredient 1 hard} and \eqref{eq:ingredient 2 hard}, we next explain two algorithmic modifications---(i) the policy update and (ii) the proxy dual variable update. The main obstacle in deriving \eqref{eq:ingredient 1 hard} is $[\cdot]_+$ on the left-hand side, which prevents a direct application of the value difference lemma. To address this issue, we define $\bar Y_k = Y_k\cdot \Ind\{\hat V_{g,1}^k(s_1) - b \geq 0\}$. Then,
\begin{align}\label{eq:hard observe}
\begin{aligned}
    Y_k [\hat V_{g,1}^{k}(s_1) - b]_+ 
    &= \bar Y_k (\hat V_{g,1}^{k}(s_1) - b).
\end{aligned}
\end{align}
This identity allows us to replace the positive-part term $Y_k [\hat V_{g,1}^{k}(s_1) - b]_+$ in \eqref{eq:ingredient 1 hard} by the linear term $\bar Y_k (\hat V_{g,1}^{k}(s_1) - b)$, thereby recovering a form compatible with the value difference lemma. Accordingly, we use $\bar Y_k$ instead of $Y_k$ in the policy update. Specifically, letting $\alpha_k = (\beta_Q/4)\sqrt{\sum_{\tau=k(m)}^{k}(1+{\color{red}\bar Y_\tau})^2}$, the policy optimization then becomes
\[
    \pi_h^{k+1}(\cdot|s) \in \argmin_{\pi(\cdot|s)\in \Delta(\calA)}\ \psi_k(\pi(\cdot|s)) + \sum_{\tau=k(m)}^k (\hat Q_{f,h}^\tau(s,\cdot) + {\color{red}\bar Y_\tau}\hat Q_{g,h}^\tau(s,\cdot))^\top \pi(\cdot|s).
\]

Next, to obtain \eqref{eq:ingredient 2 hard}, we modify the proxy dual update so that it accumulates only positive estimated constraint violations:
\begin{align*}
    Z_{k+1} = Z_k + {\color{red}[\hat V_{g,1}^k(s_1)-b]_+} -R_k.
\end{align*}
}

\section{Conclusion}

This paper studies adversarial linear CMDPs, where both loss and constraint functions are adversarially chosen under full-information feedback. We propose a computationally efficient primal-dual algorithm that achieve rate-optimal $\tbigO(\sqrt{K})$ regret and constraint violation guarantees without requiring Slater's condition. The key idea is to develop a new drift analysis that avoids policy mixing, thereby resolving the conflict between stabilizing the dual process and controlling the covering number in linear CMDPs. To make this approach applicable to the linear CMDP setting, we further employ adaptive FTRL for the policy update and a clipping mechanism for the dual variable.
{\color{alert}
We further extend the algorithm to achieve the same $\tbigO(\sqrt{K})$ guarantees for regret and hard constraint violation, which does not allow cancellation across episodes.
One limitation of this work is that we consider full-information feedback rather than bandit feedback. Developing computationally efficient algorithms for adversarial linear CMDPs under bandit feedback is an important direction for future work.
}

\bibliographystyle{plainnat}
\bibliography{ref}

\newpage
\appendix
\newpage
\startcontents[appendix]
\section*{Appendix Contents}
\printcontents[appendix]{}{1}{}

\newpage
\section{Notation}\label{appendix:notation}
We summarize the notation used throughout the paper.
\begin{table}[H]
    \centering
    \caption{Summary of Notation}
    \label{tab:notation}
    \begin{tabular}{c|l}
        $\bmu_h$ & the parameters for transition kernel \\
        $\btheta_{f,h}^k,\btheta_{g,h}^k$ & the parameters for loss and constraint functions \\
        $\bphi$ & the known feature map \\
        $\bar\bphi_h^m$ & the contracted feature \\
        $V_{f^k,h}^{\pi}, \hat V_{f,h}^k$ & the value function and its estimate \\
        $Q_{f^k,h}^{\pi}, \hat Q_{f,h}^k$ & the $Q$-function and its estimate \\
        $\bar\bbP_h^m$ & the transition kernel of a contracted CMDP in epoch $m \in E$\\
        $\bar f_h^k, \bar g_h^k$ & the loss and constraint functions of a contracted CMDP \\ 
        $E$ & the set of all epochs \\
        $k(m)$ & the first episode of epoch $m\in E$ \\
        $K(m)$ & the set of all episodes of epoch $m \in E$, i.e., $K(m)=\{k(m), \ldots, k(m+1)-1\}$ \\
        $E(k)$ & the set of epochs up to and including episode $k\in [K]$, \\
                & i.e., $E(k) =\{m \in E: k(m) \leq k\}$ \\
        $\gamma$ & the parameter for $\Phi$ \\
        $\Phi$ & the Lyapunov function $\Phi(x) = e^{\gamma x}-1$ \\
        $\alpha_k$ & the adaptive coefficient for $\psi_k$ \\
        $\psi_k$ & the negative entropy regularizer for FTRL with the adaptive coefficient $\alpha_k$ \\
        $R_k$ & the adaptive regularizer for the dual update \\
        $Z_k$ & the proxy dual variable\\
        $Y_k$ & the dual variable\\
        $\cbias$ & the absolute constant given by \Cref{lem:bias} \\
        $\calF_{k,h}$ & the $\sigma$-algebra generated by information up to $k,h$, \\
        & i.e., $\calF_{k,h} = \begin{cases}
            \sigma(\{s_1^1,a_1^1, \ldots,s_h^k,a_h^k\} \cup \{\btheta_{f,1}^1, \btheta_{g,1}^1, \ldots, \btheta_{f,H}^{k-1},\btheta_{g,H}^{k-1}\}), & h < H, \\
            \sigma(\{s_1^1,a_1^1, \ldots,s_H^k,a_H^k\} \cup \{\btheta_{f,1}^1, \btheta_{g,1}^1, \ldots, \btheta_{f,H}^{k},\btheta_{g,H}^{k}\}), & h=H.
        \end{cases}$ \\
    \end{tabular}
\end{table}

\section{Related Work}\label{appendix:related work}

\paragraph{Tabular CMDPs}
Extensive work has studied learning in tabular CMDPs, in which the state and action spaces are finite. \citet{efroni2020exploration} study exploration--exploitation in unknown tabular CMDPs using optimistic planning and primal--dual methods. With a known strictly safe baseline policy, \citet{liu2021learning, bura2022dope, yu2025improved} obtain sublinear regret while ensuring that every deployed policy satisfies its expected cumulative-cost constraint with high probability. \citet{muller2024truly, stradi2025optimal} study strong regret and strong constraint violation, which disallow cancellation across episodes. Model-free learning is studied by \citet{wei2022triple} in finite-horizon CMDPs, while \citet{wei2022provably, chen2022learning} consider the infinite-horizon average-reward setting. For adversarial objectives, \citet{qiu2020upper} study adversarial losses and stochastic constraints. \citet{stradi2024online, stradi2025policy} allow both objectives and constraints to be adversarial under full-information and bandit feedback, respectively, with adversarial-constraint guarantees relative to a fraction of the optimal reward. \citet{stradi2025learning} study adversarial losses and stochastic constraints under bandit feedback, obtaining cancellation-free violation guarantees. More recently, \citet{zhu2025optimistic, roy2026optimal} establish sublinear regret and strong constraint violation for adversarial constraints without Slater's condition. However, these works are limited to the tabular setting, and their statistical guarantees typically depend on $|\calS|$, making them inappropriate for problems with large state spaces.

\paragraph{Stochastic CMDPs with Linear Function Approximation}
Several works have studied CMDPs with linear function approximation under stochastic settings. \citet{ding2021provably} develop primal--dual policy optimization for linear mixture CMDPs, while \citet{ghosh2022provably} propose a model-free primal--dual algorithm for linear CMDPs. For the latter model, \citet{ghosh2024towards} achieve $\widetilde{O}(\sqrt{K})$ regret and hard constraint violation without Slater's condition. \citet{peng2026towards} provide a computationally efficient algorithm achieving $\widetilde{O}(\sqrt{K})$ strong regret and strong constraint violation under Slater's condition, and $\widetilde{O}(K^{3/4})$ bounds without it. Under a known strictly safe baseline policy, \citet{kitamura2025provably} achieve sublinear regret and zero episode-wise constraint violation. Instantaneous safety, which constrains individual decisions rather than the expected total cost of a policy, is investigated for linear MDPs by \citet{amani2021safe, roknilamouki2025provably} and for linear mixture MDPs by \citet{shi2023near}. These works pursue zero instantaneous violation under suitable safety and structural assumptions. In contrast, \citet{wei2024safe} allow sublinear cumulative hard violation without requiring a known safe action at every state. Beyond episodic problems, \citet{ghosh2023achieving, yu2026learning} develop algorithms for average-reward linear CMDPs. Under a different interaction model, \citet{tian2024confident} study local planning in $q_\pi$-realizable CMDPs with local simulator access, providing sample-complexity rather than online-regret guarantees.

\paragraph{Adversarial and Non-Stationary CMDPs with Linear Function Approximation}
In contrast to stochastic settings, several works consider CMDPs whose components can vary over time. In particular, a related line of work studies non-stationary CMDPs under linear mixture and linear MDP assumptions~\citep{ding2023provably, wei2023provably}. These works allow changes in transitions, objectives, and constraints, and provide dynamic-regret and constraint-violation guarantees that depend on variation budgets. This differs from unrestricted adversarial objectives under fixed dynamics, where performance is compared against a fixed feasible policy. More closely related to our setting, \citet{yu2026near} study linear mixture CMDPs with adversarial rewards under full-information feedback and a fixed, known constraint function, achieving $\widetilde{O}(\sqrt{K})$ regret and constraint violation through regularized dual updates and variance-aware estimation. \citet{yu2026primaldual} study linear CMDPs with unknown transitions, adversarial losses under full-information feedback, and stochastic constraints under bandit feedback. Their primal--dual policy-optimization algorithm combines periodic policy mixing with a regularized dual update, achieving $\widetilde{O}(K^{3/4})$ regret and constraint violation. We emphasize that the algorithm of \citet{yu2026near} can become computationally intractable when the state space is large, as its computational complexity depends on $|\calS|$, while \citet{yu2026primaldual} is statistically suboptimal in terms of $K$. In contrast, our algorithm is model-free, its computational complexity is independent of $|\calS|$, and it achieves optimal guarantees in terms of $K$, i.e.,  $\tbigO(\sqrt{K})$.

\section{Additional Discussions}\label{appendix:discussion}

\paragraph{Why is policy mixing inevitable in \citet{yu2026primaldual}?}
Previously, we described policy mixing as a bottleneck in \citet{yu2026primaldual} that leads to suboptimal guarantees. Here, we investigate why policy mixing is inevitable under their analysis. The main reason is that the uniform upper bound on $Y_k$ across $k$ relies on bounding the KL divergence between policies, which can only be guaranteed when mixing is applied.

To elaborate, we first see how policy mixing works. Specifically, the policy mixing used in their work can be expressed as follows: for some small constant $\theta\in (0,1)$,
\begin{align*}
\pi_h^k(\cdot|s) \leftarrow (1-\theta)\pi_h^k(\cdot|s) + \theta\piunif
\quad\text{and then}\quad
\pi_h^{k+1}(\cdot|s) \propto \pi_h^k (\cdot|s)\exp(-\alpha Q_{h}^k(s,\cdot)).
\end{align*}
This step is useful for ensuring that the KL divergence with respect to $\pi_h^k$ is small, which is essential for stabilizing the dual process. Indeed, the following simplified drift analysis supports this:
\begin{align}\label{eq:drift kl}
Y_{k+1}^2 - Y_k^2 \lesssim -\eta\gamma Y_k + \frac{\eta}{\alpha}\bbE\left[\sum_{h=1}^H\left(\KL(\bar\pi || \pi_h^k) - \KL(\bar\pi||\pi_h^{k+1}) \right)\right] + C,
\end{align}
where $\eta$ denotes the step size for the dual update, $\gamma$ and $\bar\pi$ denote the Slater constant and Slater policy, respectively, $\KL(\cdot||\cdot)$ denotes the KL divergence, and $C$ absorbs additional positive terms.

We emphasize that the mixing step provides a uniform bound on the KL divergence terms in \eqref{eq:drift kl}, as $\KL(\bar\pi || \pi_h^k) \leq \log(|\calA|/\theta)$, leading to a uniform upper bound on $Y_{k+1}^2 - Y_k^2$. This is a key step in their analysis for obtaining sublinear $\Violation(K)$. On the other hand, without the policy mixing step, the KL divergence term in \eqref{eq:drift kl} can become arbitrarily large, making the dual process unstable and leading to a linear $\Violation(K)$ bound.

Some readers may wonder whether we can utilize the telescoping structure between the two KL divergence terms in \eqref{eq:drift kl}. In particular, if we sum \eqref{eq:drift kl} over $k$, then the KL divergence terms take the form $\KL(\bar\pi||\pi_h^1) - \KL(\bar\pi||\pi_h^{K+1})$, which is small since we set $\pi_h^1$ to the uniform policy and KL divergence is nonnegative. Then, we may obtain a compact bound on $Y_{K+1}$. While this approach sounds reasonable, as $\Violation(K)$ highly depends on $Y_{K+1}$, a uniform bound on $Y_k$ over $k$ is still required, e.g., in Lemma 5 of \citet{yu2026primaldual}, which bounds an online mirror descent (OMD) regret term. Specifically, if we use the na\"ive upper bound $Y_k=\tbigO(\sqrt{K})$, then the OMD regret term becomes $\tbigO(K)$, implying that such a modified approach still yields suboptimal guarantees. For these reasons, the policy mixing step is unavoidable within their analysis.

\section{Additional Preliminaries}
\subsection{Contracted MDP}\label{appendix:contracted mdp}

Let $\bar\calM = (\calS, \calA, H, \{\bar \ell_h\}_{h=1}^H,s_1,\{\bar\bbP_h\}_{h=1}^H, \rho)$ denote a contracted MDP. Basically, while it shares a similar definition with an MDP introduced in \Cref{sec:problem setting}, the key difference is $\rho$. Specifically, $\rho: \calS \times\calA\times[H]\to [0,1]$ controls how much an MDP is contracted due to the following definition:
\begin{align*}
    \bar \ell_h(s,a) = \rho(s,a,h) \ell_h(s,a),\quad \bar \bbP_h(s'\mid s,a) = \rho(s,a,h) \bbP_h(s'\mid s,a)
\end{align*}
where $\bar\ell_h(s,a)$ and $\bar\bbP_h$ denote the contracted loss and transition kernel. Note that since we may have $\sum_{s'}\bar\bbP_h(s'|s,a) \leq 1$, we call it a sub-probability measure. Nevertheless, under these definitions, we have the following useful lemmas.
\begin{lemma} [Lemma 2 of \citet{cassel2024warm}] \label{lem:cassel 2}
    For any $\rho:\calS \times \calA\times[H] \to [0,1]$, $\pi\in\Pi$, $h\in[H]$, $s\in\calS$, and $\ell:\calS\times\calA \times[H] \to[0,1]$, we have $\bar V_{\ell,h}^\pi(s;\rho) \leq V_{\ell,h}^\pi(s)$.
\end{lemma}
\begin{lemma}[Lemma 1 of \citet{shani2020optimistic} and Lemma 14 of \citet{cassel2024warm}]\label{lem:extended value diff}
Let $\pi, \hat\pi$ be two policies, and let $\mathcal{M} = (\mathcal{S}, \mathcal{A}, H, s_1, \{\ell_h\}_{h=1}^H, \{\bbP_h\}_{h=1}^H)$ be a (possibly sub) MDP. 
For all $h\in [H]$, let $\hat{Q}_{\ell, h}:\calS\times\calA \rightarrow \bbR$ be an arbitrary function, and let
$\hat{V}_{\ell,h}(s) = \hat{Q}_{\ell,h}(s, \cdot)^\top \hat\pi_h(\cdot \mid s)$ for all $s\in\calS$.
Then,
\begin{align*}
V_{\ell,1}^{\pi}(s_1) - \hat{V}_{\ell,1}(s_1)
&=  \bbE_{\bbP,\pi} \left[ \sum_{h=1}^H
  \hat{Q}_{\ell,h}(s_h, \cdot)^\top(\pi_h(\cdot \mid s_h) - \hat\pi_h(\cdot \mid s_h)) \,\middle|\, s_1\right] \\
&\quad+ \bbE_{\bbP,\pi} \left[ \sum_{h=1}^H
\ell_h(s_h, a_h) + \bbP_h\hat{V}_{\ell,h+1}(s_h, a_h) - \hat{Q}_{\ell,h}(s_h, a_h) 
\,\middle|\, s_1\right],
\end{align*}
where $V_{\ell,1}^\pi$ is the value function of $\pi$, and $\bbP_h\hat{V}_{\ell,h+1}(s, a) = \sum_{s'\in\calS} \bbP_h(s'\mid s,a)\hat V_{\ell,h+1}(s')$.
\end{lemma}

\subsection{Adaptive FTRL}\label{appendix:adpative FTRL}
In this subsection, we review adaptive FTRL for online convex optimization when the solution space is the probability simplex $\Delta_{n-1} = \{\bfx \in \bbR^n: \bfx_i \geq 0 \text{ for all $i$ and } \sum_{i=1}^n \bfx_i =1\}$. While these results are standard in the literature, we include them for completeness. Most of the proofs in this section are taken from Chapter 7 in \citet{orabona2019modern} with a few modifications. We consider the following algorithm.
\begin{algorithm}[!t]
\caption{Adaptive FTRL for online convex optimization with the probability simplex}\label{alg:ada FTRL}
\textbf{Initialize:} $\bfx_1 = (1/n)\bm 1$
\begin{algorithmic}
    \For{$k=1,\ldots, K$}
        \State Take $\bfx_k$ and observe $\ell_k(\bfx) = \nabla_k^\top \bfx$
        \State Update $\psi_k(\bfx) = \alpha_k\sum_{i=1}^n \bfx_i \log\bfx_i$ and $\alpha_k$ is defined as follows:
        \[
            \alpha_k = \frac{1}{4}\sqrt{\sum_{\tau=1}^k L_\tau^2} \quad\text{where}\quad ||\nabla_\tau||_\infty \leq L_\tau \; \forall \tau \in [k]
        \]
        \State Update $\bfx_{k+1}:$
        \begin{equation}\label{eq:FTRL}
            \bfx_{k+1} \in \argmin_{\bfx\in\Delta_{n-1}}\ \psi_k(\bfx) + \sum_{\tau=1}^k \nabla_\tau^\top \bfx
        \end{equation}
    \EndFor
\end{algorithmic}
\end{algorithm}

Note that the $\bfx_{k+1}$ defined by \eqref{eq:FTRL} is uniquely determined and its closed form is written as follows (see Example 28.3 in \citet{lattimore2020bandit}): for each $i\in [n]$,
\begin{align}\label{eq:closed form}
    \bfx_{k+1,i} = \frac{\exp\left(-\frac{1}{\alpha_k}\sum_{\tau=1}^k \nabla_{\tau,i}\right)}{\sum_{j=1}^n\exp\left(-\frac{1}{\alpha_k} \sum_{\tau=1}^k \nabla_{\tau,j}\right)}.
\end{align}
The following lemma establishes a regret bound under \Cref{alg:ada FTRL}. 
\begin{lemma}\label{lem: FTRL regret}
    Suppose that we run \Cref{alg:ada FTRL} and $L_1 > 0$. For any $\bfx^\star \in \Delta_{n-1}$ and $k \in [K]$, we have
    \begin{align}\label{eq:regret AdaFTRL}
        \Regret_{\textrm{AdaFTRL}}(k) \leq \left(\frac{\log n}{4} + 4\sqrt{2}\right)\sqrt{\sum_{\tau=1}^k L_\tau^2}.
    \end{align}
\end{lemma}
\paragraph{Comparison with AdaOMD} For comparison, we briefly introduce the analysis of AdaOMD, which is the most natural extension of Adaptive Gradient Descent to the probability simplex. In particular, the update of AdaOMD can be written as follows:
\begin{align}\label{eq:AdaOMD}
    \bfx_{k+1} \in \argmin_{\bfx\in\Delta_{n-1}}\quad \frac{1}{\eta_k}\KL(\bfx||\bfx_k) + \nabla_k^\top \bfx.
\end{align}
where $\eta_k \asymp 1/\sqrt{\sum_{\tau=1}^k ||\nabla_\tau||_\infty^2}$. Its regret bound is given as follows:
\begin{align}\label{eq:regret AdaOMD}
    \Regret_{\textrm{AdaOMD}}(K) \lesssim \sup_{\bfx,\bfy\in\Delta_{n-1}}\KL(\bfx||\bfy)\sqrt{\sum_{k=1}^K ||\nabla_k||_\infty^2}.
\end{align}
Note that since $\sup_{\bfx,\bfy\in\Delta_{n-1}}\KL(\bfx||\bfy)$ is unbounded, the regert bound of AdaOMD with over the probability simplex is unbounded. That said, while the right-hand side of \eqref{eq:regret AdaOMD} is indeed invalid, we show it for comparison with \eqref{eq:regret AdaFTRL}. The regret bound of AdaFTRL as in \eqref{eq:regret AdaFTRL} takes $\log n$ instead of $\sup_{\bfx,\bfy\in\Delta_{n-1}}\KL(\bfx||\bfy)$. Therefore, AdaFTRL is appropriate for MDPs in which we must handle the policy space.

\subsubsection{Proof of Lemma~\ref{lem: FTRL regret}}
Before providing the regret analysis of \Cref{alg:ada FTRL}, we first introduce several notations and lemmas. Throughout the section, for $k=1,\ldots, K$, we define
\begin{align*}
    &\bar\psi_k(\bfx) = \alpha_k \left(\sum_{i=1}^n \bfx_i \log \bfx_i + \log n\right),\\
    & F_{k+1}(\bfx) = \bar\psi_k(\bfx) + \sum_{\tau=1}^k \nabla_\tau^\top \bfx \quad\text{and} \quad F_1(\bfx) = 0
\end{align*}
Note that the minimum of $\sum_{i=1}^n \bfx_i\log \bfx_i$ is $-\log n$ when $\bfx = (1/n)\bm 1$, and thus $\bar\psi_k$ is nonnegative. Under this definition, we have the following useful relation, which follows from $\{\alpha_k\}_{k=1}^K$ is a nondecreasing sequence:
\[
    \bar\psi_{k-1}(\bfx) - \bar\psi_k(\bfx) \leq 0, \quad \forall \bfx\in\Delta_{n-1},\ k=2,\ldots, K
\]
Moreover, since the only difference between $\bar\psi_k$ and $\psi_k$ is the term $\alpha_k \log n$, which is independent of $\bfx$, we have $\bfx_{k+1} \in \argmin_{\bfx\in \Delta_{n-1}} F_{k+1}(\bfx)$ even when $F_{k+1}$ is defined in $\bar\psi_k$. Now we introduce the following lemma as a first step of regret analysis for \Cref{alg:ada FTRL}.

\begin{lemma}\label{lem 7.1}
    Suppose that we run \Cref{alg:ada FTRL} and $L_1 > 0$. For any $\bfx^\star \in \Delta_{n-1}$ and $k\in [K]$, we have
    \begin{align*}
        \sum_{\tau=1}^k \nabla_\tau^\top(\bfx_\tau -\bfx^\star)
        &\leq \bar\psi_k(\bfx^\star) + \sum_{\tau=1}^k \left(F_\tau(\bfx_\tau)-F_{\tau+1}(\bfx_{\tau+1})+\nabla_\tau^\top\bfx_\tau\right) 
    \end{align*}
\end{lemma}
\begin{proof}
    By the definition of $F_{k+1}$,
    \[
        \sum_{\tau=1}^k\nabla_\tau^\top \bfx^\star = F_{k+1}(\bfx^\star) -\bar\psi_k(\bfx^\star)
    \]
    Moreover, by telescoping, we have
    \[
        0=-F_1(\bfx_1) + \sum_{\tau=1}^k (F_\tau(\bfx_\tau) - F_{\tau+1}(\bfx_{\tau+1})) + F_{k+1}(\bfx_{k+1}).
    \]
    Then it follows that
    \begin{align*}
        &\sum_{\tau=1}^k \left(\nabla_\tau^\top \bfx_\tau - \nabla_\tau^\top \bfx^\star\right) \\
        &= \bar\psi_k(\bfx^\star) - F_1(\bfx_1) +
        \sum_{\tau=1}^k \left(F_\tau(\bfx_\tau)-F_{\tau+1}(\bfx_{\tau+1})+\nabla_\tau^\top\bfx_\tau\right) \\
        &\quad+F_{k+1}(\bfx_{k+1})-F_{k+1}(\bfx^\star)\\
        &\leq \bar\psi_k(\bfx^\star) +
        \sum_{\tau=1}^k \left(F_\tau(\bfx_\tau)-F_{\tau+1}(\bfx_{\tau+1})+\nabla_\tau^\top\bfx_\tau\right)
    \end{align*}
    where the inequality follows from the fact that $\bfx_{k+1} \in \argmin F_{k+1}(\bfx)$ and $F_1(\bfx_1) = 0$.
\end{proof}

\begin{lemma}\label{lem 7.5}
Suppose that we run \Cref{alg:ada FTRL} and $L_1 > 0$. We have
\begin{enumerate}
    \item For $k=2,\ldots, K$, we have
    \[
        F_k(\bfx_k) - F_{k+1}(\bfx_{k+1}) + \nabla_k^\top \bfx_k \leq \frac{||\nabla_k||_\infty^2}{2\alpha_{k-1}}.
    \]
    \item For $k=1,\ldots, K$, we have
    \[
        F_k(\bfx_k) - F_{k+1}(\bfx_{k+1}) + \nabla_k^\top \bfx_k \leq 2||\nabla_k||_\infty
    \]
    \item For $k=1,\ldots, K$, we have
    \[
        F_k(\bfx_k) - F_{k+1}(\bfx_{k+1}) + \nabla_k^\top \bfx_k \leq \frac{2\sqrt{2}L_k^2}{\sqrt{\sum_{\tau=1}^kL_\tau^2}}.
    \]
\end{enumerate}
\end{lemma}
\begin{proof}
    \textbf{Proof of the first statement: }
    Note that 
    \begin{align*}
        &F_{k+1}(\bfx_{k+1}) = \bar\psi_k(\bfx_{k+1}) + \sum_{\tau=1}^k\nabla_\tau^\top\bfx_{k+1},\\
        &F_k(\bfx_{k+1}) = \bar\psi_{k-1}(\bfx_{k+1}) + \sum_{\tau=1}^{k-1}\nabla_\tau^\top \bfx_{k+1}.
    \end{align*}
    This implies that for $k=2,\ldots, K$,  
    \[  
        F_{k+1}(\bfx_{k+1}) - F_k(\bfx_{k+1}) = \bar\psi_k(\bfx_{k+1}) - \bar\psi_{k-1}(\bfx_{k+1}) + \nabla_k^\top \bfx_{k+1}.
    \]
    Equivalently, we have $F_{k+1}(\bfx_{k+1}) = F_k(\bfx_{k+1}) + \bar\psi_k(\bfx_{k+1}) - \bar\psi_{k-1}(\bfx_{k+1}) + \nabla_k^\top \bfx_{k+1}$. Then it follows that for $k=2,\ldots, K$,
    \begin{align*}
        &F_k(\bfx_k) - F_{k+1}(\bfx_{k+1}) + \nabla_k^\top \bfx_k \\
        &= F_k(\bfx_k) +\nabla_k^\top\bfx_k - (F_k(\bfx_{k+1})+ \nabla_k^\top \bfx_{k+1}) + \bar\psi_{k-1}(\bfx_{k+1}) - \bar\psi_k(\bfx_{k+1})\\
        &\leq F_k(\bfx_k) +\nabla_k^\top\bfx_k - (F_k(\bfx_{k+1})+ \nabla_k^\top \bfx_{k+1})
    \end{align*}
    where the inequality follows from the fact that $\{\alpha_k\}_{k=1}^K$ is a nondecreasing sequence, implying that $\bar\psi_{k-1}(\bfx) - \bar\psi_k(\bfx) \leq 0$ for any $\bfx$ and $k=2,\ldots, K$.

    To further deduce, recall that $\bar\psi_{k-1}$ is $\alpha_{k-1}$-strongly convex. This implies that $F_k$ is $\alpha_{k-1}$-strongly convex and thus $F_k(\bfx) + \nabla_k^\top\bfx + \Ind_{\Delta_{n-1}}(\bfx)$ is also $\alpha_{k-1}$-strongly convex, where $\Ind_{\Delta_{n-1}}$ denotes the indicator function with respect to $\Delta_{n-1}$. Moreover, as $\bfx_k \in \argmin_{\bfx\in\Delta_{n-1}} F_k(\bfx)$, we have $\bm 0 \in \partial (F_k + \Ind_{\Delta_{n-1}})(\bfx_k)$. Then this implies that $\nabla_k$ is a subgradient of $F_k(\bfx) + \nabla_k^\top \bfx + \Ind_{\Delta_{n-1}}(\bfx)$ at $\bfx_k$. It follows that
    \begin{align*}
        &F_k(\bfx_k)+\nabla_k^\top\bfx_k- \left(F_k(\bfx_{k+1}) + \nabla_k^\top\bfx_{k+1}\right)\\
        &=F_k(\bfx_k)+\nabla_k^\top\bfx_k+\Ind_{\Delta_{n-1}}(\bfx_k)- \left(F_k(\bfx_{k+1}) + \nabla_k^\top\bfx_{k+1}+\Ind_{\Delta_{n-1}}(\bfx_{k+1})\right) \\
        &\leq \nabla_k^\top(\bfx_k-\bfx_{k+1}) - \frac{\alpha_{k-1}}{2}||\bfx_{k+1}-\bfx_k||_1^2 \\
        &\leq ||\nabla_k||_\infty||\bfx_k - \bfx_{k+1}||_1 - \frac{\alpha_{k-1}}{2}||\bfx_{k+1}-\bfx_k||_1^2\\
        &\leq \frac{||\nabla_k||_\infty^2}{2\alpha_{k-1}}
    \end{align*}
    where the equality follows from the fact that $\bfx_k, \bfx_{k+1}\in\Delta_{n-1}$, the first inequality follows from the first-order property of $\alpha_{k-1}$-strongly convex function, the second inequality is due to H\"older's inequality, and the last inequality follows from the fact that $bx-ax^2 \leq b^2/(4a)$.
    We deduce that for $k=2,\ldots, K$,
    \begin{align*}
        F_k(\bfx_k) - F_{k+1}(\bfx_{k+1}) + \nabla_k^\top \bfx_k \leq \frac{||\nabla_k||_\infty^2}{2\alpha_{k-1}}.
    \end{align*}

    \textbf{Proof of the second statement: }
    Since $\bfx_k \in \argmin_{\bfx\in\Delta_{n-1}} F_k(\bfx)$, we have $F_k(\bfx_k) \leq F_k(\bfx_{k+1})$. It follows that for $k=2,\ldots, K$,
    \begin{align*}
        F_k(\bfx_k) - F_{k+1}(\bfx_{k+1}) + \nabla_k^\top \bfx_k &\leq F_k(\bfx_{k+1}) - F_{k+1}(\bfx_{k+1}) + \nabla_k^\top\bfx_k \\
        &= \bar\psi_{k-1}(\bfx_{k+1}) - \bar\psi_k(\bfx_{k+1}) - \nabla_k^\top \bfx_{k+1} + \nabla_k^\top\bfx_k\\
        &\leq - \nabla_k^\top \bfx_{k+1} + \nabla_k^\top\bfx_k\\
        &\leq 2||\nabla_k||_\infty
    \end{align*}
    where the second inequality follows from the fact that $ \alpha_k$ is nondecreasing, and the last inequality follows from H\"older's inequality and the fact that $||\bfx_{k+1} - \bfx_k||_1 \leq2$.

    Next, we consider $k=1$. In this case, 
    \begin{align*}
        F_1(\bfx_1) - F_2(\bfx_2) + \nabla_1^\top \bfx_1 
        &= -\bar\psi_1(\bfx_2) + \nabla_1^\top(\bfx_1-\bfx_2) \\
        &\leq \nabla_1^\top(\bfx_1-\bfx_2) \\
        &\leq 2||\nabla_1||_\infty
    \end{align*}
    where the first inequality follows from the fact that $\bar\psi_1(\bfx_2) \geq 0$, and the second inequality follows from H\"older's inequality and the fact that $||\bfx_{1} - \bfx_2||_1 \leq2$. Hence, the second statement holds for all $k=1,\ldots, K$.

    \textbf{Proof of the third statement: } By the first and second statements, for $k=2,\ldots, K$, we have
    \begin{align*}
        F_k(\bfx_k) - F_{k+1}(\bfx_{k+1}) + \nabla_k^\top \bfx_k 
        &\leq \min\left\{\frac{||\nabla_k||_\infty^2}{2\alpha_{k-1}}, 2||\nabla_k||_\infty\right\}\\
        &= \min\left\{\frac{2||\nabla_k||_\infty^2}{\sqrt{\sum_{\tau=1}^{k-1}L_\tau^2}}, 2||\nabla_k||_\infty\right\}\\
        &\leq \min\left\{\frac{2L_k^2}{\sqrt{\sum_{\tau=1}^{k-1}L_\tau^2}}, 2L_k\right\}\\
        &= 2\sqrt{\min\left\{\frac{L_k^4}{\sum_{\tau=1}^{k-1}L_\tau^2}, L_k^2\right\}}\\
        &\leq 2\sqrt{\frac{2}{\frac{\sum_{\tau=1}^{k-1}L_\tau^2}{L_k^4} + \frac{1}{L_k^2}}}\\
        &= \frac{2\sqrt{2}L_k^2}{\sqrt{\sum_{\tau=1}^kL_\tau^2}}
    \end{align*}
    where the second inequality follows from $||\nabla_k||_\infty\leq L_k$, and the third inequality follows from the fact that $\min\{a,b\} \leq \frac{2}{(1/a)+(1/b)}$, assuming $L_k > 0$. Note that if $L_k=0$, then the desired relation holds trivially. For $k=1$, by the second statement, we have
    \begin{align*}
        F_1(\bfx_1) - F_2(\bfx_2) + \nabla_1^\top \bfx_1  \leq 2||\nabla_1||_\infty \leq 2L_1 \leq  \frac{2\sqrt{2}L_1^2}{\sqrt{\sum_{\tau=1}^1 L_\tau^2}}.
    \end{align*}
    Hence, the third statement holds for all $k=1,\ldots, K$.
\end{proof}

\begin{proof}[\textbf{Proof of \Cref{lem: FTRL regret}}]
    By \Cref{lem 7.1} and the third statement of~\Cref{lem 7.5}, we have
    \begin{align*}
        \sum_{\tau=1}^k \nabla_\tau^\top(\bfx_\tau -\bfx^\star) 
        &\leq \bar\psi_k(\bfx^\star) + \sum_{\tau=1}^k\frac{2\sqrt{2}L_\tau^2}{\sqrt{\sum_{\tau'=1}^\tau L_{\tau'}^2}} \\
        &= \alpha_k\left(\sum_{i=1}^n \bfx_i^\star \log\bfx_i^\star + \log n\right) + \sum_{\tau=1}^k\frac{2\sqrt{2}L_\tau^2}{\sqrt{\sum_{\tau'=1}^\tau L_{\tau'}^2}}\\
        &\leq \alpha_k\log n + \sum_{\tau=1}^k\frac{2\sqrt{2}L_\tau^2}{\sqrt{\sum_{\tau'=1}^\tau L_{\tau'}^2}}\\
        &\leq \alpha_k\log n + 4\sqrt{2}\sqrt{\sum_{\tau=1}^k L_\tau^2}\\
        &= \frac{\log n}{4}\sqrt{\sum_{\tau=1}^k L_\tau^2} + 4\sqrt{2}\sqrt{\sum_{\tau=1}^k L_\tau^2}
    \end{align*}
    where the second inequality follows from the fact that $\sum_{i=1}^n\bfx_i^\star\log\bfx_i^\star \leq 0$, and the third inequality follows from $\sum_{\tau=1}^k y_\tau / \sqrt{\sum_{\tau'=1}^\tau y_{\tau'}} \leq 2\sqrt{\sum_{\tau=1}^k y_\tau}$ for any nonnegative sequence $\{y_\tau\}_{\tau\geq1}$ with $y_1>0$. This completes the proof.
\end{proof}

\subsubsection{Proof Overview of Eq. (\ref{eq:regret AdaOMD})}
    Under the update of \eqref{eq:AdaOMD}, we can use the following standard fact (e.g., Lemma 14 in \citet{wei2020online}):
    \begin{align*}
        \nabla_k^\top \bfx_{k+1} + \frac{1}{\eta_k}\KL(\bfx_{k+1}||\bfx_k) \leq \nabla_k^\top \bfx^\star + \frac{1}{\eta_k}\KL(\bfx^\star||\bfx_k) - \frac{1}{\eta_k}\KL(\bfx^\star || \bfx_{k+1}).
    \end{align*}
    It follows that
    \begin{align*}
        &\nabla_k^\top (\bfx_k - \bfx^\star) \\
        &\leq \nabla_k^\top(\bfx_k - \bfx_{k+1}) + \frac{1}{\eta_k}\KL(\bfx^\star||\bfx_k) - \frac{1}{\eta_k}\KL(\bfx^\star || \bfx_{k+1}) - \frac{1}{\eta_k}\KL(\bfx_{k+1}||\bfx_k)\\
        &\leq ||\nabla_k||_\infty||\bfx_k-\bfx_{k+1}||_1 - \frac{1}{2\eta_k}||\bfx_{k+1}-\bfx_k||_1^2 + \frac{1}{\eta_k}\KL(\bfx^\star||\bfx_k) - \frac{1}{\eta_k}\KL(\bfx^\star || \bfx_{k+1})\\
        &\leq 8||\nabla_k||_\infty^2\eta_k + \frac{1}{\eta_k}\KL(\bfx^\star||\bfx_k) - \frac{1}{\eta_k}\KL(\bfx^\star || \bfx_{k+1}).
    \end{align*}
    where the second inequality follows from H\"older's inequality and Pinsker's inequality, and the second inequality follows from the fact that $bx-ax^2 \leq 4b^2/a$. By summing over $k=1,\ldots, K$, we have
    {\allowdisplaybreaks
    \begin{align*}
        \Regret_{\textrm{AdaOMD}}(K) 
        &= \sum_{k=1}^K \nabla_k^\top (\bfx_k - \bfx^\star) \\
        &\lesssim \sum_{k=1}^K 8||\nabla_k||_\infty^2\eta_k + \sum_{k=1}^K \left(\frac{1}{\eta_k}\KL(\bfx^\star||\bfx_k) - \frac{1}{\eta_k}\KL(\bfx^\star || \bfx_{k+1})\right) \\
        &\lesssim \sum_{k=1}^K 8||\nabla_k||_\infty^2\eta_k + \sum_{k=1}^{K-1} \left(\frac{1}{\eta_{k+1}} - \frac{1}{\eta_{k}}\right) \KL(\bfx^\star||\bfx_{k+1}) \\
        &\lesssim \sqrt{\sum_{k=1}^K ||\nabla_k||_\infty^2} + \sup_{\bfx,\bfy\in\Delta_{n-1}}\KL(\bfx||\bfy)\sum_{k=1}^{K-1} \left(\frac{1}{\eta_{k+1}} - \frac{1}{\eta_{k}}\right) \\
        &\lesssim \sqrt{\sum_{k=1}^K ||\nabla_k||_\infty^2} + \frac{\sup_{\bfx,\bfy\in\Delta_{n-1}}\KL(\bfx||\bfy)}{\eta_K}\\
        &\lesssim \sup_{\bfx,\bfy\in\Delta_{n-1}}\KL(\bfx||\bfy)\sqrt{\sum_{k=1}^K ||\nabla_k||_\infty^2}
    \end{align*}}
    where the third inequality follows from that fact that $\sum_{k=1}^K y_k / \sqrt{\sum_{\tau=1}^k y_k} \leq 2\sqrt{\sum_{k=1}^K y_k}$, and the fourth inequality is due to the telescoping sum.
    \qed

\section{Parameterization and Covering Number}\label{appendix:param and cov}
In this subsection, we introduce parameterizations for $Q$-function estimates $\hat Q$, policies $\pi$, and value function estimates $\hat V$, which are useful for analyzing covering number~\citep{cassel2024warm, yu2026primaldual}. For the $Q$-function estimate, given $\beta\in \bbR$, $\bfw \in \bbR^d$, and $\bLambda \in \bbR^{d\times d}$, we define $\hat Q(\cdot,\cdot; \beta,\bfw, \bLambda)$ as
\[
    \hat Q(\cdot,\cdot; \beta,\bfw,\bLambda) = (\bphi(\cdot,\cdot)^\top \bfw - \beta ||\bphi(\cdot,\cdot)||_{\bLambda}) \cdot \sigma(-\beta_\phi||\bphi(\cdot,\cdot)||_{\bLambda} + \log K).
\]
Note that $\hat Q(\cdot,\cdot;\beta,\bfw,\bLambda)$ captures $\{\hat Q_{f,h}^k, \hat Q_{g,h}^k\}_{h\in[H],k\in[K]}$ in \Cref{alg:main}, since these can be written as for each $\ell=f,g$ and $(h,k)\in [H]\times[K]$,
\begin{align*}
    \hat Q_{\ell,h}^k(\cdot,\cdot) 
    &= \bar\bphi_h^m(\cdot,\cdot)^\top \left[\btheta_{\ell,h}^k + \hat\bmu_h^k \hat V_{\ell,h+1}^k\right] - \beta_p ||\bar\bphi_h^m(\cdot,\cdot)||_{(\bLambda_h^{k(m)})^{-1}} \\
    &= \left(\bphi(\cdot,\cdot)^\top\left[\btheta_{\ell,h}^k + \hat\bmu_h^k \hat V_{\ell,h+1}^k\right] - \beta_p||\bphi(\cdot,\cdot)||_{(\bLambda_h^{k(m)})^{-1}}\right)\cdot \sigma(-\beta_\phi||\bphi(\cdot,\cdot)||_{(\bLambda_h^{k(m)})^{-1}}+\log K).
\end{align*}
Given some nonnegative boundedness constants $C_\beta, C_\bfw, C_Q$, we define the function class $\calQ(C_\beta,C_\bfw, C_Q)$ as
\begin{equation*}
    \calQ(C_\beta, C_\bfw, C_Q) = \left\{\hat Q(\cdot,\cdot;\beta,\bfw,\bLambda): 
    \begin{array}{l}
    |\beta|\leq C_\beta,\ ||\bfw|| \leq C_\bfw,\ (2K)^{-1}\Id \preceq \bLambda \preceq \Id, \\
    ||\hat Q(\cdot,\cdot;\beta, \bfw, \bLambda)||_\infty \leq C_Q
    \end{array}
    \right\}.
\end{equation*}
Next, we define the parameterization for $\pi$, which is defined, for each $(s,a)\in\calS\times \calA$,
\[
    \pi(a|s;\beta,\bfw, \bLambda) = \frac{\exp(\hat Q(s,a;\beta,\bfw, \bLambda))}{\sum_{a'\in\calA} \exp(\hat Q(s,a';\beta,\bfw, \bLambda))}.
\]
By \eqref{eq:closed form}, we know that $\pi(\cdot|\cdot;\beta,\bfw,\bLambda)$ captures $\{\pi_h^{k+1}\}_{h\in[H],k\in[K]}$ in \Cref{alg:main}. Subsequently, we define the policy class $\Pi(C_\beta,C_\bfw, C_Q, C_Y)$ as
\[
    \Pi(C_\beta,C_\bfw, C_Q, C_Y) = \left\{\pi(\cdot|\cdot;\beta,\bfw,\bLambda): \begin{array}{l}
        |\beta|\leq (1+C_Y)C_\beta,\ ||\bfw||\leq (1+C_Y)C_\bfw, \\
        (2K)^{-1}\Id \preceq \bLambda \preceq \Id,\ \\
        ||\hat Q(\cdot,\cdot;\beta,\bfw,\bLambda)||_\infty \leq (1+C_Y)C_Q
    \end{array}
    \right\}.
\]
Finally, we define the function class of value function estimates $\calV(C_\beta, C_\bfw, C_Q, C_Y)$ as follows.
\[
    \calV(C_\beta,C_\bfw, C_Q, C_Y) = \left\{ \hat V(\cdot): 
        \hat V(\cdot) = \sum_{a\in\calA}\pi(a|\cdot)\hat Q(\cdot,a),\
        \begin{array}{l}
             \hat Q \in \calQ(C_\beta,C_\bfw, C_Q),\\
             \pi \in \Pi(KC_\beta,KC_\bfw, KC_Q, C_Y)
        \end{array}
    \right\}.
\]
\begin{lemma}[Lemma 10 in \citet{yu2026primaldual}] Let $1\leq \beta_\phi, C_\bfw, C_\beta$. For any $\hat Q(\cdot,\cdot;\beta_1, \bfw_1, \bLambda_1)$, $\hat Q(\cdot,\cdot;\beta_2, \bfw_2, \bLambda_2) \in \calQ(C_\beta,C_\bfw, C_Q)$, we have
\begin{align*}
    &\left\|\hat Q(\cdot,\cdot; \beta_1,\bfw_1,\bLambda_1) - \hat Q(\cdot,\cdot;\beta_2,\bfw_2,\bLambda_2)\right\|_\infty \\
    &\leq 4\sqrt{K}\beta_\phi\max\{C_\bfw, C_\beta\} \left\|(\beta_1,\bfw_1,\bLambda_1) - (\beta_2, \bfw_2, \bLambda_2)\right\|
\end{align*}
where $\left\|(\beta_1,\bfw_1,\bLambda_1) - (\beta_2, \bfw_2, \bLambda_2)\right\| = \sqrt{(\beta_1 - \beta_2)^2 + ||\bfw_1-\bfw_2||^2 + ||\bLambda_1 - \bLambda_2||_F^2}$.
\end{lemma}

The following lemmas are modifications of Lemmas 11 and 12 in \citet{yu2026primaldual}. These are obtained by taking $n=0$ for their lemmas. Here, $n$ denotes the number of mixing steps that have been applied to the policy, where $n=0$ in our case, as we do not apply the mixing step. For further details on the mixing step, we refer the reader to \citet{yu2026primaldual}.
\begin{lemma}[Lemma 11 in \citet{yu2026primaldual}]
    Let $1\leq \beta_\phi, C_\bfw, C_\beta$, and $0\leq C_Y$. For any $\pi(\cdot|\cdot;\beta_1, \bfw_1, \bLambda_1)$, $\pi(\cdot|\cdot;\beta_2, \bfw_2, \bLambda_2) \in \Pi(C_\beta,C_\bfw, C_Q,C_Y)$, we have for any $s\in\calS$,
    \begin{align*}
        &||\pi(\cdot|s;\beta_1,\bfw_1,\bLambda_1)-\pi(\cdot|s;\beta_2,\bfw_2,\bLambda_2)||_1\\
        &\leq 32\sqrt{K}\beta_\phi(1+C_Y)\max\{C_\bfw,C_\beta\}||(\beta_1,\bfw_1,\bLambda_1)-(\beta_2,\bfw_2,\bLambda_2)||
    \end{align*}
\end{lemma}
\begin{lemma}[Lemma 12 in \citet{yu2026primaldual}]\label{lem:covering number}
    Let $1\leq \beta_\phi, C_\bfw, C_\beta, C_Q$, and $0\leq C_Y$. Given $\epsilon > 0$, let $\calN_\epsilon(\calV(C_\beta,C_\bfw, C_Q,C_Y))$ denote the $\epsilon$-covering number of $\calV(C_\beta,C_\bfw, C_Q,C_Y)$ with respect to $\ell_\infty$-norm. Then we have
    \[
        \log \calN_\epsilon(\calV(C_\beta,C_\bfw, C_Q, C_Y)) \leq 12d^2 \log\left(1 + \frac{2C_1C_2}{\epsilon}\right)
    \]
    where $C_1 = 33\sqrt{K^3}\beta_\phi \max\{C_\bfw, C_\beta\}C_Q(1+C_Y)$, $C_2 = 2(1+C_Y)K(C_\beta+C_\bfw)+2\sqrt{d}$.
\end{lemma}

\section{Good Event}\label{appendix:good event}
In this section, we introduce the good event $\calE_g$. The definition of $\calE_g$ is similar to that of \citet{cassel2024warm, yu2026primaldual}. The main differences stem from its analysis, since our algorithm considers the adaptive FTRL for $\pi_h^{k+1}$, whereas those works mainly focus on the OMD with a fixed step size, which is equivalent to the non-adaptive FTRL. Moreover, we also have to consider the sophisticated update rule for $Y_{k+1}$.

Now, we introduce the formal definition of $\calE_g$ as follows. We define $\calE_g = \calE_1 \cap \calE_2$, where 
\begin{align*}
    \calE_1 &=\left\{\forall(k,h)\in [K]\times [H]: 
    \begin{array}{c}
         ||(\bmu_h - \hat\bmu_h^k)\hat V_{f,h+1}^k||_{\bLambda_h^k} \leq \beta_p, \ ||\hat Q_{f,h}^k||_\infty \leq \beta_Q, \\
         ||(\bmu_h - \hat\bmu_h^k)\hat V_{g,h+1}^k||_{\bLambda_h^k} \leq \beta_p, \ ||\hat Q_{g,h}^k||_\infty \leq \beta_Q
    \end{array}
    \right\},
    \\
    \calE_2 &=\left\{\sum_{\tau=1}^k \bbE_{\bbP,\pi^\tau}[W_\tau] \leq 2\sum_{\tau=1}^k W_\tau + 4H(3\beta_p + 8\beta_Q\beta_\phi^2)\log\frac{4K^2}{\delta}, \quad \forall k \in [K]\right\}
\end{align*}
Here, we define $\beta_p, \beta_Q, \beta_\phi,$ and $ W_k$ as follows:
\begin{align*}
    &\beta_p = 187dH\sqrt{\log(2HK/\delta)}, \\
    &\beta_Q = 2H, \\
    & \beta_\phi = 4\beta_p \log K, \\
    & W_\tau = \sum_{h=1}^H \left(3\beta_p||\bphi(s_h^\tau,a_h^\tau)||_{(\bLambda_h^\tau)^{-1}} + 8\beta_Q\beta_\phi^2||\bphi(s_h^\tau,a_h^\tau)||_{(\bLambda_h^\tau)^{-1}}^2\right).
\end{align*}

Before proving that $\calE_g$ holds with high probability, we first define the proxy good event $\bar\calE_g = \bar\calE_1 \cap\calE_2$, where
\begin{align*}
    \bar\calE_1 &=\left\{\forall (k,h,\hat V) \in [K]\times[H]\times \calV(\beta_p, \sqrt{d}+K\beta_Q, \beta_Q, K^8): ||(\bmu_h - \hat \bmu_h^k)\hat V||_{\bLambda_h^k} \leq \beta_p\right\}.
\end{align*}
We first prove that the proxy good event $\bar\calE_g$ holds with high probability. 
\begin{lemma}\label{lem:proxy good event}
    Let $1\leq \beta_\phi \leq K$. Then $\Pr(\bar\calE_g) \geq 1-\delta$ for any $\delta\in (0,1)$.
\end{lemma}
\begin{proof}
    We closely follow the proof of Lemma 14 in \citet{yu2026primaldual}. We first prove that $\Pr(\bar\calE_1)\geq 1-\delta/2$ and $\Pr(\calE_2)\geq 1-\delta/2$ separately, and then apply a union bound to show $\Pr(\bar\calE_g)\geq 1-\delta$. 
    
    To show that $\Pr(\bar\calE_1)\geq 1-\delta/2$, we bound the $\epsilon$-covering number of the value function class $\calV(\beta_p, \sqrt{d} + K\beta_Q, \beta_Q, K^8)$ for some $\epsilon > 0$, where $\epsilon$ will be specified later. Note that $1\leq \beta_\phi,\ \sqrt{d} + K\beta_Q,\ \beta_p$, and $0\leq K^8$, meaning that the conditions of \Cref{lem:covering number} are satisfied. Then, since we have $||\hat V||_\infty \leq \beta_Q$ for any $\hat V \in \calV(\beta_p, \sqrt{d} + K\beta_Q, \beta_Q, K^8)$, by \Cref{lem:covering number}, we have
    \begin{align*}
        \log \calN_\epsilon\left(\calV(\beta_p, \sqrt{d} + K\beta_Q, \beta_Q, K^8)\right) \leq 12 d^2 \log \left(1+\frac{2C_1C_2}{\epsilon}\right).
    \end{align*}
    Here, $C_1,C_2$ can be bounded as follows. Since $\sqrt{d}\leq\beta_p \leq \beta_\phi \leq K$ and $\beta_Q = 2H$, we have
    \begin{align*}
        C_1 &= 33\sqrt{K^3}\beta_\phi\max\{\sqrt{d} + K\beta_Q, \beta_p\}\beta_Q(1+K^8)\leq 528H^2K^{11.5},\\
        C_2 &= 2(1+K^8)K(\beta_p + \sqrt{d} + K\beta_Q) + 2\sqrt{d} \leq 22\sqrt{d}HK^{10}.
    \end{align*}
    Then it follows that
    \begin{align*}
        \log \calN_\epsilon\left(\calV(\beta_p, \sqrt{d} + K\beta_Q, \beta_Q, K^8)\right) 
        &\leq 12d^2\log\left(1 + \frac{23232\sqrt{d}H^3K^{21.5}}{\epsilon}\right)\\
        &= 12d^2\log\left(1 + 23232H^2K^{22.5}\right) \\
        &\leq 270d^2\log(2HK).
    \end{align*}
    where the equality follows by taking $\epsilon = \beta_Q\sqrt{d}/(2K) =H\sqrt{d}/K$. Then by \Cref{lem:cassel 22}, we have with probability at least $1-\delta/2$, for all $h\in [H], k\in [K], \hat V \in \calV(\beta_p, \sqrt{d} + K\beta_Q, \beta_Q, K^8)$,
    \begin{align*}
        ||(\bmu_h - \hat\bmu_h^k)\hat V||_{\bLambda_h^k} 
        &\leq 4\beta_Q \sqrt{d\log(K+1) + 2\log \calN_\epsilon\left(\calV(\beta_p, \sqrt{d} + K\beta_Q, \beta_Q, K^8)\right) + 2\log(2H^2/\delta)}\\
        &\leq 4\beta_Q\sqrt{d\log(K+1) + 540d^2\log(2HK) + 2\log(2H^2/\delta)}\\
        &\leq 187dH\sqrt{\log(2HK/\delta)}.
    \end{align*}
    Since we set $\beta_p = 187dH\sqrt{\log(2HK/\delta)}$, this implies that $\Pr(\bar\calE_1)\geq 1-\delta/2$.

    Next, we prove that $\calE_2$ holds with probability at least $1-\delta/2$. Note that for any $(s,a)\in\calS\times\calA$,
    \begin{align*}
        \sum_{h=1}^H\left(3\beta_p||\bphi(s,a)||_{(\bLambda_h^\tau)^{-1}} + 8\beta_Q\beta_\phi^2||\bphi(s,a)||_{(\bLambda_h^\tau)^{-1}}^2\right) \leq H(3\beta_p + 8\beta_Q\beta_\phi^2).
    \end{align*}
    This implies that $W_\tau \leq H(3\beta_p + 8\beta_Q\beta_\phi^2)$ almost surely. By the algorithm, the trajectory $\{(s_h^\tau,a_h^\tau)\}_{h\in[H]}$ are generated under $\bbP$ and $\pi^\tau$ for all $\tau\in [k]$. Moreover, since $\pi^\tau$ is $\calF_{\tau-1,H}$-measurable, $\{W_\tau\}_{\tau\geq1}$ is adapted to the filtration $\{\calF_{\tau,H}\}_{\tau\geq 1}$. Then by \Cref{lem:rosen d4} with probability at least $1-\delta/(2K)$, we have
    \begin{align*}
        \sum_{\tau=1}^k \bbE_{\bbP,\pi^\tau}[W_\tau] 
        &= \sum_{\tau=1}^k \bbE\left[\sum_{h=1}^H \left(3\beta_p||\bphi(s_h^\tau,a_h^\tau)||_{(\bLambda_h^\tau)^{-1}} + 8\beta_Q\beta_\phi^2||\bphi(s_h^\tau,a_h^\tau)||_{(\bLambda_h^\tau)^{-1}}^2\right)\mid \calF_{\tau-1,H}\right]\\
        &\leq 2\sum_{\tau=1}^k W_\tau + 4H(3\beta_p+8\beta_Q\beta_\phi^2)\log\frac{4K^2}{\delta}.
    \end{align*}
    By taking a union bound over $k$, this implies that $\Pr(\calE_2) \geq 1-\delta/2$. Moreover, since $\Pr(\bar\calE_1)\geq 1-\delta/2$ and $\Pr(\calE_2)\geq 1-\delta/2$, by a union bound, we have $\Pr(\bar\calE_g) \geq 1-\delta$, which concludes the proof.
\end{proof}

\begin{lemma}\label{lem:proxy good event 2}
    Suppose that $\bar\calE_g$ holds. If $\pi_h^k \in \Pi(K\beta_p, K\sqrt{d}+K^2\beta_Q, K\beta_{Q,h}, K^8)$ for all $h\in [H]$, then we have $\hat Q_{f,h}^k,\hat Q_{g,h}^k \in \calQ(\beta_p, \sqrt{d} + K\beta_Q, \beta_{Q,h})$ and $\hat V_{f,h}^k, \hat V_{g,h}^k \in \calV(\beta_p, \sqrt{d} + K\beta_Q, \beta_{Q,h}, K^8)$ for all $h \in [H+1]$, where $\beta_{Q,h} = 2(H-h+1)$.
\end{lemma}
\begin{proof}
    We closely follow the proof of Lemma 15 in \citet{yu2026primaldual}. We prove the lemma by backward induction on $h = H+1, \ldots, 1$. For the base case, since \Cref{alg:main} sets $\hat Q_{f,H+1}^k(\cdot,\cdot),\hat Q_{g,H+1}^k(\cdot,\cdot) = 0$ and $\hat V_{f,H+1}^k(\cdot), \hat V_{g,H+1}^k(\cdot)=0$, it is trivial that $\hat Q_{f,H+1}^k,\hat Q_{g,H+1}^k \in \calQ(\beta_p, \sqrt{d} + K\beta_Q, \beta_{Q,H+1})$ and $\hat V_{f,H+1}^k, \hat V_{g,H+1}^k \in \calV(\beta_p, \sqrt{d} + K\beta_Q, \beta_{Q,H+1}, K^8)$. 
    
    As an induction hypothesis, we assume that $\hat Q_{f,h+1}^k,\hat Q_{g,h+1}^k \in \calQ(\beta_p, \sqrt{d} + K\beta_Q, \beta_{Q,h+1})$ and $\hat V_{f,h+1}^k, \hat V_{g,h+1}^k \in \calV(\beta_p, \sqrt{d} + K\beta_Q, \beta_{Q,h+1}, K^8)$. Given this, we have to show that $\hat Q_{f,h}^k, \hat Q_{g,h}^k \in \calQ(\beta_p, \sqrt{d} + K\beta_Q, \beta_{Q,h})$ and $\hat V_{f,h}^k, \hat V_{g,h}^k \in \calV(\beta_p, \sqrt{d} + K\beta_Q, \beta_{Q,h}, K^8)$. First, for the case of $f$, we have
    \begin{align*}
        |\hat Q_{f,h}^k(s,a)| 
        &= \left|\bar\bphi_h^m(s,a)^\top (\btheta_{f,h}^k + \hat\bmu_h^k\hat V_{f,h+1}^k) - \beta_p\left\|\bar\bphi_h^m(s,a)\right\|_{(\bLambda_h^{k(m)})^{-1}}\right| \\
        &\leq \left|\bar\bphi_h^m(s,a)^\top (\btheta_{f,h}^k + \bmu_h\hat V_{f,h+1}^k)\right|+ \left|\bar\bphi_h^m(s,a)^\top \left(\hat\bmu_h^k\hat V_{f,h+1}^k-\bmu_h \hat V_{f,h+1}^k\right)\right|\\
        &\quad+\beta_p\left\|\bar\bphi_h^m(s,a)\right\|_{(\bLambda_h^{k(m)})^{-1}} \\
        &\leq \left|\bar\bphi_h^m(s,a)^\top (\btheta_{f,h}^k + \bmu_h\hat V_{f,h+1}^k)\right| + \left\|(\hat\bmu_h^k-\bmu_h)\hat V_{f,h+1}^k\right\|_{\bLambda_h^{k(m)}}\left\|\bar\bphi_h^m(s,a)\right\|_{(\bLambda_h^{k(m)})^{-1}} \\
        &\quad+ \beta_p\left\|\bar\bphi_h^m(s,a)\right\|_{(\bLambda_h^{k(m)})^{-1}}
    \end{align*}
    where the first inequality follows from the triangle inequality, and the second inequality follows from the Cauchy-Schwarz inequality. Note that $\hat V_{f,h+1} \in \calV(\beta_p, \sqrt{d} + K\beta_Q, \beta_{Q,h+1}, K^8)$ by the induction hypothesis. Furthermore, we have
    \[
        ||(\hat\bmu_h^k-\bmu_h)\hat V_{f,h+1}^k||_{\bLambda_h^{k(m)}} \leq ||(\hat\bmu_h^k-\bmu_h)\hat V_{f,h+1}^k||_{\bLambda_h^{k}} \leq \beta_p
    \]
    where the first inequality follows from the fact that $\bLambda_h^{k(m)}\preceq \bLambda_h^k$ and $\bar\calE_1$. It leads to
    \begin{align*}
        &|\hat Q_{f,h}^k(s,a)|\\
        &\leq \left|\bar\bphi_h^m(s,a)^\top (\btheta_{f,h}^k + \bmu_h\hat V_{f,h+1}^k)\right| + 2\beta_p\left\|\bar\bphi_h^m\right\|_{(\bLambda_h^{k(m)})^{-1}} \\
        &=\sigma\left(-\beta_\phi||\bphi(s,a)||_{(\bLambda_h^{k(m)})^{-1}}+\log K\right)\left|f_h^k(s,a) + \bbP_h \hat V_{f,h+1}^k\right| + 2\beta_p\left\|\bar\bphi_h^m\right\|_{(\bLambda_h^{k(m)})^{-1}} \\
        &\leq\left|f_h^k(s,a) + \bbP_h \hat V_{f,h+1}^k(s,a)\right| + 2\beta_p\left\|\bar\bphi_h^m\right\|_{(\bLambda_h^{k(m)})^{-1}} \\
        &\leq 1 + ||\hat V_{f,h+1}^k||_\infty + 2\beta_p\left\|\bar\bphi_h^m\right\|_{(\bLambda_h^{k(m)})^{-1}}
    \end{align*}
    where the second inequality follows from the fact that $\sigma(z) \leq 1$ for any $z\in\bbR$, the last inequality follows from the fact that $|\bbP_h \hat V_{f,h+1}^k| = |\sum_{s'}\bbP_h(s'|s,a)\hat V_{f,h+1}^k(s')| \leq ||\hat V_{f,h+1}^k||_\infty$. Moreover, we have
    \begin{align*}
        2\beta_p\left\|\bar\bphi_h^m\right\|_{(\bLambda_h^{k(m)})^{-1}} 
        &= 2\beta_p \sigma\left(-\beta_\phi||\bphi(s,a)||_{(\bLambda_h^{k(m)})^{-1}}+\log K\right)\left\|\bphi(s,a)\right\|_{(\bLambda_h^{k(m)})^{-1}}\\
        &\leq 2\beta_p\cdot\frac{2\log K}{\beta_\phi}\\
        &\leq 1
    \end{align*}
    where the first inequality follows from \Cref{lem:cassel 18}, and the second inequality is true due to $\beta_\phi = 4\beta_p\log K$. Thus, we have
    \begin{align}\label{eq:lem:proxy2:1}
        |\hat Q_{f,h}^k(s,a)| \leq 2 + ||\hat V_{f,h+1}^k||_\infty \leq 2(H-h+1)= \beta_{Q,h}.
    \end{align}
    Moreover, note that 
    \begin{align}\label{eq:lem:proxy2:2}
    \begin{aligned}
        \left\|\btheta_{f,h}^k + \hat\bmu_h^k \hat V_{f,h+1}^k\right\| 
        &\leq \sqrt{d} + \left\|\hat\bmu_h^k \hat V_{f,h+1}^k\right\|\\
        &=\sqrt{d} + \left\|(\bLambda_h^k)^{-1}\sum_{\tau=1}^{k-1}\bphi(s_h^\tau, a_h^\tau)\hat V_{f,h+1}^k(s_{h+1}^\tau) \right\|\\
        &\leq\sqrt{d} + \left\|(\bLambda_h^k)^{-1}\right\|\left\|\sum_{\tau=1}^{k-1}\bphi(s_h^\tau, a_h^\tau)\hat V_{f,h+1}^k(s_{h+1}^\tau) \right\|\\
        &\leq \sqrt{d} + K\beta_Q
    \end{aligned}
    \end{align}
    where the first inequality is due to $||\btheta_{f,h}^k||\leq \sqrt{d}$ followed from the linear CMDP assumptions, and the last inequality is true because $||(\bLambda_h^k)^{-1}|| \leq 1$ and $||\sum_{\tau=1}^{k-1}\bphi(s_h^\tau,a_h^\tau)\hat V_{f,h+1}^k(s_{h+1}^\tau)|| \leq K\beta_Q$. Hence, by \eqref{eq:lem:proxy2:1} and \eqref{eq:lem:proxy2:2}, together with $(2K)^{-1}\Id\preceq (\bLambda_h^{k(m)})^{-1}\preceq \Id$, we have $\hat Q_{f,h}^k \in \calQ(\beta_p, \sqrt{d}+K\beta_Q, \beta_{Q,h})$. Moreover, we have $\hat V_{f,h}^k \in \calV(\beta_p, \sqrt{d} + K\beta_Q, \beta_{Q,h}, K^8)$, since $\pi_h^k \in \Pi(K\beta_p, K\sqrt{d}+K^2\beta_Q, K\beta_{Q,h}, K^8)$ and $\hat V_{f,h}^k$ is given by the inner product of $\pi_h^k$ and $\hat Q_{f,h}^k$. For the case of $g$, by applying the same argument, we have $\hat Q_{g,h}^k \in \calQ(\beta_p, \sqrt{d}+K\beta_Q, \beta_{Q,h})$ and $\hat V_{g,h}^k \in \calV(\beta_p, \sqrt{d} + K\beta_Q, \beta_{Q,h}, K^8)$. This completes the induction.
\end{proof}

\begin{lemma}\label{lem:good event}
    Let $1 \leq \beta_\phi\leq K$, and let $H\geq 2$. Then $\Pr(\calE_g) \geq 1-\delta$ for any $\delta \in (0,1)$.
\end{lemma}
\begin{proof}
    Suppose that $\bar\calE_g = \bar\calE_1 \cap \calE_2$ holds, which occurs with probability at least $1-\delta$ by \Cref{lem:proxy good event}. To show that $\Pr(\calE_g)\geq 1-\delta$, it is sufficient to show that $\bar\calE_g$ implies $\calE_1$. 
    
    Before showing this implication, we first show that $\pi_h^k \in \Pi(K\beta_p, K\sqrt{d}+K^2\beta_Q, K\beta_{Q,h}, K^8)$, where $\beta_{Q,h} = 2(H-h+1)$ for all $(h,k)\in [H] \times [K]$. We prove this by induction on $k\in K(m)$ for each fixed $m\in E$. For the base case, we have $\pi_h^{k(m)} \in \Pi(K\beta_p, K\sqrt{d}+K^2\beta_Q, K\beta_{Q,h}, K^8)$ for all $h\in [H]$, since we set $\pi_h^{k(m)} = \piunif$ for all $h$ at the beginning of each epoch $m$. 

    Then the induction can be written as follows. Given $k\in K(m)$, we assume that $\pi_h^{k'} \in \Pi(K\beta_p, K\sqrt{d}+K^2\beta_Q, K\beta_{Q,h}, K^8)$ for all $h\in [H]$ and $k'$ such that $k(m) \leq k' < k$. To complete the induction, we show that $\pi_h^{k} \in \Pi(K\beta_p, K\sqrt{d}+K^2\beta_Q, K\beta_{Q,h}, K^8)$.

    For all $k(m)\leq k'<k$, we observe that
    \begin{align*}
        &\hat Q_{\ell,h}^{k'}(\cdot,\cdot)\\
        &= \bar\bphi_h^m(\cdot,\cdot)^\top \left[ \btheta_{\ell,h}^{k'} + \hat\bmu_h^{k'} \hat V_{\ell,h+1}^{k'}\right] - \beta_p ||\bar\bphi_h^m(\cdot,\cdot)||_{(\bLambda_h^{k(m)})^{-1}}\\
        &=\left(\bphi(\cdot,\cdot)^\top\bfw_{\ell,h}^{k'} - \beta_p||\bphi(\cdot,\cdot)||_{(\bLambda_h^{k(m)})^{-1}}\right)\cdot\sigma\left(-\beta_\phi||\bphi(\cdot,\cdot)||_{(\bLambda_h^{k(m)})^{-1}}+\log K\right)
    \end{align*}
    where
    \begin{align*}
        \bfw_{\ell,h}^{k'} = \btheta_{\ell,h}^{k'} + \hat\bmu_h^{k'}\hat V_{\ell,h+1}^{k'}.
    \end{align*}
    Moreover, since $\pi_h^{k'} \in \Pi(K\beta_p, K\sqrt{d}+K^2\beta_Q, K\beta_{Q,h}, K^8)$ by the induction hypothesis, then by \Cref{lem:proxy good event 2}, we have $\hat Q_{f,h}^{k'},\hat Q_{g,h}^{k'} \in \calQ(\beta_p, \sqrt{d}+K\beta_Q, \beta_{Q,h})$ and $\hat V_{f,h}^{k'}, \hat V_{g,h}^{k'} \in \calV(\beta_p, \sqrt{d}+K\beta_Q, \beta_{Q,h}, K^8)$. This implies that for each $\ell \in \{f,g\}$,
    \begin{align}\label{eq:lem:good event:w Q bound}
        ||\bfw_{\ell,h}^{k'}|| \leq \sqrt{d} + K\beta_Q, \quad ||\hat Q_{\ell,h}^{k'}(\cdot,\cdot)||_\infty \leq \beta_{Q,h}.
    \end{align}
    
    Recall that
    \begin{align*}
        &\pi_h^{k}(\cdot|s) \\
        &\propto \exp\left(-\frac{1}{\alpha_{k-1}}\sum_{\tau=k(m)}^{k-1}\left(\hat Q_{f,h}^\tau(s,\cdot)+Y_\tau \hat Q_{g,h}^\tau(s,\cdot)\right)\right)\\
        &\propto \exp\left(\hat Q(s,\cdot; \beta_h^k, \bfw_h^k,(\bLambda_h^{k(m)})^{-1})\right)
    \end{align*}
    where
    \begin{align*}
        \beta_h^k = - \frac{1}{\alpha_{k-1}}\sum_{\tau=k(m)}^{k-1}(1+Y_\tau)\beta_p, \quad \bfw_h^k = -\frac{1}{\alpha_{k-1}}\sum_{\tau=k(m)}^{k-1}(\bfw_{f,h}^{\tau}+Y_\tau\bfw_{g,h}^{\tau}).
    \end{align*}
    To show that $\pi_h^{k}(\cdot|s)\propto \exp(\hat Q(s,\cdot; \beta_h^k, \bfw_h^k,(\bLambda_h^{k(m)})^{-1})) \in \Pi(K\beta_p, K\sqrt{d}+K^2\beta_Q, K\beta_{Q,h}, K^8)$, we observe that
    \begin{align*}
        |\beta_h^k| = \left|\sum_{\tau=k(m)}^{k-1}\frac{(1+Y_\tau)\beta_p}{\alpha_{k-1}}\right|\leq \left|\sum_{\tau=k(m)}^{k-1}(1+Y_\tau)\beta_p\right| \leq (1+K^8)K\beta_p
    \end{align*}
    where the first inequality follows from $\alpha_{k-1} = H/2\sqrt{\sum_{\tau=k(m)}^{k-1}(1+Y_\tau)^2} \geq 1$, as $H\geq 2$ and $Y_\tau\geq 0$, and the second inequality is due to $Y_\tau \leq K^8$. Moreover,
    \begin{align*}
        ||\bfw_h^k||
        &=\left\|\sum_{\tau=k(m)}^{k-1}\frac{\bfw_{f,h}^{\tau}+Y_\tau\bfw_{g,h}^{\tau}}{\alpha_{k-1}}\right\| \\
        &\leq \left\|\sum_{\tau=k(m)}^{k-1}(\bfw_{f,h}^{\tau}+Y_\tau\bfw_{g,h}^{\tau})\right\|\\
        &\leq (1+K^8)(K\sqrt{d}+K^2\beta_Q)
    \end{align*}
    where the first inequality follows from $\alpha_{k-1} \geq 1$, and the second inequality follows from the fact that $Y_\tau \leq K^8$ and \eqref{eq:lem:good event:w Q bound}. Again, by \eqref{eq:lem:good event:w Q bound}, we have
    \begin{align*}
        \left\|\hat Q(\cdot,\cdot; \beta_h^k, \bfw_h^k,(\bLambda_h^{k(m)})^{-1})\right\|_\infty 
        &= \left\|\frac{1}{\alpha_{k-1}}\sum_{\tau=k(m)}^{k-1}\left(\hat Q_{f,h}^\tau(\cdot,\cdot)+Y_\tau \hat Q_{g,h}^\tau(\cdot,\cdot)\right)\right\|_\infty\\
        &\leq K(1+K^8)\beta_{Q,h}.
    \end{align*}
    Finally, under the induction hypothesis, we show that $\pi_h^{k}(\cdot|s)\propto \exp(\hat Q(s,\cdot; \beta_h^k, \bfw_h^k,(\bLambda_h^{k(m)})^{-1}))$, where $|\beta_h^k|\leq (1+K^8)K\beta_p$, $||\bfw_h^k||\leq (1+K^8)(K\sqrt{d}+K^2\beta_Q)$, and $||\hat Q(\cdot,\cdot; \beta_h^k, \bfw_h^k,(\bLambda_h^{k(m)})^{-1})||_\infty\leq K(1+K^8)\beta_{Q,h}$. Moreover, it is clear that $(2K)^{-1}\Id\preceq (\bLambda_h^{k(m)})^{-1} \preceq \Id$. These imply that $\pi_h^{k}(\cdot|s)\propto \exp(\hat Q(s,\cdot; \beta_h^k, \bfw_h^k,(\bLambda_h^{k(m)})^{-1})) \in \Pi(K\beta_p, K\sqrt{d}+K^2\beta_Q, K\beta_{Q,h}, K^8)$ as desired. This concludes the induction, showing that $\pi_h^{k} \in \Pi(K\beta_p, K\sqrt{d}+K^2\beta_Q, K\beta_{Q,h}, K^8)$. 

    We conclude the proof as follows. Since we have shown that $\pi_h^{k}(\cdot|s) \in \Pi(K\beta_p, K\sqrt{d}+K^2\beta_Q, K\beta_{Q,h}, K^8)$ for all $(h,k)\in [H]\times[K]$, by \Cref{lem:proxy good event 2}, we have $\hat Q_{f,h}^k,\hat Q_{g,h}^k \in \calQ(\beta_p, \sqrt{d} + K\beta_Q, \beta_{Q,h})$ and $\hat V_{f,h}^k, \hat V_{g,h}^k \in \calV(\beta_p, \sqrt{d} + K\beta_Q, \beta_{Q,h}, K^8)$ for all $(h,k)$. These imply that
    \begin{align*}
        ||\hat Q_{f,h}^k||_\infty, ||\hat Q_{g,h}^k||_\infty \leq \beta_{Q}.
    \end{align*}
    Moreover, since $\bar\calE_1$ is assumed, we have 
    \begin{align*}
        ||(\bmu_h-\hat\bmu_h^k)\hat V_{f,h+1}^k||_{\bLambda_h^k}, \ ||(\bmu_h-\hat\bmu_h^k)\hat V_{g,h+1}^k||_{\bLambda_h^k}\leq \beta_p.
    \end{align*}
    These imply that $\calE_1$ holds. As a consequence, we showed that if $\bar\calE_g$ holds, then $\calE_1$ holds; thus $\calE_g$ holds. Again, since $\Pr(\bar\calE_g)\geq  1-\delta$, we have $\Pr(\calE_g)\geq 1-\delta$.
\end{proof}

\section{Deferred Proofs for Section~\ref{sec:analysis} and Theorem~\ref{thm:main}}
\label{appendix:proofs}
\subsection{Proof of Lemma~\ref{lem:R sum}}
\begin{lemma}[Restatement of \Cref{lem:R sum}]\label{lem:R sum appendix}
    Assume that the good event $\calE_g$ holds. Let $H\geq 2$, and let $K \geq \cbias dH(\log|\calA|+20)\log(2H/\delta)$, where $\cbias$ is the absolute constant given by \Cref{lem:bias}. We have
    \begin{align*}
        \sum_{k=1}^K R_k \leq 
        12\sqrt{dH^5K}(\log|\calA|+20)\log(3dH(\log|\calA|+20)K).
    \end{align*}
\end{lemma}
\begin{proof}
    Note that
    \begin{align*}
        \sum_{k=1}^K R_k 
        &= \sum_{k=1}^K \left(108\gamma H^2+\beta_R \frac{Y_k}{\sqrt{\sum_{\tau=1}^k Y_\tau^2}}\right) \\
        &= 108K\gamma H^2 + \beta_R\sum_{k=1}^K \frac{Y_k}{\sqrt{\sum_{\tau=1}^k Y_\tau^2}}\\
        &\leq 3H\sqrt{K} + \beta_R\sum_{k=1}^K \frac{Y_k}{\sqrt{\sum_{\tau=1}^k Y_\tau^2}}
    \end{align*}
    where the inequality follows from the fact that $108K\gamma H^2 \leq 3H\sqrt{K}$, as $\gamma \leq 1/(36H\sqrt{K})$. Moreover, we have
    \begin{align*}
        \sum_{k=1}^K \frac{Y_k}{\sqrt{\sum_{\tau=1}^k Y_\tau^2}} 
        &\leq \sqrt{K}\sqrt{\sum_{k=1}^K\frac{Y_k^2}{\sum_{\tau=1}^k Y_\tau^2}} \\
        &\leq \sqrt{K}\sqrt{1+\log\left(\frac{1}{\gamma^2}\sum_{k=1}^K Y_k^2\right)}\\
        &\leq \sqrt{K}\sqrt{1+\log\frac{K^{17}}{\gamma^2}}\\
        &\leq \sqrt{K}\sqrt{1+\log(144dH^5(\log|\calA|+20)^2K^{18})}\\
        &\leq \sqrt{18K\log(3dH(\log|\calA|+20)K)}
    \end{align*}
    where the first inequality follows from the Cauchy-Schwarz inequality, the second inequality follows from the fact that $\sum_{k=1}^K y_k/(\sum_{\tau=1}^ky_\tau) \leq 1+\log(\sum_{k=1}^Ky_k / y_1)$ for any positive sequence $\{y_k\}_{k\geq 1}$ and that $Y_1 = \gamma$, and the third inequality follows from $Y_k \leq K^8$, and the fourth inequality is due to 
    \begin{align*}
        \frac{1}{\gamma} &= \max\{36H\sqrt{K}, 12H^2(\log|\calA|+20)\sqrt{dH\log(2K)}\}\\
        &\leq12H^2(\log|\calA|+20)\sqrt{dHK}.
    \end{align*}
    Finally, we deduce that
    \begin{align*}
        \sum_{k=1}^K R_k  
        &\leq 3H\sqrt{K} + 2H^2(\log|\calA|+20)\sqrt{dH\log(2K)}\sqrt{18K\log(3dH(\log|\calA|+20)K)}\\
        &\leq 12\sqrt{dH^5}(\log|\calA|+20)\sqrt{K}\log(3dH(\log|\calA|+20)K).
    \end{align*}
    This completes the proof.
\end{proof}

\subsection{Proofs of Lemma~\ref{lem:FTRL lemma policy} and Eq. (\ref{eq:ingredient 1})}
In this subsection, we first prove \Cref{lem:FTRL lemma policy}, which establishes the regret bound for adaptive FTRL over the policy space. We then prove the lemmas that bound the bias and optimism terms, following standard arguments in the online MDP literature. Finally, by combining these results with the value difference lemma (\Cref{lem:extended value diff}), we establish \eqref{eq:ingredient 1}.

\begin{lemma}[Restatement of Lemma~\ref{lem:FTRL lemma policy}]\label{lem:FTRL lemma policy appendix}
    Assume that the good event $\calE_g$ holds. Then we have for any $k \in [K]$,
    \begin{align*}
        &\sum_{m\in E(k)}\sum_{\tau\in K(m), \tau\leq k} \bbE_{\bar\bbP^m,\pi^\star}\left[\sum_{h=1}^H \left(\hat Q_{f,h}^\tau(s_h^\tau,\cdot) + Y_\tau\hat Q_{g,h}^\tau(s_h^\tau,\cdot)\right)^\top(\pi_h^\tau(\cdot|s_h^\tau) - \pi_h^\star(\cdot|s_h^\tau))\right] \\
        &\leq H^2\left(\log|\calA|+20\right)\sqrt{dH\log(2K)}\sqrt{k + \sum_{\tau=1}^k Y_\tau^2}
    \end{align*}
\end{lemma}
\begin{proof}
    Recall that 
    \begin{align*}
        \pi_h^{k+1}(\cdot|s) \in \argmin_{\pi(\cdot|s)\in \Delta(\calA)}\ \psi_k(\pi(\cdot|s)) + \sum_{\tau=k(m)}^k \left(\hat Q_{f,h}^\tau(s,\cdot) + Y_\tau\hat Q_{g,h}^\tau(s,\cdot)\right)^\top \pi(\cdot|s)
    \end{align*}
    where $\psi_k(\pi(\cdot|s)) = \alpha_k\sum_{a\in\calA}\pi(a|s)\log\pi(a|s)$ and $\alpha_k = (\beta_Q/4)\sqrt{\sum_{\tau=k(m)}^k (1+Y_\tau)^2}$. Note that under the good event $\calE_g$, we have $||\hat Q_{f,h}^k(s,\cdot)||_\infty,||\hat Q_{g,h}^k(s,\cdot)||_\infty\leq \beta_Q=2H$. Now we apply \Cref{lem: FTRL regret}. In particular, we have $n\leftarrow |\calA|$, $\nabla_k \leftarrow \hat Q_{f,h}^k(s,\cdot) + Y_k\hat Q_{g,h}^k(s,\cdot)$, $\bfx_1 \leftarrow \pi_h^{k(m)}(\cdot|s)$, $\bfx_k \leftarrow \pi_h^k(\cdot|s)$, $L_k\leftarrow \beta_Q(1+Y_k)$ $\bfx^\star \leftarrow \pi_h^\star(\cdot|s)$. Moreover, note that $L_k>0$ for all $k\in [K]$, as $\beta_Q(1+Y_k) >0$. Then by \Cref{lem: FTRL regret}, we have
    \begin{align*}
        &\sum_{\tau \in k(m), \tau\leq k}\left(\hat Q_{f,h}^\tau(s,\cdot) + Y_\tau\hat Q_{g,h}^\tau(s,\cdot)\right)^\top(\pi_h^\tau(\cdot|s) - \pi_h^\star(\cdot|s)) \\
        &\leq \left(\frac{\log|\calA|}{4}+4\sqrt{2}\right)\beta_Q\sqrt{\sum_{\tau\in k(m), \tau\leq k}(1+Y_\tau)^2}.
    \end{align*}
    
    By summing over $h=1,\ldots,H$ and taking $\bbE_{\bar\bbP^m,\pi^\star}$ on both sides, we have
    \begin{align*}
        &\sum_{\tau\in k(m), \tau\leq k} \bbE_{\bar\bbP^m,\pi^\star}\left[\sum_{h=1}^H\left(\hat Q_{f,h}^\tau(s_h^\tau,\cdot) + Y_\tau\hat Q_{g,h}^\tau(s_h^\tau,\cdot)\right)^\top(\pi_h^\tau(\cdot|s_h^\tau) - \pi_h^\star(\cdot|s_h^\tau))\right] \\
        &\leq H\left(\frac{\log|\calA|}{4}+4\sqrt{2}\right)\beta_Q\sqrt{\sum_{\tau\in k(m),\tau\leq k}(1+Y_\tau)^2}.
    \end{align*}
    Next, we take the summation over $m\in E(k)$. Then it follows that
    \begin{align*}
        &\sum_{m\in E(k)}\sum_{\tau\in k(m),\tau\leq k} \bbE_{\bar\bbP^m,\pi^\star}\left[\sum_{h=1}^H\left(\hat Q_{f,h}^\tau(s_h^\tau,\cdot) + Y_\tau\hat Q_{g,h}^\tau(s_h^\tau,\cdot)\right)^\top(\pi_h^\tau(\cdot|s_h^\tau) - \pi_h^\star(\cdot|s_h^\tau))\right] \\
        &\leq H\left(\frac{\log|\calA|}{4}+4\sqrt{2}\right)\beta_Q \sum_{m\in E(k)}\sqrt{\sum_{\tau\in k(m),\tau\leq k}(1+Y_\tau)^2}\\
        &\leq H\left(\frac{\log|\calA|}{4}+4\sqrt{2}\right)\beta_Q\sqrt{|E|}\sqrt{\sum_{\tau=1}^k (1+Y_\tau)^2}\\
        &\leq H\left(\frac{\sqrt{2}\log|\calA|}{4}+8\right)\beta_Q\sqrt{|E|}\sqrt{k + \sum_{\tau=1}^k Y_\tau^2}
    \end{align*}
    where the second inequality follows from the Cauchy-Schwarz inequality and $|E(k)| \leq |E|$, and the last inequality follows from $(a+b)^2 \leq 2(a^2+b^2)$. Note that $|E|\leq (3/2)dH\log(2K)$ by \Cref{lem:num epoch} and $\beta_Q = 2H$. It leads to 
    \begin{align*}
        &\sum_{m\in E(k)}\sum_{\tau\in k(m), \tau\leq k} \bbE_{\bar\bbP^m,\pi^\star}\left[\sum_{h=1}^H\left(\hat Q_{f,h}^\tau(s_h^\tau,\cdot) + Y_\tau\hat Q_{g,h}^\tau(s_h^\tau,\cdot)\right)^\top(\pi_h^\tau(\cdot|s_h^\tau) - \pi_h^\star(\cdot|s_h^\tau))\right]\\
        &\leq H^2\left(\frac{\sqrt{2}\log|\calA|}{2}+16\right)\sqrt{(3/2)dH\log(2K)}\sqrt{k + \sum_{\tau=1}^k Y_\tau^2}\\
        &\leq H^2\left(\log|\calA|+20\right)\sqrt{dH\log(2K)}\sqrt{k + \sum_{\tau=1}^k Y_\tau^2}
    \end{align*}
    where the last inequality follows from $\beta_Q = 2H$, $(\sqrt{2}/2)\sqrt{3/2} \leq 1$, and $16\sqrt{3/2}\leq 20$.
\end{proof}

\begin{lemma}\label{lem:optimism}
    Assume that the good event $\calE_g$ holds. For all $(s,a,h,k,\ell) \in \calS \times \calA \times [H]\times [K] \times \{f,g\}$, we have
    \begin{align*}
        \hat Q_{\ell,h}^k(s,a) - \bar\bphi_h^m(s,a)^\top(\btheta_{\ell,h}^k + \bmu_h\hat V_{\ell,h+1}^k) \leq  0
    \end{align*}
\end{lemma}
\begin{proof}
    For any $(s,a,h,k)\in\calS \times \calA \times [H] \times [K]$, we have
    \begin{align*}
        &\hat Q_{f,h}^k(s,a) - \bar\bphi_h^m(s,a)^\top (\btheta_{f,h}^k + \bmu_h \hat V_{f,h+1}^k) \\
        &= \bar\bphi_h^m(s,a)^\top \left((\hat \bmu_h^k-\bmu_h) \hat V_{f,h+1}^k\right) - \beta_p||\bar\bphi_h^m(s,a)||_{(\bLambda_h^{k(m)})^{-1}} \\
        &\leq ||\bar\bphi_h^m(s,a)||_{(\bLambda_h^{k})^{-1}}\left\|(\hat \bmu_h^k-\bmu_h)\hat V_{f,h+1}\right\|_{\bLambda_h^k} - \beta_p||\bar\bphi_h^m(s,a)||_{(\bLambda_h^{k(m)})^{-1}}\\
        &\leq \beta_p||\bar\bphi_h^m(s,a)||_{(\bLambda_h^{k})^{-1}} - \beta_p||\bar\bphi_h^m(s,a)||_{(\bLambda_h^{k(m)})^{-1}}\\
        &\leq 0
    \end{align*}
    
    where the first inequality follows from the Cauchy-Schwarz inequality, the second inequality follows from $\calE_g$, and the last inequality follows from the fact that $\bLambda_h^{k(m)} \preceq \bLambda_h^k$. Moreover, the same argument applies to the case of $g$. 
\end{proof}

\begin{lemma}\label{lem:bias}
    Assume that the good event $\calE_g$ holds. For all $(k,\ell) \in [K] \times \{f,g\}$, there exists an absolute constant $\cbias$ such that
    \begin{align*}
        \sum_{\tau=1}^k(V_{\ell^\tau,1}^{\pi^\tau}(s_1) - \hat V_{\ell,1}^\tau(s_1)) \leq\cbias\left(d^{3/2}H^2\sqrt{K}\log(2HK/\delta)+d^3H^4 \log^4(2HK/\delta)\right)
    \end{align*}
\end{lemma}
\begin{proof}
    Consider the case of $f$. By \Cref{lem:extended value diff}, we have
    \begin{align}\label{eq:bias 0}
    \begin{aligned}
        &\sum_{\tau=1}^k(V_{f^\tau,1}^{\pi^\tau}(s_1) - \hat V_{f,1}^\tau(s_1))\\
        &= \sum_{\tau=1}^k\bbE_{\bbP,\pi^\tau}\left[\sum_{h=1}^H \left(f_h^\tau(s_h,a_h) + \bbP_h \hat V_{f,h+1}^\tau(s_h,a_h) - \hat Q_{f,h}^\tau(s_h,a_h)\right)\mid s_1\right]\\
        &= \sum_{\tau=1}^k\bbE_{\bbP,\pi^\tau}\left[\sum_{h=1}^H \left(\bphi(s_h,a_h)^\top(\btheta_{f,h}^\tau + \bmu_h \hat V_{f,h+1}^\tau) - \hat Q_{f,h}^\tau(s_h,a_h)\right)\mid s_1\right].
    \end{aligned}
    \end{align}
    Note that
    \begin{align}\label{eq:bias 1}
    \begin{aligned}
        &\bphi(s,a)^\top(\btheta_{f,h}^\tau + \bmu_h \hat V_{f,h+1}^\tau) - \hat Q_{f,h}^\tau(s,a) \\
        &=\bphi(s,a)^\top(\btheta_{f,h}^\tau + \bmu_h \hat V_{f,h+1}^\tau) - \bar\bphi_h^m(s,a)^\top(\btheta_{f,h}^\tau + \hat\bmu_h^\tau \hat V_{f,h+1}^\tau) + \beta_p||\bar\bphi_h^m(s,a)||_{(\bLambda_h^{k(m)})^{-1}}\\
        &=\bphi(s,a)^\top(\btheta_{f,h}^\tau + \bmu_h \hat V_{f,h+1}^\tau) - \bar\bphi_h^m(s,a)^\top(\btheta_{f,h}^\tau + \bmu_h \hat V_{f,h+1}^\tau) \\
        &\quad+ \bar\bphi_h^m(s,a)^\top \left((\bmu_h-\hat\bmu_h^\tau)\hat V_{f,h+1}^\tau\right) + \beta_p||\bar\bphi_h^m(s,a)||_{(\bLambda_h^{k(m)})^{-1}}\\
        &\leq\bphi(s,a)^\top(\btheta_{f,h}^\tau + \bmu_h \hat V_{f,h+1}^\tau) - \bar\bphi_h^m(s,a)^\top(\btheta_{f,h}^\tau + \bmu_h \hat V_{f,h+1}^\tau) \\
        &\quad+ ||\bar\bphi_h^m(s,a)||_{(\bLambda_h^\tau)^{-1}} \left\|(\bmu_h-\hat\bmu_h^\tau)\hat V_{f,h+1}^\tau\right\|_{\bLambda_h^\tau} + \beta_p||\bar\bphi_h^m(s,a)||_{(\bLambda_h^{k(m)})^{-1}}\\
        &\leq\left(\bphi(s,a)-\bar\bphi_h^m(s,a)\right)^\top(\btheta_{f,h}^\tau + \bmu_h \hat V_{f,h+1}^\tau) + \beta_p||\bphi(s,a)||_{(\bLambda_h^\tau)^{-1}} +  \beta_p||\bphi(s,a)||_{(\bLambda_h^{k(m)})^{-1}}
    \end{aligned}
    \end{align}
    where the last inequality follows from $\calE_g$ and that $\sigma(z) \leq 1$ for all $z\in\bbR$. By \Cref{lem:cassel 3}, we have
    \begin{align}\label{eq:bias 2}
    \begin{aligned}
        &\left(\bphi(s,a)-\bar\bphi_h^m(s,a)\right)^\top(\btheta_{f,h}^\tau + \bmu_h \hat V_{f,h+1}^\tau) \\
        &\leq \left(4\beta_\phi^2||\bphi(s,a)||_{(\bLambda_h^\tau)^{-1}}^2 + 2K^{-1}\right)|\bphi(s,a)^\top (\btheta_{f,h}^\tau + \bmu_h \hat V_{f,h+1}^\tau)|\\
        &=\left(4\beta_\phi^2||\bphi(s,a)||_{(\bLambda_h^\tau)^{-1}}^2 + 2K^{-1}\right)|f_h^\tau(s,a) + \bbP_h\hat V_{f,h+1}^\tau(s,a)|\\
        &\leq 16H\beta_\phi^2||\bphi(s,a)||_{(\bLambda_h^\tau)^{-1}}^2 + 8HK^{-1}
    \end{aligned}
    \end{align}
    where the second inequality follows from $|f_h^\tau(s,a) + \bbP_h\hat V_{f,h+1}^\tau(s,a)|\leq 4H$ under $\calE_g$. Moreover, we have
    \begin{align}\label{eq:bias 3}
    \begin{aligned}
        \beta_p||\bphi(s,a)||_{(\bLambda_h^{k(m)})^{-1}} 
        &\leq \beta_p\sqrt{\frac{\det(\bLambda_h^\tau)}{\det(\bLambda_h^{k(m)})}}||\bphi(s,a)||_{(\bLambda_h^\tau)^{-1}}\\
        &\leq 2\beta_p||\bphi(s,a)||_{(\bLambda_h^\tau)^{-1}}
    \end{aligned}
    \end{align}
    where the first inequality follows from \Cref{lem:abbasi 12}, and the second inequality follows from  $\det(\bLambda_h^\tau)\leq 2\det(\bLambda_h^{k(m)})$. Applying \eqref{eq:bias 2} and \eqref{eq:bias 3} to \eqref{eq:bias 1} leads to 
    \begin{align*}
        &\bphi(s,a)^\top(\btheta_{f,h}^\tau + \bmu_h \hat V_{f,h+1}^\tau) - \hat Q_{f,h}^\tau(s,a) \\
        &\leq 16H\beta_\phi^2||\bphi(s,a)||_{(\bLambda_h^\tau)^{-1}}^2 + 8HK^{-1} + 3\beta_p||\bphi(s,a)||_{(\bLambda_h^\tau)^{-1}}.
    \end{align*}
    By summing over $h$, taking $\bbE_{\bbP,\pi^\tau}$, and summing over $\tau$, we have
    \begin{align*}
        &\sum_{\tau=1}^k\bbE_{\bbP,\pi^\tau}\left[\sum_{h=1}^H\left(\bphi(s_h,a_h)^\top(\btheta_{f,h}^\tau + \bmu_h \hat V_{f,h+1}^\tau) - \hat Q_{f,h}^\tau(s_h,a_h)\right)\mid s_1\right] \\
        &\leq \sum_{\tau=1}^k\bbE_{\bbP,\pi^\tau}\left[\sum_{h=1}^H\left(16H\beta_\phi^2||\bphi(s_h,a_h)||_{(\bLambda_h^\tau)^{-1}}^2 + 8HK^{-1} + 3\beta_p||\bphi(s_h,a_h)||_{(\bLambda_h^\tau)^{-1}}\right) \mid s_1\right]\\
        &\leq 2\sum_{\tau=1}^k\sum_{h=1}^H\left(3\beta_p||\bphi(s_h^\tau,a_h^\tau)||_{(\bLambda_h^\tau)^{-1}} + 8\beta_Q\beta_\phi^2||\bphi(s_h^\tau,a_h^\tau)||_{(\bLambda_h^\tau)^{-1}}^2\right) \\
        &\quad+ 4H(3\beta_p + 8\beta_Q\beta_\phi^2)\log\frac{4K^2}{\delta} + 8H^2
    \end{align*}
    where the second inequality follows from $\calE_g$. Applying this to \eqref{eq:bias 0} yields
    \begin{align}\label{eq:bias 4}
    \begin{aligned}
        &\sum_{\tau=1}^k(V_{f^\tau,1}^{\pi^\tau}(s_1) - \hat V_{f,1}^\tau(s_1)) \\
        &\leq 2\sum_{\tau=1}^k\sum_{h=1}^H\left(3\beta_p||\bphi(s_h^\tau,a_h^\tau)||_{(\bLambda_h^\tau)^{-1}} + 8\beta_Q\beta_\phi^2||\bphi(s_h^\tau,a_h^\tau)||_{(\bLambda_h^\tau)^{-1}}^2\right) \\
        &\quad+ 4H(3\beta_p + 8\beta_Q\beta_\phi^2)\log\frac{4K^2}{\delta} + 8H^2.
    \end{aligned}
    \end{align}
    Note that for each $h\in [H]$,
    \begin{align*}
        \sum_{\tau=1}^k||\bphi(s_h^\tau,a_h^\tau)||_{(\bLambda_h^\tau)^{-1}}^2 \leq 2\log\det(\bLambda_h^{k+1}) \leq 2d\log(k+1)
    \end{align*}
    where the first inequality is due to \Cref{lem: jin d2}. This implies that
    \begin{align*}
        \sum_{\tau=1}^k\sum_{h=1}^H||\bphi(s_h^\tau,a_h^\tau)||_{(\bLambda_h^\tau)^{-1}} 
        &\leq \sqrt{kH}\sqrt{\sum_{\tau=1}^k\sum_{h=1}^H||\bphi(s_h^\tau,a_h^\tau)||_{(\bLambda_h^\tau)^{-1}}^2}\\
        &\leq \sqrt{2dH^2k\log (k+1)}
    \end{align*}
    where the first inequality follows from the Cauchy-Schwarz inequality. Note that $k\leq K$. Applying this to \eqref{eq:bias 4} leads to the following. For all $k\in [K]$,
    \begin{align*}
        \sum_{\tau=1}^k(V_{f^\tau,1}^{\pi^\tau}(s_1) - \hat V_{f,1}^\tau(s_1)) \leq\cbias\left(d^{3/2}H^2\sqrt{K}\log(2HK/\delta)+d^3H^4 \log^4(2HK/\delta)\right)
    \end{align*}
    where $\cbias$ is an absolute constant. Note that the same argument applies to the case of $g$. This completes the proof.
\end{proof}

\begin{lemma}[Formal statement of \eqref{eq:ingredient 1}]\label{lem:surrogate regret}
Assume that the good event $\calE_g$ holds. For all $k \in [K]$, we have 
\begin{align*}
    &\sum_{\tau=1}^k \left(V_{f^\tau,1}^{\pi^\tau}(s_1) - V_{f^\tau,1}^{\pi^\star}(s_1)\right) + \sum_{\tau=1}^k Y_\tau \left(\hat V_{g,1}^{\tau}(s_1) - b\right) \\
    &\leq H^2\left(\log|\calA|+20\right)\sqrt{dH\log(2K)}\sqrt{k + \sum_{\tau=1}^k Y_\tau^2} \\
    &\quad+\cbias d^{3/2}H^2\sqrt{K}\log(2HK/\delta)+\cbias d^3H^4 \log^4(2HK/\delta)
\end{align*}
where $\cbias$ is the absolute constant given by \Cref{lem:bias}.
\end{lemma}
\begin{proof}
For a fixed $k\in [K]$, we first decompose the desired term. Note that
\begin{align*}
    &\sum_{\tau=1}^k \left(V_{f^\tau,1}^{\pi^\tau}(s_1) - V_{f^\tau,1}^{\pi^\star}(s_1)\right) + \sum_{\tau=1}^k Y_\tau \left(\hat V_{g,1}^\tau(s_1) - b\right)\\
    &= \sum_{\tau=1}^k \left(\hat V_{f,1}^{\tau}(s_1) - V_{f^\tau,1}^{\pi^\star}(s_1)\right) + \sum_{\tau=1}^k\left(V_{f^\tau,1}^{\pi^\tau}(s_1)-\hat V_{f,1}^{\tau}(s_1)\right) + \sum_{\tau=1}^k Y_\tau \left(\hat V_{g,1}^{\tau}(s_1) - b\right)\\
    &\leq \sum_{\tau=1}^k \left(\hat V_{f,1}^{\tau}(s_1) - V_{f^\tau,1}^{\pi^\star}(s_1)\right) + \sum_{\tau=1}^k Y_\tau \left(\hat V_{g,1}^{\tau}(s_1) - V_{g^\tau,1}^{\pi^\star}(s_1)\right) \\
    &\quad+ \sum_{\tau=1}^k\left(V_{f^\tau,1}^{\pi^\tau}(s_1)-\hat V_{f,1}^{\tau}(s_1)\right)
\end{align*}
where the inequality follows from the fact that $V_{g^\tau,1}^{\pi^\star}(s_1)\leq b$ and that $Y_\tau$ is nonnegative for all $\tau\geq 1$. To bound the first two sums, recall that $\hat V_{f,h}^{\tau}(s) + Y_\tau\hat V_{g,h}^{\tau}(s) = \sum_{a\in\calA} \pi_h^\tau(a|s)(\hat Q_{f,h}^\tau(s,a) + Y_\tau\hat Q_{g,h}^\tau(s,a))$ for all $(s,h) \in \calS \times [H]$. Let $E(k)=\{m\in E: k(m) \leq k\}$, i.e., the set of epochs up to and including the epoch containing episode $k$. Then we have
\begin{align*}
    &\sum_{\tau=1}^k \left(\hat V_{f,1}^{\tau}(s_1) - V_{f^\tau,1}^{\pi^\star}(s_1)\right) + \sum_{\tau=1}^k Y_\tau \left(\hat V_{g,1}^{\tau}(s_1) - V_{g^\tau,1}^{\pi^\star}(s_1)\right) \\
    &\leq\sum_{\tau=1}^k \left(\hat V_{f,1}^{\tau}(s_1) - \bar V_{f^\tau,1}^{\pi^\star}(s_1)\right) + \sum_{\tau=1}^k Y_\tau \left(\hat V_{g,1}^{\tau}(s_1) - \bar V_{g^\tau,1}^{\pi^\star}(s_1)\right) \\
    &= \sum_{m\in E(k)}\sum_{\tau\in k(m), \tau\leq k} \bbE_{\bar\bbP^m,\pi^\star}\left[\sum_{h=1}^H \left(\hat Q_{f,h}^\tau(s_h^\tau,\cdot) + Y_\tau\hat Q_{g,h}^\tau(s_h^\tau,\cdot)\right)^\top(\pi_h^\tau(\cdot|s_h^\tau) - \pi_h^\star(\cdot|s_h^\tau))\right]\\
    &\quad+\sum_{m\in E(k)}\sum_{\tau\in k(m), \tau\leq k}\bbE_{\bar\bbP^m,\pi^\star}\left[\sum_{h=1}^H \left(\hat Q_{f,h}^\tau(s_h^\tau,a_h^\tau)-\bar f_h^\tau(s_h^\tau,a_h^\tau) - \sum_{s'\in\calS}\bar\bbP_h^m(s'|s_h^\tau,a_h^\tau)\hat V_{f,h+1}^\tau(s')\right)\right]\\
    &\quad+\sum_{m\in E(k)}\sum_{\tau\in k(m), \tau\leq k}\bbE_{\bar\bbP^m,\pi^\star}\left[\sum_{h=1}^H Y_\tau\left(\hat Q_{g,h}^\tau(s_h^\tau,a_h^\tau)-\bar g_h(s_h^\tau,a_h^\tau) - \sum_{s'\in\calS}\bar\bbP_h^m(s'|s_h^\tau,a_h^\tau)\hat V_{g,h+1}^\tau(s')\right)\right]
\end{align*}
where the inequality follows from \Cref{lem:cassel 2}, and the equality is due to \Cref{lem:extended value diff}. Moreover, by the definition of contracted MDPs, we have for $\ell =f,g$,
\begin{align*}
    &\hat Q_{\ell,h}^\tau(s_h^\tau,a_h^\tau)-\bar \ell_h^\tau(s_h^\tau,a_h^\tau) - \sum_{s'\in\calS}\bar\bbP_h^m(s'|s_h^\tau,a_h^\tau)\hat V_{\ell,h+1}^\tau(s') \\
    &= \hat Q_{\ell,h}^\tau(s,a) - \bar\bphi_h^m(s,a)^\top(\btheta_{\ell,h}^\tau + \bmu_h\hat V_{\ell,h+1}^\tau).
\end{align*}
This leads to
\begin{align*}
    &\sum_{\tau=1}^k \left(V_{f^\tau,1}^{\pi^\tau}(s_1) - V_{f^\tau,1}^{\pi^\star}(s_1)\right) + \sum_{\tau=1}^k Y_\tau \left(\hat V_{g,1}^{\tau}(s_1) - b\right) \\
    &\leq
    \underbrace{\sum_{m\in E(k)}\sum_{\tau\in k(m), \tau\leq k} \bbE_{\bar\bbP^m,\pi^\star}\left[\sum_{h=1}^H \left(\hat Q_{f,h}^\tau(s_h^\tau,\cdot) + Y_\tau\hat Q_{g,h}^\tau(s_h^\tau,\cdot)\right)^\top(\pi_h^\tau(\cdot|s_h^\tau) - \pi_h^\star(\cdot|s_h^\tau))\right]}_{\text{(I)}}\\
    &\quad+\underbrace{\sum_{m\in E(k)}\sum_{\tau\in k(m), \tau\leq k}\bbE_{\bar\bbP^m,\pi^\star}\left[\sum_{h=1}^H \left(\hat Q_{f,h}^\tau(s_h^\tau,a_h^\tau) - \bar\bphi_h^m(s_h^\tau,a_h^\tau)^\top(\btheta_{f,h}^\tau + \bmu_h\hat V_{f,h+1}^\tau)\right)\right]}_{\text{(II)}}\\
    &\quad+\underbrace{\sum_{m\in E(k)}\sum_{\tau\in k(m), \tau\leq k}\bbE_{\bar\bbP^m,\pi^\star}\left[\sum_{h=1}^H Y_\tau\left(\hat Q_{g,h}^\tau(s_h^\tau,a_h^\tau) - \bar\bphi_h^m(s_h^\tau,a_h^\tau)^\top(\btheta_{g,h}^\tau + \bmu_h\hat V_{g,h+1}^\tau)\right)\right]}_{\text{(III)}}\\
    &\quad+ \underbrace{\sum_{\tau=1}^k\left(V_{f^\tau,1}^{\pi^\tau}(s_1)-\hat V_{f,1}^{\tau}(s_1)\right)}_{\text{(IV)}}.
\end{align*}
By \Cref{lem:FTRL lemma policy appendix}, we have
\begin{align*}
    \text{(I)} \leq H^2\left(\log|\calA|+20\right)\sqrt{dH\log(2K)}\sqrt{k + \sum_{\tau=1}^k Y_\tau^2}.
\end{align*}
By \Cref{lem:optimism}, we have
\begin{align*}
    \text{(II)}, \text{(III)} \leq 0.
\end{align*}
By \Cref{lem:bias}, we have
\begin{align*}
    \text{(IV)} \leq \cbias\left(d^{3/2}H^2\sqrt{K}\log(2HK/\delta)+d^3H^4 \log^4(2HK/\delta)\right)
\end{align*}
where $\cbias$ is the absolute constant given by \Cref{lem:bias}. Combining these yields
\begin{align*}
    &\sum_{\tau=1}^k \left(V_{f^\tau,1}^{\pi^\tau}(s_1) - V_{f^\tau,1}^{\pi^\star}(s_1)\right) + \sum_{\tau=1}^k Y_\tau \left(\hat V_{g,1}^{\tau}(s_1) - b\right) \\
    &\leq H^2\left(\log|\calA|+20\right)\sqrt{dH\log(2K)}\sqrt{k + \sum_{\tau=1}^k Y_\tau^2} \\
    &\quad+\cbias d^{3/2}H^2\sqrt{K}\log(2HK/\delta)+\cbias d^3H^4 \log^4(2HK/\delta).
\end{align*}
\end{proof}

\subsection{Proofs of Lemmas~\ref{lem:drift} and~\ref{lem:no clip}}

\begin{lemma}[Restatement of \Cref{lem:drift}]\label{lem:drift appendix}
Assume that the good event $\calE_g$ holds. Then we have
    \begin{align*}
        \Phi(Z_{k+1}) - \Phi(Z_k) \leq \Phi'(Z_k)\left(\hat V_{g,1}^k(s_1)-b-\frac{1}{2}R_k\right).
    \end{align*}
\end{lemma}
\begin{proof}
    We closely follow the proof of Lemma 7 in \citet{yu2026constrained}. Note that
    \begin{align}\label{eq:drift 0}
    \begin{aligned}
        \Phi(Z_{k+1}) - \Phi(Z_k) 
        &= e^{\gamma Z_{k+1}} - e^{\gamma Z_k}\\
        &=e^{\gamma Z_k}\left(e^{\gamma (Z_{k+1} - Z_k)} - 1\right)\\
        &=e^{\gamma Z_k}\left(e^{\gamma(\hat V_{g,1}^k(s_1) - b - R_k) } - 1\right)\\
        &\leq e^{\gamma Z_k}\left(\gamma(\hat V_{g,1}^k(s_1) - b - R_k) + \gamma^2(\hat V_{g,1}^k(s_1) - b - R_k)^2\right)\\
        &= \Phi'(Z_k)\left(\hat V_{g,1}^k(s_1) - b - R_k + \gamma(\hat V_{g,1}^k(s_1) - b - R_k)^2\right)
    \end{aligned}
    \end{align}
    where the inequality follows from the fact that $e^z \leq 1 + z + z^2$ for all $z\leq 1$ and that $\gamma(\hat V_{g,1}^k(s_1)-b-R_k) \leq \gamma (\hat V_{g,1}^k(s_1) - b) \leq 2\gamma H \leq 1$. This is true because $\gamma,R_k,b\geq0$, $\hat V_{g,1}^k(s_1) \leq 2H$ under $\calE_g$, and $\gamma \leq 1/(2H)$. The last equality follows from $e^{\gamma Z_k} = \Phi'(Z_k)/\gamma$. Moreover, we have
    \begin{align}\label{eq:drift 1}
    \begin{aligned}
        \gamma (\hat V_{g,1}^k(s_1) - b -R_k)^2 
        &\leq 3\gamma (\hat V_{g,1}^k(s_1) - b)^2 + 3\cdot108^2\cdot\gamma^3 H^4 + 3\gamma\beta_R^2\frac{Y_k^2}{\sum_{\tau=1}^k Y_\tau^2}\\
        &\leq 27\gamma H^2 + 3\cdot108^2\cdot\gamma^3 H^4 + 3\gamma\beta_R^2\frac{Y_k^2}{\sum_{\tau=1}^k Y_\tau^2}\\
        &\leq 27\gamma H^2 + 27\gamma H^2 + 3\gamma\beta_R^2\frac{Y_k^2}{\sum_{\tau=1}^k Y_\tau^2}\\
        &\leq 54\gamma H^2 + \frac{1}{2}\beta_R\frac{Y_k}{\sqrt{\sum_{\tau=1}^k Y_\tau^2}}\\
        &= \frac{1}{2}R_k
    \end{aligned}
    \end{align}
    where the first inequality follows from $(x+y+z)^2\leq 3(x^2+y^2+z^2)$ for any $x,y,z$, the second inequality follows from $-3H\leq \hat V_{g,1}^k(s_1) -b \leq 2H$ under $\calE_g$, the third inequality follows from $\gamma \leq 1/(36H)$, and the last inequality follows from $\gamma\leq 1/(6\beta_R)$ and that $Y_k /\sqrt{\sum_{\tau=1}^k Y_\tau^2}\leq 1$. Since $\Phi'(Z_k)\geq0$, by applying \eqref{eq:drift 1} to \eqref{eq:drift 0}, we have
    \begin{align*}
        \Phi(Z_{k+1}) - \Phi(Z_k) \leq \Phi'(Z_k)\left(\hat V_{g,1}^k(s_1)-b-\frac{1}{2}R_k\right).
    \end{align*}
    This concludes the proof.
\end{proof}

\begin{lemma}[Restatement of \Cref{lem:no clip}]\label{lem:no clip appendix}
    Assume that the good event $\calE_g$ holds. Let $H\geq 2$, and let $K \geq \cbias dH(\log|\calA|+20)\log(2H/\delta)$, where $\cbias$ is the absolute constant given by \Cref{lem:bias}. Then we have 
    \[
        \Phi'(Z_k) = Y_k,\quad \forall k\in [K+1]
    \]
\end{lemma}
\begin{proof}
    We prove the statement by contradiction. Suppose that there exists some $k$ such that $\Phi'(Z_k) \neq Y_k$, i.e., $\{k\in[K+1]: \Phi'(Z_k)\neq Y_k\}\neq \varnothing$. Given this, let $\khit$ denote the first episode in which this event occurs, i.e., $\khit = \min\{k\in [K+1]: \Phi'(Z_k) \neq Y_k\}$. To conclude the proof, it is sufficient to show that $\Phi'(Z_{\khit}) = Y_{\khit}$, which contradicts the assumption.
    
    Note that $\khit-1 \leq K$. Then by \Cref{lem:surrogate regret}, we have
    \begin{align*}
    \begin{aligned}
        &\sum_{\tau=1}^{\khit-1} \left(V_{f^\tau,1}^{\pi^\tau}(s_1) - V_{f^\tau,1}^{\pi^\star}(s_1)\right) + \sum_{\tau=1}^{\khit-1} Y_\tau \left(\hat V_{g,1}^{\tau}(s_1) - b\right)\\
        &\leq H^2\left(\log|\calA|+20\right)\sqrt{dH\log(2K)}\sqrt{\khit-1 + \sum_{\tau=1}^{\khit-1} Y_\tau^2} \\
    &\quad+\cbias d^{3/2}H^2\sqrt{K}\log(2HK/\delta)+\cbias d^3H^4 \log^4(2HK/\delta).
    \end{aligned}
    \end{align*}
    Note that for each $\tau \leq \khit-1$, we assumed that $\Phi'(Z_\tau) = Y_\tau$. Then the above inequality can be rewritten as
    \begin{align}\label{eq:no clip 1}
    \begin{aligned}
        &\sum_{\tau=1}^{\khit-1} \left(V_{f^\tau,1}^{\pi^\tau}(s_1) - V_{f^\tau,1}^{\pi^\star}(s_1)\right) + \sum_{\tau=1}^{\khit-1} \Phi'(Z_\tau) \left(\hat V_{g,1}^{\tau}(s_1) - b\right)\\
        &\leq H^2\left(\log|\calA|+20\right)\sqrt{dH\log(2K)}\sqrt{\khit-1 + \sum_{\tau=1}^{\khit-1} Y_\tau^2} \\
        &\quad+\cbias d^{3/2}H^2\sqrt{K}\log(2HK/\delta)+\cbias d^3H^4 \log^4(2HK/\delta).
    \end{aligned}
    \end{align}
    Moreover, by summing the result of \Cref{lem:drift} over $\tau=1,\ldots,\khit-1$, we have 
    \begin{align}\label{eq:no clip 2}
        \sum_{\tau=1}^{\khit-1}\left(\Phi(Z_{\tau+1}) - \Phi(Z_\tau)\right) \leq \sum_{\tau=1}^{\khit-1}\Phi'(Z_\tau)\left(\hat V_{g,1}^\tau(s_1) - b -\frac{1}{2}R_\tau\right).
    \end{align}
    Since the left-hand side is a telescoping sum and $\Phi(Z_1) = 0$, it can be written as $\Phi(Z_{\khit})$. Then combining \eqref{eq:no clip 1} and \eqref{eq:no clip 2} yields 
    \begin{align}\label{eq:no clip 3}
    \begin{aligned}
        &\sum_{\tau=1}^{\khit-1} \left(V_{f^\tau,1}^{\pi^\tau}(s_1) - V_{f^\tau,1}^{\pi^\star}(s_1)\right) + \Phi(Z_{\khit})\\
        &\leq H^2\left(\log|\calA|+20\right)\sqrt{dH\log(2K)}\sqrt{\khit-1 + \sum_{\tau=1}^{\khit-1} Y_\tau^2} \\
        &\quad+\cbias d^{3/2}H^2\sqrt{K}\log(2HK/\delta)+\cbias d^3H^4 \log^4(2HK/\delta)\\
        &\quad - \sum_{\tau=1}^{\khit-1}\frac{\Phi'(Z_\tau)}{2}R_\tau\\
        &\leq H^2\left(\log|\calA|+20\right)\sqrt{dH\log(2K)}\sqrt{\khit-1} \\
        &\quad+H^2\left(\log|\calA|+20\right)\sqrt{dH\log(2K)}\sqrt{\sum_{\tau=1}^{\khit-1} Y_\tau^2} \\
        &\quad+\cbias d^{3/2}H^2\sqrt{K}\log(2HK/\delta)+\cbias d^3H^4 \log^4(2HK/\delta)\\
        &\quad - \sum_{\tau=1}^{\khit-1}\frac{\Phi'(Z_\tau)}{2}R_\tau
    \end{aligned}
    \end{align}
    where the second inequality follows from the fact that $\sqrt{x+y}\leq \sqrt{x}+\sqrt{y}$. Again, since $\Phi'(Z_\tau) = Y_\tau$ for all $\tau \leq \khit-1$, we have
    \begin{align}\label{eq:no clip 4}
    \begin{aligned}
    \sum_{\tau=1}^{\khit-1}\frac{\Phi'(Z_\tau)}{2}R_\tau 
        &= \sum_{\tau=1}^{\khit-1}\frac{\Phi'(Z_\tau)}{2}\left(108\gamma H^2 + \beta_R\frac{\Phi'(Z_\tau)}{\sqrt{\sum_{\tau'=1}^\tau\Phi'(Z_{\tau'})^2}}\right)\\
        &\geq \frac{\beta_R}{2}\sum_{\tau=1}^{\khit-1}\frac{\Phi'(Z_\tau)^2}{\sqrt{\sum_{\tau'=1}^\tau\Phi'(Z_{\tau'})^2}}\\
        &= \frac{\beta_R}{2}\sum_{\tau=1}^{\khit-1}\frac{\sum_{\tau'=1}^\tau\Phi'(Z_{\tau'})^2 - \sum_{\tau'=1}^{\tau-1}\Phi'(Z_{\tau'})^2}{\sqrt{\sum_{\tau'=1}^\tau\Phi'(Z_{\tau'})^2}}\\
        &\geq \frac{\beta_R}{2}\sum_{\tau=1}^{\khit-1}\left(\sqrt{\sum_{\tau'=1}^\tau\Phi'(Z_{\tau'})^2}- \sqrt{\sum_{\tau'=1}^{\tau-1}\Phi'(Z_{\tau'})^2}\right)\\
        &\geq \frac{\beta_R}{2}\sqrt{\sum_{\tau=1}^{\khit-1}\Phi'(Z_\tau)^2}\\
        &=H^2(\log|\calA| + 20)\sqrt{dH\log(2K)}\sqrt{\sum_{\tau=1}^{\khit-1}\Phi'(Z_\tau)^2}.
    \end{aligned}
    \end{align}
    By combining \eqref{eq:no clip 3} and \eqref{eq:no clip 4}, we have
    \begin{align}\label{eq:no clip 5}
    \begin{aligned}
        &\sum_{\tau=1}^{\khit-1} \left(V_{f^\tau,1}^{\pi^\tau}(s_1) - V_{f^\tau,1}^{\pi^\star}(s_1)\right) + \Phi(Z_{\khit}) \\
        &\leq
        H^2\left(\log|\calA|+20\right)\sqrt{dH\log(2K)}\sqrt{\khit-1} \\
        &\quad+\cbias d^{3/2}H^2\sqrt{K}\log(2HK/\delta)+\cbias d^3H^4 \log^4(2HK/\delta).
    \end{aligned}
    \end{align}
    Since $-H \leq V_{f^\tau,1}^{\pi^\tau}(s_1) - V_{f^\tau,1}^{\pi^\star}(s_1)$, it leads to
    \begin{align*}
        \Phi(Z_{\khit}) + 1
        &\leq H(\khit-1)+H^2\left(\log|\calA|+20\right)\sqrt{dH\log(2K)}\sqrt{\khit-1} \\
        &\quad+\cbias d^{3/2}H^2\sqrt{K}\log(2HK/\delta)+\cbias d^3H^4 \log^4(2HK/\delta)+1.
    \end{align*}
    To simplify the above relation, we use the following algebraic facts: $\khit-1\leq K$; $\log K \leq\log(2K)\leq K$; $\log(2HK/\delta) \leq K\log(2H/\delta)$. Then it results in
    \begin{align*}
        \Phi(Z_{\khit})+1\leq4\cbias d^3H^4K^4\left(\log|\calA|+20\right)\log^4\left(2H/\delta\right).
    \end{align*}
    Moreover, by the choice of $\gamma$, we have $\gamma \leq 1/(36H\sqrt{K})$. Then it follows that
        \begin{align*}
        \Phi'(Z_{\khit})
        &\leq
        \gamma(\Phi(Z_{\khit})+1)\\
        &\leq \frac{1}{9}\cbias d^3H^3K^{7/2}(\log|\calA|+20)\log^4(2H/\delta)\\
        &< K^8
        \end{align*}
    where the last inequality follows from $K \geq \cbias dH(\log|\calA|+20)\log(2H/\delta)$. Note that this implies that $Y_{\khit} = \min\{\Phi'(Z_{\khit}),K^8\}=\Phi'(Z_{\khit})$, which contradicts the assumption that $\khit$ is the first episode $k$ that $\Phi'(Z_k) \neq Y_k$. Hence, this implies that $\{k\in[K+1]: \Phi'(Z_k)\neq Y_k\}=\varnothing$, which is equivalent to $\Phi'(Z_k)=Y_k$ for all $k\in [K+1]$. This concludes the proof.
\end{proof}

\subsection{Proof of Theorem~\ref{thm:main}}

In this section, we provide the proof of \Cref{thm:main}. Before presenting the proof, we first establish the following lemma, which bounds $\Regret(k)$, $Y_{k+1}$, and $Z_{k+1}$ for all $k\in[K]$ by combining the results obtained so far.
\begin{lemma}\label{lem:Reg Y Z}
    Assume that the good event $\calE_g$ holds. Let $H\geq 2$, and let $K \geq \cbias dH(\log|\calA|+20)\log(2H/\delta)$, where $\cbias$ is the absolute constant given by \Cref{lem:bias}. Then the following statements hold.
    \begin{enumerate}
        \item We have
        \begin{align*}
            \sum_{k=1}^K (V_{f^k,1}^{\pi^k}(s_1) - V_{f^k,1}^{\pi^\star}(s_1)) 
            &\leq H^2\left(\log|\calA|+20\right)\sqrt{dH\log(2K)}\sqrt{K} + 1\\
            &\quad+\cbias d^{3/2}H^2\sqrt{K}\log(2HK/\delta)+\cbias d^3H^4 \log^4(2HK/\delta).
        \end{align*}
        \item For any $k\in[K]$, 
        \begin{align*}
            Y_{k+1} \leq \frac{1}{9}\cbias d^3 H^3K^{7/2}(\log|\calA|+20)\log^4(2H/\delta).
        \end{align*}
        \item For any $k\in[K]$, 
        \begin{align*}
            Z_{k+1} \leq 12H^2(\log|\calA|+20)\sqrt{dHK}\log\left(4\cbias d^3H^4K^4(\log|A|+20)\log^4(2H/\delta)\right).
        \end{align*}
    \end{enumerate}
\end{lemma}
\begin{proof}
    By \Cref{lem:no clip appendix}, we know that $\Phi'(Z_k) = Y_k$ for all $k\in [K+1]$. By applying the same argument for deriving \eqref{eq:no clip 5} in \Cref{lem:no clip appendix} but for any $k \in [K]$, instead of $\khit-1$, we have for each $k$,
    \begin{align}
    \begin{aligned}
        &\sum_{\tau=1}^{k} \left(V_{f^\tau,1}^{\pi^\tau}(s_1) - V_{f^\tau,1}^{\pi^\star}(s_1)\right) + \Phi(Z_{k+1}) \\
        &\leq
        H^2\left(\log|\calA|+20\right)\sqrt{dH\log(2K)}\sqrt{K} \\
        &\quad+\cbias d^{3/2}H^2\sqrt{K}\log(2HK/\delta)+\cbias d^3H^4 \log^4(2HK/\delta).
    \end{aligned}
    \end{align}
    Since $\Phi(Z_{k+1})\geq -1$, we have the first statement. 
    
    Moreover, to derive the second statement, since $\sum_{\tau=1}^{k} \left(V_{f^\tau,1}^{\pi^\tau}(s_1) - V_{f^\tau,1}^{\pi^\star}(s_1)\right) \geq -HK$, we have
    \begin{align}\label{eq:Reg Y Z:1}
    \begin{aligned}
        \Phi(Z_{k+1})+1 
        &\leq HK+H^2\left(\log|\calA|+20\right)\sqrt{dH\log(2K)}\sqrt{K} \\
        &\quad+\cbias d^{3/2}H^2\sqrt{K}\log(2HK/\delta)+\cbias d^3H^4 \log^4(2HK/\delta)\\
        &\leq 4\cbias d^3H^4K^4(\log|\calA|+20)\log^4(2H/\delta)
    \end{aligned}
    \end{align}
    Since $\gamma\leq 1/(36H\sqrt{K})$, we have
    \begin{align*}
        Y_{k+1} = \Phi'(Z_{k+1}) 
        &= \gamma(\Phi(Z_{k+1})+1)\\
        &\leq \frac{1}{9}\cbias d^3 H^3K^{7/2}(\log|\calA|+20)\log^4(2H/\delta).
    \end{align*}
    Finally, the third statement can be obtained as follows. Again, from \eqref{eq:Reg Y Z:1}, we derive that
    \begin{align*}
        Z_{k+1}
        &\leq \frac{1}{\gamma}\log\left(4\cbias d^3H^4K^4(\log|A|+20)\log^4(2H/\delta)\right)\\
        &\leq 12H^2(\log|\calA|+20)\sqrt{dHK}\log\left(4\cbias d^3H^4K^4(\log|A|+20)\log^4(2H/\delta)\right).
    \end{align*}
\end{proof}

\begin{proof}[\textbf{Proof of Theorem~\ref{thm:main}}]
    To prove the theorem, we consider the following two cases: (1) $K<\beta_\phi$ or $K < \cbias dH(\log|\calA|+20)\log(2H/\delta)$ and (2) $\beta_\phi\leq K$ and $\cbias dH(\log|\calA|+20)\log(2H/\delta)\leq K$.

    \textbf{Case 1:} Suppose that $K<\beta_\phi$ or $K < \cbias dH(\log|\calA|+20)\log(2H/\delta)$. In this case, we have
    \begin{align*}
        K &< \max\{\beta_\phi, \cbias dH(\log|\calA|+20)\log(2H/\delta)\}\\
        &\leq \beta_\phi+ \cbias dH(\log|\calA|+20)\log(2H/\delta)\\
        &\leq 748dH\sqrt{\log(2HK/\delta)}\log K + \cbias dH(\log|\calA|+20)\log(2H/\delta).
    \end{align*}
    Then the trivial bounds for $\Regret(K)$ and $\Violation(K)$, i.e., $HK$, satisfy
    \begin{align*}
        \Regret(K),\Violation(K) &\leq HK=\tbigO\left(dH^3\right).
    \end{align*}

    \textbf{Case 2:} 
    Suppose that $\beta_\phi\leq K$ and $\cbias dH(\log|\calA|+20)\log(2H/\delta)\leq K$. It is clear that $1\leq \beta_\phi$. In this case, we can assume $\calE_g$, which occurs with probability at least $1-\delta$ by \Cref{lem:good event}. Moreover, Lemmas~\ref{lem:no clip appendix},~\ref{lem:Reg Y Z}, and~\ref{lem:R sum appendix} hold.

    To bound $\Regret(K)$, by the first statement of \Cref{lem:Reg Y Z}, we have
    \begin{align*}
        \Regret(K) &= \sum_{k=1}^K (V_{f^k,1}^{\pi^k}(s_1) - V_{f^k,1}^{\pi^\star}(s_1)) \\
        &\leq H^2\left(\log|\calA|+20\right)\sqrt{dH\log(2K)}\sqrt{K} + 1\\
        &\quad+\cbias d^{3/2}H^2\sqrt{K}\log(2HK/\delta)+\cbias d^3H^4 \log^4(2HK/\delta)\\
        &=\tbigO\left(\sqrt{d^3H^4K} +\sqrt{dH^5 K} + d^3H^4\right).
    \end{align*}

    To bound $\Violation(K)$, we observe that
    \begin{align*}
        \Violation(K) 
        &= \sum_{k=1}^K \left(V_{g^k,1}^{\pi^k}(s_1)-b\right)\\
        &= \sum_{k=1}^K \left(V_{g^k,1}^{\pi^k}(s_1)-\hat V_{g,1}^k(s_1)\right) + \sum_{k=1}^K \left(\hat V_{g,1}^{k}(s_1)-b\right)\\
        &=\sum_{k=1}^K \left(V_{g^k,1}^{\pi^k}(s_1)-\hat V_{g,1}^k(s_1)\right) + Z_{K+1} + \sum_{k=1}^K R_k
    \end{align*}
    where the last equality can be obtained by summing $\hat V_{g,1}^k(s_1) -b = Z_{k+1}-Z_k + R_k$ over $k$, and the fact that $Z_1 = 0$. Then by \Cref{lem:bias}, the third statement of \Cref{lem:Reg Y Z}, and \Cref{lem:R sum appendix}, we have
    \begin{align*}
        &\Violation(K) \\
        &\leq\cbias\left(d^{3/2}H^2\sqrt{K}\log(2HK/\delta)+d^3H^4 \log^4(2HK/\delta)\right) \\
        &\quad+12H^2(\log|\calA|+20)\sqrt{dHK}\log\left(4\cbias d^3H^4K^4(\log|A|+20)\log^4(2H/\delta)\right)\\
        &\quad+ 12\sqrt{dH^5}(\log|\calA|+20)\sqrt{K}\log(3dH(\log|\calA|+20)K) \\
        &=\tbigO\left(\sqrt{d^3H^4K}+ \sqrt{dH^5K} +d^3H^4\right).
    \end{align*}
\end{proof}

{\color{alert}
\newpage
\section{Deferred Algorithm and Proofs for Section~\ref{sec:extension}}
\label{appendix:extension}

\begin{algorithm}[!t]
\caption{Algorithm for Hard Constraint Violation}\label{alg:hard}
\textbf{Input:} $\delta \in (0,1)$; $\beta_p=187dH\sqrt{\log(2HK/\delta)}$; $\beta_\phi=4\beta_p\log K$; $\beta_Q=2H$; $\beta_R = 2H^2(\log|\calA|+20)\sqrt{dH\log(2K)}$; $\gamma= \min\{1/(36H\sqrt{K}), 1/(6\beta_R)\}$\\
\textbf{Initialization:} $m\leftarrow 0$; $\Phi(x) = e^{\gamma x}-1$; $Z_1\leftarrow 0$; $Y_1 \leftarrow \Phi'(Z_1)$; $\bLambda_h^1 \leftarrow \Id \ \forall h\in [H]$; $\hat Q_{f,H+1}^k(s,a),\hat Q_{g,H+1}^k(s,a),\hat V_{f,H+1}^k(\cdot), \hat V_{g,H+1}^k(\cdot)=0$, $\forall k \in [K]$
\begin{algorithmic}[1]
    \For{$k=1, \ldots, K$}
        \If{$k=1$ or $\exists h' \in [H]$ such that $\det(\bLambda_{h'}^k) \geq 2\det(\bLambda_{h'}^{k(m)})$}
            \State $m \leftarrow m+1$ and $k(m) \leftarrow k$ 
            \State $\pi_h^{k(m)}(\cdot|s) \leftarrow \pi_{\text{unif}}(\cdot|s), \ \forall h \in [H]$ 
            \State $\bar\bphi_h^m(s,a) = \bphi(s,a)\sigma(-\beta_\phi ||\bphi(s,a)||_{(\bLambda_h^{k(m)})^{-1}} + \log K), \ \forall h\in [H]$ 
        \EndIf
        \State Execute $\pi^k$ and observe $\{s_h^k,a_h^k,\btheta_{f,h}^k, \btheta_{g,h}^k\}_{h=1}^H$
        \For{$h=H,\ldots,1$}
            \State $\bLambda_h^{k+1} \leftarrow \bLambda_h^k+ \bphi(s_h^k,a_h^k)\bphi(s_h^k,a_h^k)^\top$ 
            \For{$\ell = f,g$}
                \State $\hat\bmu_h^k \hat V_{\ell,h+1}^k \leftarrow (\bLambda_h^k)^{-1}\sum_{\tau \in [k-1]}\bphi(s_h^\tau, a_h^\tau) \hat V_{\ell,h+1}^k(s_{h+1}^\tau)$
                \State $\hat Q_{\ell,h}^k(s,a) \leftarrow \bar\bphi_h^m(s,a)^\top \left[ \btheta_{\ell,h}^k + \hat\bmu_h^k \hat V_{\ell,h+1}^k\right] - \beta_p ||\bar\bphi_h^m(s,a)||_{(\bLambda_h^{k(m)})^{-1}}$
                \State $\hat V_{\ell,h}^k (\cdot)\leftarrow \sum_{a\in \calA}\pi_h^k(a|\cdot) \hat Q_{\ell,h}^k(\cdot,a)$ 
            \EndFor
        \EndFor
            \State ${\color{red} \bar Y_k = Y_k\cdot \Ind\{\hat V_{g,1}^k(s_1) \geq b\}}$
            \State Update the policy $\pi_h^{k+1}$ for all $h \in [H]$: 
            \[
                \pi_h^{k+1}(\cdot|s) \in \argmin_{\pi(\cdot|s)\in \Delta(\calA)}\ \psi_k(\pi(\cdot|s)) + \sum_{\tau=k(m)}^k (\hat Q_{f,h}^\tau(s,\cdot) + {\color{red}\bar Y_\tau}\hat Q_{g,h}^\tau(s,\cdot))^\top \pi(\cdot|s)
            \]
            \hspace{5mm}
            where $\psi_k(\pi(\cdot|s)) = \alpha_k\sum_{a\in\calA}\pi(a|s)\log\pi(a|s)$ and $\alpha_k = \frac{\beta_Q}{4}\sqrt{\sum_{\tau=k(m)}^{k}(1+{\color{red}\bar Y_\tau})^2}.$
        
        \State $R_k = 108\gamma H^2 + \beta_R \frac{Y_k}{\sqrt{\sum_{\tau=1}^{k}Y_\tau^2}}$ 
        \State $Z_{k+1} = Z_k + {\color{red}\left[\hat V_{g,1}^k(s_1) - b\right]_+} - R_k$
        \State $Y_{k+1} = \min\{\Phi'(Z_{k+1}), K^8\}$
    \EndFor
\end{algorithmic}
\end{algorithm}

\begin{lemma}\label{lem:good event hard}
    Let $1 \leq \beta_\phi\leq K$, and let $H\geq 2$. Then $\Pr(\calE_g) \geq 1-\delta$ for any $\delta \in (0,1)$.
\end{lemma}
\begin{proof}
    The value function updates retain their original form. In the FTRL update, \(Y_\tau\) is replaced by \(\bar Y_\tau\) both in the linear term and in the normalization \(\alpha_k\). Here, we note that $\bar Y_k$ is bounded, i.e., $\bar Y_k \leq Y_k \leq K^8$, which implies that the same value function class and policy class also apply for \Cref{alg:hard}. Therefore, Lemmas~\ref{lem:proxy good event} and~\ref{lem:proxy good event 2} apply. Then, we can prove that $\Pr(\calE_g) \geq 1-\delta$ by applying the same proof of \Cref{lem:good event}.
\end{proof}

\begin{lemma}[Formal statement of \eqref{eq:ingredient 1 hard}]\label{lem:surrogate regret hard}
Assume that the good event $\calE_g$ holds. For all $k \in [K]$, we have 
\begin{align*}
    &\sum_{\tau=1}^k \left(V_{f^\tau,1}^{\pi^\tau}(s_1) - V_{f^\tau,1}^{\pi^\star}(s_1)\right) + \sum_{\tau=1}^k Y_\tau \left[\hat V_{g,1}^{\tau}(s_1) - b\right]_+ \\
    &\leq H^2\left(\log|\calA|+20\right)\sqrt{dH\log(2K)}\sqrt{k + \sum_{\tau=1}^k Y_\tau^2} \\
    &\quad+\cbias d^{3/2}H^2\sqrt{K}\log(2HK/\delta)+\cbias d^3H^4 \log^4(2HK/\delta)
\end{align*}
where $\cbias$ is the absolute constant given by \Cref{lem:bias}.
\end{lemma}
\begin{proof}
    Since \Cref{alg:hard} applies the FTRL update for $\pi_h^{k+1}$ along with $\bar Y_k$, applying the same argument as in the proof of \Cref{lem:surrogate regret}, we have
    \begin{align*}
    &\sum_{\tau=1}^k \left(V_{f^\tau,1}^{\pi^\tau}(s_1) - V_{f^\tau,1}^{\pi^\star}(s_1)\right) + \sum_{\tau=1}^k \bar Y_\tau \left(\hat V_{g,1}^{\tau}(s_1) - b\right) \\
    &\leq H^2\left(\log|\calA|+20\right)\sqrt{dH\log(2K)}\sqrt{k + \sum_{\tau=1}^k \bar Y_\tau^2} \\
    &\quad+\cbias d^{3/2}H^2\sqrt{K}\log(2HK/\delta)+\cbias d^3H^4 \log^4(2HK/\delta).
    \end{align*}
    As shown in \eqref{eq:hard observe}, we have $\bar Y_\tau(\hat V_{g,1}^\tau(s_1) -b) = Y_\tau [\hat V_{g,1}^\tau(s_1) -b]_+$. Moreover, it is trivial that $\bar Y_\tau \leq Y_\tau$. Then, it can be rewritten as
    \begin{align*}
        &\sum_{\tau=1}^k \left(V_{f^\tau,1}^{\pi^\tau}(s_1) - V_{f^\tau,1}^{\pi^\star}(s_1)\right) + \sum_{\tau=1}^k Y_\tau \left[\hat V_{g,1}^{\tau}(s_1) - b\right]_+ \\
        &\leq H^2\left(\log|\calA|+20\right)\sqrt{dH\log(2K)}\sqrt{k + \sum_{\tau=1}^k Y_\tau^2} \\
        &\quad+\cbias d^{3/2}H^2\sqrt{K}\log(2HK/\delta)+\cbias d^3H^4 \log^4(2HK/\delta).
    \end{align*}
    This concludes the proof.
\end{proof}

\begin{lemma}\label{lem:drift appendix hard}
Assume that the good event $\calE_g$ holds. Then we have
    \begin{align*}
        \Phi(Z_{k+1}) - \Phi(Z_k) \leq \Phi'(Z_k)\left(\left[\hat V_{g,1}^k(s_1)-b\right]_+-\frac{1}{2}R_k\right).
    \end{align*}
\end{lemma}
\begin{proof}
    By applying the same argument as in the proof of \Cref{lem:drift appendix} together with $Z_{k+1} = Z_k + [\hat V_{g,1}^k(s_1) - b]_+ - R_k$, we prove the statement.
\end{proof}

\begin{lemma}\label{lem:no clip appendix hard}
    Assume that the good event $\calE_g$ holds. Let $H\geq 2$, and let $K \geq \cbias dH(\log|\calA|+20)\log(2H/\delta)$, where $\cbias$ is the absolute constant given by \Cref{lem:bias}. Then we have 
    \[
        \Phi'(Z_k) = Y_k,\quad \forall k\in [K+1].
    \]
\end{lemma}
\begin{proof}
    By applying the same argument as in the proof of \Cref{lem:no clip appendix} together with Lemmas~\ref{lem:surrogate regret hard} and~\ref{lem:drift appendix hard}, we prove the statement.
\end{proof}

\begin{lemma}\label{lem:hat V < V}
    Assume that the good event $\calE_g$ holds. Let $H\geq 2$, and let $K \geq \cbias dH(\log|\calA|+20)\log(2H/\delta)$, where $\cbias$ is the absolute constant given by \Cref{lem:bias}. Then we have for all $k \in [K]$,
    \begin{align*}
        \hat V_{g,1}^k(s_1) \leq V_{g^k,1}^{\pi^k}(s_1).
    \end{align*}
\end{lemma}
\begin{proof}
    We prove the statement by induction on $h$. Fix $k \in [K]$. For the base case $h= H+1$, it is trivial that $\hat V_{g,H+1}^k(s) = V_{g^k,H+1}^{\pi^k}(s) = 0$ for all $s \in \calS$. We then assume that $\hat V_{g,h+1}^k(s) \leq V_{g^k,h+1}^{\pi^k}(s)$ for all $s\in \calS$. 

    Next, to complete the induction, we show that $\hat V_{g,h}^k(s) \leq V_{g^k,h}^{\pi^k}(s)$ for all $s\in\calS$. 
    Letting $z = -\beta_\phi ||\bphi(s,a)||_{(\bLambda_h^{k(m)})^{-1}} + \log K$,  recall that   
    \begin{align*}
        \hat Q_{g,h}^k(s,a) 
        &= \bar\bphi_h^m(s,a)^\top \left[ \btheta_{g,h}^k + \hat\bmu_h^k \hat V_{g,h+1}^k\right] - \beta_p ||\bar\bphi_h^m(s,a)||_{(\bLambda_h^{k(m)})^{-1}}\\
        &= \bar\bphi_h^m(s,a)^\top \left[ \btheta_{g,h}^k + \bmu_h \hat V_{g,h+1}^k\right] \\
        &\quad+ \bar\bphi_h^m(s,a)^\top (\hat\bmu_h^k - \bmu_h)\hat V_{g,h+1}^k - \beta_p ||\bar\bphi_h^m(s,a)||_{(\bLambda_h^{k(m)})^{-1}}\\
        &\leq \bar\bphi_h^m(s,a)^\top \left[ \btheta_{g,h}^k + \bmu_h \hat V_{g,h+1}^k\right] \\
        &\quad+ ||\bar\bphi_h^m(s,a)||_{(\bLambda_h^{k})^{-1}} ||(\hat\bmu_h^k - \bmu_h)\hat V_{g,h+1}^k||_{\bLambda_h^{k}} - \beta_p ||\bar\bphi_h^m(s,a)||_{(\bLambda_h^{k(m)})^{-1}}\\
        &\leq \bar\bphi_h^m(s,a)^\top \left[ \btheta_{g,h}^k + \bmu_h \hat V_{g,h+1}^k\right] \\
        &\quad+ \beta_p ||\bar\bphi_h^m(s,a)||_{(\bLambda_h^{k})^{-1}} - \beta_p ||\bar\bphi_h^m(s,a)||_{(\bLambda_h^{k(m)})^{-1}}\\
        &\leq \bar\bphi_h^m(s,a)^\top \left[ \btheta_{g,h}^k + \bmu_h \hat V_{g,h+1}^k\right]\\
        &= \sigma(z)\left(g_h^k(s,a) + \bbP_h\hat V_{g,h+1}^k(s,a)\right)\\
        &\leq \sigma(z)\left(g_h^k(s,a) + \bbP_h V_{g^k,h+1}^{\pi^k}(s,a)\right) \\
        &\leq g_h^k(s,a) + \bbP_h V_{g^k,h+1}^{\pi^k}(s,a) \\
        &= Q_{g^k,h}^{\pi^k}(s,a)
    \end{align*}
    where the first inequality follows from the Cauchy-Schwarz inequality, the second inequality follows from $\calE_g$, the third inequality follows from $\bLambda_h^{k(m)} \preceq \bLambda_h^k$, the fourth inequality follows from the induction hypothesis, and the last inequality follows from $\sigma(z) \leq 1$ and that $g_h^k(s,a) + \bbP_h V_{g^k,h+1}^{\pi^k}(s,a)$ is nonnegative. 
    Thus, we have $\hat Q_{g,h}^k(s,a) \leq Q_{g^k,h}^{\pi^k}(s,a)$ for all $(s,a)$. Finally, it follows that
    \begin{align*}
       \hat V_{g,h}^k(s) = \sum_{a\in\calA} \pi_h^k(a|s)\hat Q_{g,h}^k(s,a) \leq  \sum_{a\in\calA} \pi_h^k(a|s)Q_{g^k,h}^{\pi^k}(s,a) = V_{g^k,h}^{\pi^k}(s).
    \end{align*}
    This completes the induction. Therefore, the statement holds.
\end{proof}

\begin{lemma}\label{lem:Reg Y Z hard}
    Assume that the good event $\calE_g$ holds. Let $H\geq 2$, and let $K \geq \cbias dH(\log|\calA|+20)\log(2H/\delta)$, where $\cbias$ is the absolute constant given by \Cref{lem:bias}. Then the following statements hold.
    \begin{enumerate}
        \item We have
        \begin{align*}
            \sum_{k=1}^K (V_{f^k,1}^{\pi^k}(s_1) - V_{f^k,1}^{\pi^\star}(s_1)) 
            &\leq H^2\left(\log|\calA|+20\right)\sqrt{dH\log(2K)}\sqrt{K} + 1\\
            &\quad+\cbias d^{3/2}H^2\sqrt{K}\log(2HK/\delta)+\cbias d^3H^4 \log^4(2HK/\delta).
        \end{align*}
        \item For any $k\in[K]$, 
        \begin{align*}
            Y_{k+1} \leq \frac{1}{9}\cbias d^3 H^3K^{7/2}(\log|\calA|+20)\log^4(2H/\delta).
        \end{align*}
        \item For any $k\in[K]$, 
        \begin{align*}
            Z_{k+1} \leq 12H^2(\log|\calA|+20)\sqrt{dHK}\log\left(4\cbias d^3H^4K^4(\log|A|+20)\log^4(2H/\delta)\right).
        \end{align*}
    \end{enumerate}
\end{lemma}
\begin{proof}
    By applying the same argument as in the proof of \Cref{lem:Reg Y Z}, we prove the statements.
\end{proof}

\begin{proof}[\textbf{Proof of Theorem~\ref{thm:hard}}]
    To prove the theorem, we consider the following two cases: (1) $K<\beta_\phi$ or $K < \cbias dH(\log|\calA|+20)\log(2H/\delta)$ and (2) $\beta_\phi\leq K$ and $\cbias dH(\log|\calA|+20)\log(2H/\delta)\leq K$.

    \textbf{Case 1:} Suppose that $K<\beta_\phi$ or $K < \cbias dH(\log|\calA|+20)\log(2H/\delta)$. In this case, we have
    \begin{align*}
        K &< \max\{\beta_\phi, \cbias dH(\log|\calA|+20)\log(2H/\delta)\}\\
        &\leq \beta_\phi+ \cbias dH(\log|\calA|+20)\log(2H/\delta)\\
        &\leq 748dH\sqrt{\log(2HK/\delta)}\log K + \cbias dH(\log|\calA|+20)\log(2H/\delta).
    \end{align*}
    Then the trivial bounds for $\Regret(K)$ and $\VioHard(K)$, i.e., $HK$, satisfy
    \begin{align*}
        \Regret(K),\VioHard(K) &\leq HK=\tbigO\left(dH^3\right).
    \end{align*}

    \textbf{Case 2:} 
    Suppose that $\beta_\phi\leq K$ and $\cbias dH(\log|\calA|+20)\log(2H/\delta)\leq K$. It is clear that $1\leq \beta_\phi$. In this case, we can assume $\calE_g$, which occurs with probability at least $1-\delta$ by \Cref{lem:good event hard}. Moreover, Lemmas~\ref{lem:no clip appendix hard},~\ref{lem:Reg Y Z hard}, and~\ref{lem:R sum appendix} hold.

    To bound $\Regret(K)$, by the first statement of \Cref{lem:Reg Y Z hard}, we have
    \begin{align*}
        \Regret(K) &= \sum_{k=1}^K (V_{f^k,1}^{\pi^k}(s_1) - V_{f^k,1}^{\pi^\star}(s_1)) \\
        &\leq H^2\left(\log|\calA|+20\right)\sqrt{dH\log(2K)}\sqrt{K} + 1\\
        &\quad+\cbias d^{3/2}H^2\sqrt{K}\log(2HK/\delta)+\cbias d^3H^4 \log^4(2HK/\delta)\\
        &=\tbigO\left(\sqrt{d^3H^4K} +\sqrt{dH^5 K} + d^3H^4\right).
    \end{align*}

    To bound $\VioHard(K)$, we observe that
    \begin{align*}
        \VioHard(K) 
        &= \sum_{k=1}^K \left[V_{g^k,1}^{\pi^k}(s_1)-b\right]_+\\
        &\leq \sum_{k=1}^K \left[V_{g^k,1}^{\pi^k}(s_1)-\hat V_{g,1}^k(s_1)\right]_+ + \sum_{k=1}^K \left[\hat V_{g,1}^{k}(s_1)-b\right]_+\\
        &= \sum_{k=1}^K \left(V_{g^k,1}^{\pi^k}(s_1)-\hat V_{g,1}^k(s_1)\right) + \sum_{k=1}^K \left[\hat V_{g,1}^{k}(s_1)-b\right]_+\\
        &= \sum_{k=1}^K \left(V_{g^k,1}^{\pi^k}(s_1)-\hat V_{g,1}^k(s_1)\right) + Z_{K+1} + \sum_{k=1}^K R_k
    \end{align*}
    where the inequality follows from the fact that $[a+b]_+ \leq [a]_+ + [b]_+$ for any $a,b\in\bbR$, the second equality is due to $V_{g^k,1}^{\pi^k}(s_1) - \hat V_{g,1}^k(s_1) \geq 0$ by \Cref{lem:hat V < V}, and     the last equality can be obtained by summing $[\hat V_{g,1}^k(s_1) -b]_+ = Z_{k+1}-Z_k + R_k$ over $k$, and the fact that $Z_1 = 0$. Then by \Cref{lem:bias}, the third statement of \Cref{lem:Reg Y Z hard}, and \Cref{lem:R sum appendix}, we have
    \begin{align*}
        &\VioHard(K) \\
        &\leq\cbias\left(d^{3/2}H^2\sqrt{K}\log(2HK/\delta)+d^3H^4 \log^4(2HK/\delta)\right) \\
        &\quad+12H^2(\log|\calA|+20)\sqrt{dHK}\log\left(4\cbias d^3H^4K^4(\log|A|+20)\log^4(2H/\delta)\right)\\
        &\quad+ 12\sqrt{dH^5}(\log|\calA|+20)\sqrt{K}\log(3dH(\log|\calA|+20)K) \\
        &=\tbigO\left(\sqrt{d^3H^4K}+ \sqrt{dH^5K} +d^3H^4\right).
    \end{align*}
\end{proof}

}

\section{Numerical Experiments}
\label{app:numerical-experiments}

We evaluate \Cref{alg:main} on the chain CMDP of \citet{yu2026near}, with periodically alternating losses and a fixed constraint, and compare it with the algorithm of \citet{yu2026primaldual}. We use $H=10$, state space $\mathcal S=\{0,\ldots,H+1\}$, action space $\mathcal A=\{-1,1\}^4$, and an exact linear-MDP representation with feature dimension $d=22$. Each episode starts at $s_1=0$. For every step $h$ and state $s<H$, we set $P_h(s+1\mid s,a)=0.95-0.01\sum_{i=1}^{4}a_i$ and $P_h(H+1\mid s,a)=1-P_h(s+1\mid s,a)$, with all other transition probabilities equal to zero. States $H$ and $H+1$ are absorbing: $P_h(s\mid s,a)=1$ for $s\in\{H,H+1\}$. Each episode lasts $H$ steps, including steps after reaching an absorbing state. Both methods receive full-information loss feedback and the same known constraint function, while transitions remain unknown. We conduct five simulations with paired random seeds, each running for $K=5{,}000$ episodes.

Let $q(a)=\frac{1}{8}\sum_{i=1}^{4}(a_i+1)$. For $s<H$, we set $f_h^k(s,a)=1-0.4q(a)$ when $\lfloor k/10\rfloor$ is even and $f_h^k(s,a)=1-0.4(1-q(a))$ otherwise. At the absorbing states, we set $f_h^k(H,a)=1$ and $f_h^k(H+1,a)=0$. The constraint function is fixed: $g_h^k(s,a)=1-q(a)$ for $s<H$ and $g_h^k(s,a)=1$ otherwise, with budget $b=4$. These definitions apply to every step $h$ and episode $k$, with episode indices starting at $k=1$.

To promote learning, we reduce the optimistic bonus coefficients while preserving their different dependence on $K$: we set $\beta_p=0.01$ and $\beta_b=0.01K^{1/4}$, with contraction coefficients $\beta_\phi=4\beta_p\log K$ and $\beta_w=4\beta_b\log K$. For our dual update, we set $\beta_R=0.1$ and $\gamma=\min\{0.1/\sqrt K,1/(6\beta_R)\}$, and replace the coefficient 108 in $R_k$ with 0.1. These practical parameter settings differ from those covered by \Cref{thm:main}. All remaining algorithmic parameters follow the respective papers, with $\delta=0.05$.

\begin{figure}[t]
    \centering
    \includegraphics[width=\linewidth]{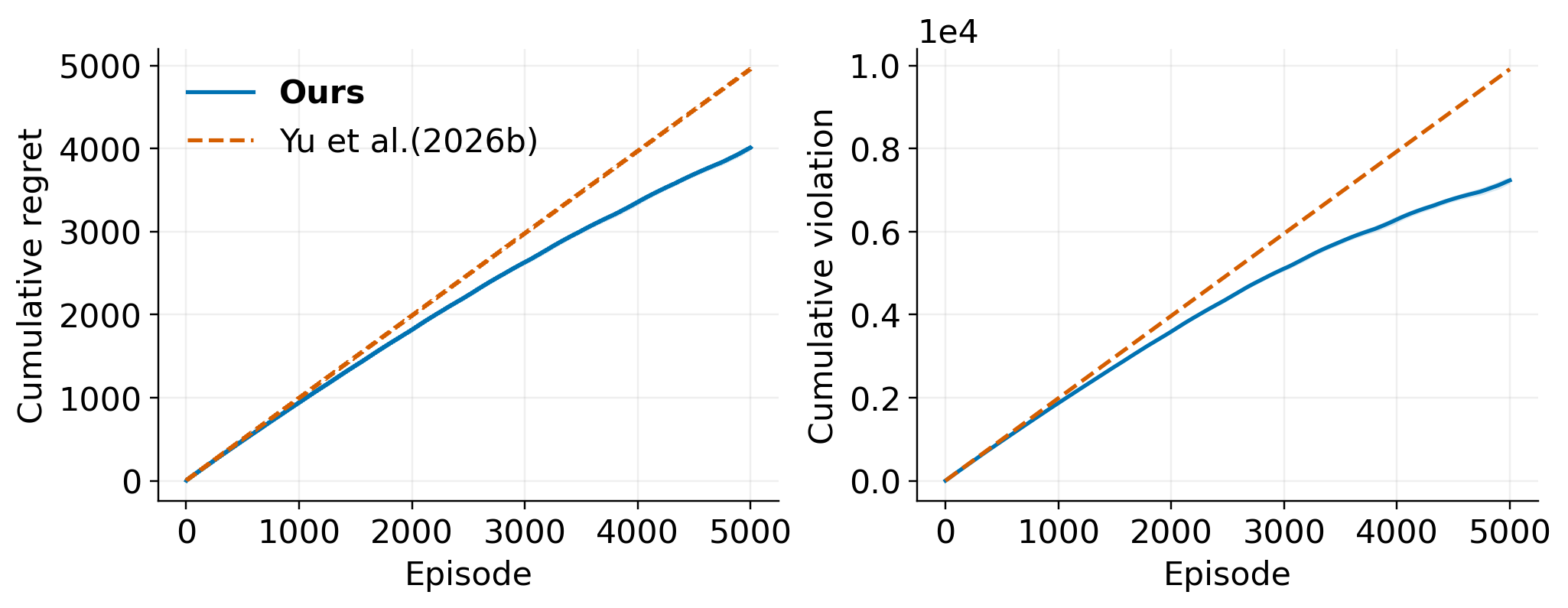}
    \caption{Cumulative regret (left) and the positive part of cumulative
    constraint violation (right) for $K=5{,}000$ episodes. Curves show averages
    over five seeds; shaded regions indicate 95\% confidence intervals.}
    \label{fig:numerical-comparison}
\end{figure}

\Cref{fig:numerical-comparison} summarizes the results. Our method achieves lower cumulative regret and constraint violation than the baseline over the later episodes. 
The baseline of \citet{yu2026primaldual} exhibits approximately linear growth in both quantities. These results illustrate an empirical advantage of our method on this instance under the tested parameterizations.

\section{Auxiliary Lemmas}

\begin{lemma}[Lemma 22 of \citet{cassel2024warm}] 
\label{lem:cassel 22}
Let $\hat\bmu_h^k:\bbR^{|\calS|}\to\bbR^d$ be the linear operator defined in Line~\ref{eq:muV estimate} of \Cref{alg:main}. For all $h\in [H]$, let $\calV_h \subseteq \bbR^{|\calS|}$ be a set of mappings $\hat V:\calS \to \bbR$ such that $||\hat V||_\infty \leq \beta_Q$ and $\beta_Q\geq 1$. With probability at least $1-\delta$, for all $h\in[H]$, $\hat V \in \calV_{h+1}$, and $k\geq 1$,
\begin{align*}
    ||(\bmu_h-\hat\bmu_h^k)\hat V||_{\bLambda_h^k} \leq 4\beta_Q\sqrt{d\log(K+1)+2\log(H\calN_\epsilon/\delta)}
\end{align*}
where $\epsilon\leq \beta_Q\sqrt{d}/(2K)$, $\calN_\epsilon = \sum_{h=1}^H \calN_\epsilon(\calV_h)$, and $\calN_\epsilon(\calV_h)$ is the $\epsilon$-covering number of $\calV_h$ with respect to the $\ell_\infty$-norm.
\end{lemma}

\begin{lemma}[Lemma 24 of \citet{cassel2024warm}]\label{lem:cassel 24}
Let \( \mathcal{V} = \{ V(\cdot; \btheta) : ||\btheta|| \leq W \} \) denote a class of functions \( V : \mathcal{S} \to \mathbb{R} \). Suppose that any \( V \in \mathcal{V} \) is \( L \)-Lipschitz with respect to \( \btheta \) and the supremum distance, i.e.,
\[
|| V(\cdot; \btheta_1) - V(\cdot; \btheta_2) ||_{\infty} \leq L || \btheta_1 - \btheta_2 ||, \quad ||\btheta_1||, ||\btheta_2|| \leq W.
\]
Let \( \mathcal{N}_\epsilon(\calV) \) be the \( \epsilon \)-covering number of \( \mathcal{V} \) with respect to the $\ell_\infty$-norm. Then
\[
\log \mathcal{N}_\epsilon(\calV) \leq d \log(1 + 2WL/\epsilon).
\]
\end{lemma}

\begin{lemma}[Lemma 8 in \citet{cassel2024warm}]\label{lem:num epoch}
    The number of epochs $|E|$ is bounded by $(3/2)dH\log(2K)$.
\end{lemma}

\begin{lemma}[Lemma D.4 of \citet{rosenberg2020near}]\label{lem:rosen d4}
Let \( \{X_t\}_{t \geq 1} \) be a sequence of random variables with expectation adapted to a filtration \( \mathcal{F}_t \). Suppose that \( 0 \leq X_t \leq C\) almost surely. Then with probability at least \( 1 - \delta \),
\[
\sum_{t=1}^T \mathbb{E}[X_t \mid \mathcal{F}_{t-1}] \leq 2 \sum_{t=1}^T X_t + 4 C\log \frac{2T}{\delta}.
\]
\end{lemma}

\begin{lemma}[Lemma 18 in \citet{cassel2024warm}]\label{lem:cassel 18}
    For any $K\geq 2$, $z\geq 0$, we have $\max_{z\geq 0} [z\cdot\sigma(-\beta z +\log K)] \leq 2(\log K)/\beta$.
\end{lemma}
\begin{lemma}[Lemma 19 in \citet{cassel2024warm}]\label{lem:cassel 19}
    For any $K\geq 1$, $z\geq 0$, we have $\sigma(z-\log K) \leq 2(z^2 + K^{-1})$.
\end{lemma}

\begin{lemma}[Lemma 3 in \citet{cassel2024warm}]\label{lem:cassel 3}
    For any $m\in E$ and $(s,a,h,k) \in \calS\times\calA\times[H]\times K(m)$, and $v \in \bbR^d$, we have
    \begin{align*}
        \left(\bphi(s,a) - \bar\bphi_h^m(s,a)\right)^\top \bfv \leq \left(4\beta_\phi^2||\bphi(s,a)||_{(\bLambda_h^k)^{-1}}^2 + 2K^{-1}\right)|\bphi(s,a)^\top\bfv|
    \end{align*}
\end{lemma}

\begin{lemma}[Lemma 12 in \citet{abbasi2011improved}]\label{lem:abbasi 12}
    Let $\bfA, \bfB$ be positive definite matrices, and let $\bfC$ be a positive semidefinite matrix such that $\bfA = \bfB + \bfC$. Then, we have 
    \begin{align*}
        \sup_{\bfx \neq \bm0} \frac{\bfx^\top\bfA \bfx}{\bfx^\top\bfB\bfx} \leq \frac{\det(\bfA)}{\det(\bfB)}.
    \end{align*}
\end{lemma}

\begin{lemma}[Lemma D.2 of \citet{jin2020provably}]\label{lem: jin d2}
Let \( \{\bphi_t\}_{t \geq 0} \) be a bounded sequence in \( \mathbb{R}^d \) satisfying \( \sup_{t \geq 0} ||\bphi_t|| \leq 1 \). Let \( \bLambda_0 \in \mathbb{R}^{d \times d} \) be a positive definite matrix. For any \( t \geq 0 \), define
\[
\bLambda_t = \bLambda_0 + \sum_{j=1}^t \bphi_j \bphi_j^\top.
\]
Then, if the smallest eigenvalue of \( \bLambda_0 \) satisfies \( \lambda_{\min}(\bLambda_0) \geq 1 \), we have
\[
\log \left[ \frac{\det(\bLambda_t)}{\det(\bLambda_0)} \right]
\leq \sum_{j=1}^t \bphi_j^\top \bLambda_{j-1}^{-1} \bphi_j
\leq 2 \log \left[ \frac{\det(\bLambda_t)}{\det(\bLambda_0)} \right].
\]
\end{lemma}

\subsection*{AI use statement}
In this work, we used generative AI tools to polish paragraphs, check the correctness of mathematical arguments, assist with the implementation of numerical experiments, {\color{alert}and suggest the key idea underlying the extension to hard constraint violation, specifically the introduction of $\bar Y_k$.} We have not used generative AI tools to write complete mathematical proofs or complete drafts of the manuscript, and other tasks requiring disclosure are not applicable to this work. We have reviewed all AI-assisted work, including the experimental code. We take responsibility for the final content of this work, including all text, claims, proofs, and artifacts produced with the aid of generative AI.

\end{document}